\documentclass[9pt, twocolumn]{IEEEtran}
\usepackage{cpremieee}
\usepackage{flushend}
\title{Data Smashing}
\author{\begin{tabular}{cc}
Ishanu Chattopadhyay & Hod Lipson \\ \texttt{ic99@cornell.edu} & \texttt{hod.lipson@cornell.edu}
\end{tabular}}
\begin{document} 
\figstrue

\maketitle
\begin{abstract} Investigation  of  the underlying physics or biology from empirical data requires a  quantifiable notion of similarity - when do two observed data sets indicate nearly  identical generating processes, and when they do not. The discriminating characteristics  to look for in data is often determined by heuristics designed by  experts, $e.g.$,  distinct shapes of ``folded'' lightcurves may be used as ``features'' to  classify variable stars, while  determination of pathological   brain states might require a Fourier analysis of  brainwave activity. Finding good features is non-trivial.  Here, we propose a  universal solution to this problem: we delineate a  principle for quantifying similarity between  sources of arbitrary data streams, without a priori knowledge, features or training. We uncover an algebraic structure on a space of symbolic models for quantized data, and show that such stochastic generators may be added and uniqely inverted; and that a model and its  inverse always sum to the   generator of flat white noise. Therefore,  every data stream has an anti-stream: data generated by the inverse model. Similarity between two streams, then,  is the degree to which one, when summed to the other's anti-stream, mutually  annihilates all statistical structure to  noise. We call this data smashing.
We present diverse applications, including disambiguation of brainwaves pertaining to epileptic seizures, detection of anomalous cardiac rhythms, and  classification of astronomical objects from raw photometry. In our examples, the data smashing principle, without access to any  domain knowledge,   meets or exceeds the performance of specialized algorithms tuned by domain experts. 

\end{abstract}
\allowdisplaybreaks{

\section{Motivation \& Contribution}\label{sec-mot}
The term ``data smashing'' might conjure up images of erasing information or destroying hard drives. But just as smashing atoms can reveal their composition, ``colliding'' quantitative data streams can reveal their hidden structure. 

We describe here a new principle, where quantitative data streams have corresponding anti-streams, which inspite of being non-unique, are tied to  the stream's unique statistical structure. We then describe  ``data smashing'', a process by which streams and anti-streams can be algorithmically collided to reveal differences that are difficult to detect using conventional techniques. We establish this principle formally, describe how we implemented it in practice, and report its performance on a number of real-world cases. The results show that without access to any domain knowledge, data smashing meets or exceeds the accuracy achieved by specialized algorithms and heuristics devised by domain experts.  

Nearly all automated discovery systems today rely, at their core, on the ability to compare data: From automatic image recognition to discovering new astronomical objects, such systems must be able to compare and contrast data records in order to group them, classify them, or identify the odd-one-out. Despite rapid growth in the amount of data collected and the increasing rate at which it can be processed, analysis of quantitative data streams still relies heavily on knowing what to look for. 

Any time a data mining algorithm searches beyond simple correlations, a human expert must help define a notion of similarity - by specifying important distinguishing ``features'' of the data to compare, or by training learning algorithms using copious amounts of examples. The data smashing principle removes the reliance on expert-defined features or examples, and in many cases, does so faster and with better accuracy than traditional methods.

This paper is organized as follows:  Sections~\ref{sec-mot}-\ref{sec-appl} describe the key concepts, along with a brief but complete description of the approach. The mathematical details, including proffs of correctness, are presenetd in Sections~\ref{sec-pfsa}-\ref{sec-stropr}. Qunatization schemes are discussed in Section~\ref{sec-quant}. Comparisons with some standard notions of statistical dependencies is carried out in Section~\ref{sec-addl}, and the paper is concluded in Secion~\ref{sec-conc}.  
%
\section{Anti-streams}
The notion of data smashing applies only to data in the form of an ordered series of digits or symbols, such as acoustic waves from a microphone, light intensity over time from a telescope, traffic density along a road, or network activity from a router.  The anti-stream contains the ``opposite'' information from the original data stream, and is produced by algorithmically inverting the statistical distribution of symbol sequences appearing in the original stream. For example, sequences of digits that were common in the original stream will be rare in the anti-stream, and vice versa. Streams and anti-streams can then be algorithmically ``collided'' in a way that systematically cancels any common statistical structure in the original streams, leaving only information relating to their statistically significant differences. We call this the principle of \textit{Information Annihilation} (See Fig.~\ref{fig1}).

Data smashing involves two data streams and proceeds in three steps:   raw data streams are first quantized, by converting continuous value to a string of characters or symbols. The simplest example of such quantization is where all positive values are mapped to the symbol ``1'' and all negative values to ``0'', thus generating a string of bits. Next, we select one of the quantized input streams, and generate its anti-stream. Finally, we annihilate this anti-stream against the remaining quantized input stream and measure what information remains. The remaining information is estimated from the deviation of the resultant stream from flat white noise (FWN). 
\iffigs
\begin{figure}[t]
\centering

\vspace{-10pt}

\hspace{-14pt}\includegraphics[width=3in]{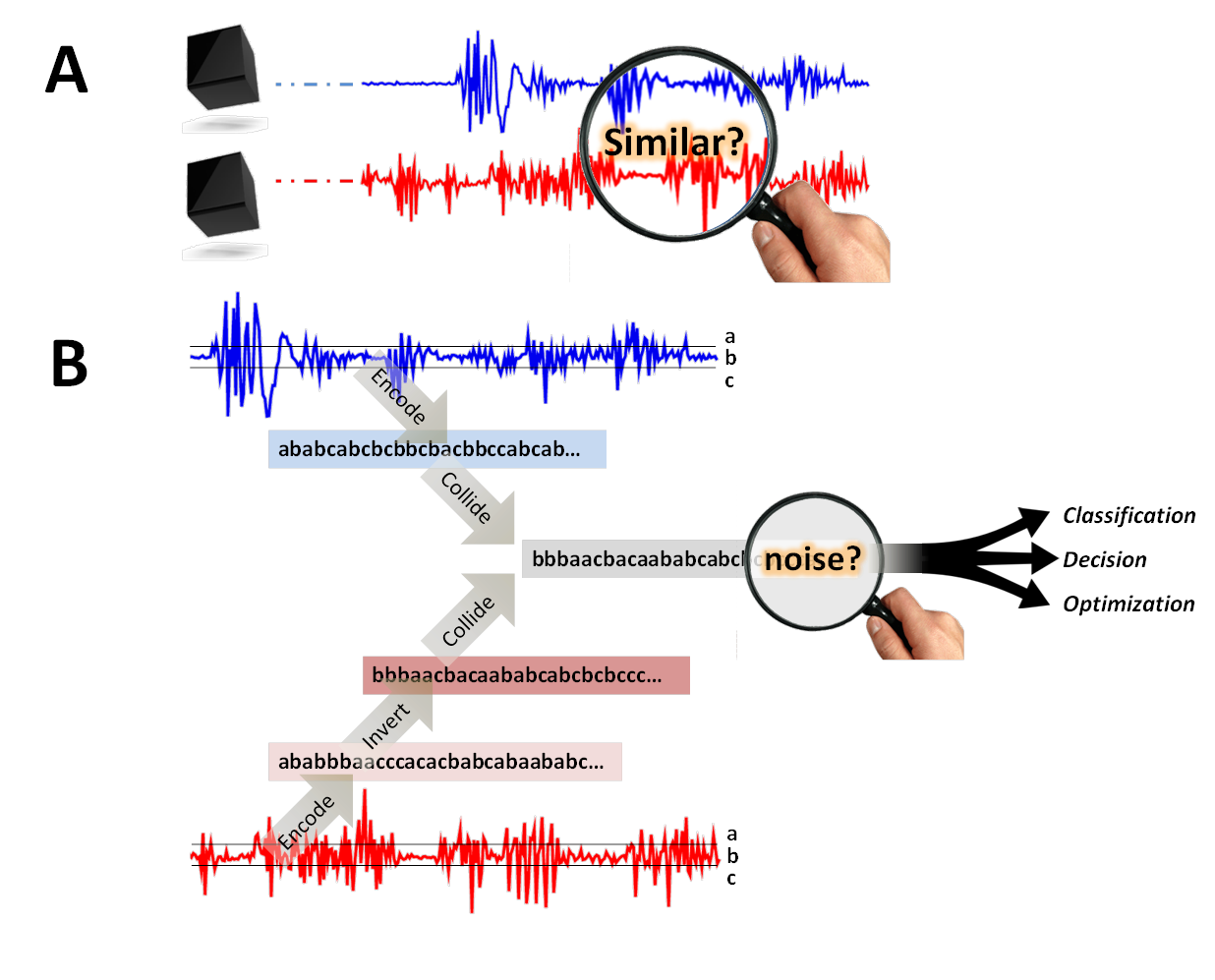}
\vspace{-20pt} 
\captionN{\textbf{Data smashing:} (A) determining the similarity between two data streams is key to any data mining process, but relies heavily on human-prescribed criteria. (B) Data smashing first encodes each data stream, then collides one with the inverse of the other. The randomness of the resulting stream  reflects the similarity of the original streams, leading to a cascade of downstream applications involving classification, decision and optimization.}\label{fig1}
\end{figure}
\fi
\iffigs
\begin{figure*}[t]

\centering
\input{Figures/sc_figtwo.tex}
\vspace{-4pt}

\captionN{\textbf{Calculation of causal similarity  using information annihilation.} (A) We quantize  raw signals to symbolic sequences over the chosen  alphabet, and compute a causal similarity  between such sequences. The underlying theory is established assuming the existence of generative probabilistic automata for these  sequences, but our algorithms  do not require explicit model construction, or  a priori knowledge of their structures. 
(B) Concept of stream inversion; while we can find the group inverse of a given PFSA algebraically, we can also transform a generated sequence directly to one that represents the inverse model, without constructing  the  model itself. 
(C) Summing PFSAs $G$ and its inverse $-G$ yields the zero PFSA $W$. We can  carry  out  this annihilation purely at the sequence level to get flat white noise. (D) Circuit that allows us to measure  similarity distance between streams $s_1,s_2$ via computation of  $\epsilon_{11}$, $\epsilon_{22}$ and $\epsilon_{12}$ (See Table 1). Given a threshold $\epsilon^\star>0$, if $\epsilon_{kk} < \epsilon^\star$, then we have sufficient data for stream $s_k$ ($k=1,2$). Additionally if $\epsilon_{12} \leqq \epsilon^\star$, then we conclude that $s_1,s_2$ have the same stochastic source with high probability (which converges exponentially fast to 1 with length of input).
}\label{fig2}
\end{figure*}
\fi
\iffigs
\begin{table*}[t]
\centering
\captionN{Algorithms For Stream Operations}\label{tab1}

{\bf  \sffamily  \small \itshape (Procedures below are used to assemble the  annihilation circuit
shown in Fig. 2D, which carries out data smashing)}
\vskip .4em
\definecolor{blcol}{RGB}{150,150,250}
\definecolor{murcol}{RGB}{255,150,150}
\definecolor{healcol}{RGB}{150,235,150}
\definecolor{colzeta}{RGB}{250,250,200}
\definecolor{leftpcol}{RGB}{200,255,200}
\definecolor{rightpcol}{RGB}{200,200,255}
\definecolor{boxcol}{RGB}{210,200,200}
\definecolor{linecol}{RGB}{200,180,180}
\definecolor{colplus}{RGB}{235,220,220}
\definecolor{colinv}{RGB}{200,220,255}
\definecolor{cof}{RGB}{219,144,71}
\definecolor{pur}{RGB}{200,200,200}
\definecolor{greeo}{RGB}{91,173,69}
\definecolor{greet}{RGB}{52,111,72}
 \definecolor{nodecol}{RGB}{180,180,220}
 \definecolor{nodeedge}{RGB}{240,248,255}
  \definecolor{nodecolb}{RGB}{220,180,180}
  \definecolor{nodecolc}{RGB}{180,220,180}
  \definecolor{nodecolcD}{RGB}{100,160,100} 
  \definecolor{nodecolW}{RGB}{190,190,190}
  \definecolor{edgecol}{RGB}{160,160,180}   
 \definecolor{nodecolD}{RGB}{140,140,180}
  \definecolor{nodecolbD}{RGB}{180,140,140}
\tikzset{oplus/.style={path picture={%
      \draw[black]
       (path picture bounding box.south) -- (path picture bounding box.north) 
       (path picture bounding box.west) -- (path picture bounding box.east);
      }}} 
\tikzset{%
  highlight/.style={draw=Red4,rectangle,rounded corners=1pt, opacity=1,fill=gray!10,thick,inner sep=-1.25pt,on background layer,fill opacity=.3}
}
\tikzset{%
  highlightg/.style={draw=DodgerBlue4,rectangle,rounded corners=1pt, opacity=1,fill=gray!10,thick,inner sep=-1.25pt,on background layer,fill opacity=.3}
}
\tikzset{%
fshadow/.style={      preaction={
         fill=black,opacity=.1,
         path fading=circle with fuzzy edge 20 percent,
         transform canvas={xshift=1mm,yshift=-1mm}
       }}
}
\def\ITEMSEP{.5pt}  \small \sffamily 
\newcommand{\tikzmark}[2]{\tikz[overlay,remember picture,baseline=(#1.base)] \node (#1) {#2};}
\newcommand{\Highlight}[3][submatrix]{%
    \tikz[overlay,remember picture]{
    \node[highlight,fit=(#2.north west) (#3.south east)] (#1) {};}
}
\newcommand{\HighlightG}[3][submatrix]{%
    \tikz[overlay,remember picture]{
    \node[highlightg,fit=(#2.north west) (#3.south east)] (#1) {};}
}\newcommand{\scpy}{\begin{tikzpicture}[very thick,scale=10] 
\node at (0.05,0.05) [fill=green!10,draw,rectangle, drop shadow] (n2){$\mathcal{C}$};
 \node at (-0.065,0.1) {$s$};
 \node at (0.165,0.055) {$s'$};
\draw (-0.05,0.1) -- (0.15,0.1);
\draw (-0.02,0.05) -- (0.03,0.05);
\draw (-0.02,0.1) -- (-0.02,0.05);
\draw  (0.07,0.05) -- (0.15,0.05);
\end{tikzpicture}
}
\newcommand{\sinv}{\begin{tikzpicture}[very thick,scale=10] 
\node at (0.05,0.05) [fill=colinv,draw,diamond, drop shadow,inner sep=0.5pt,minimum size=10pt] (n2){$\mathds{-1}$};
 \node at (-0.065,0.05) {$s$};
 \node at (0.165,0.055) {$s'$};
\draw (-0.05,0.05) -- (0.015,0.05);
\draw  (0.085,0.05) -- (0.15,0.05);
\end{tikzpicture}
}
\newcommand{\ssum}{\begin{tikzpicture}[very thick,scale=10] 
\node at (0.05,0.05) [fill=colplus,draw,oplus, circle,drop shadow,inner sep=0.5pt,minimum size=15pt, very thick] (n2){};
 \node at (-0.065,0.1) {$s_1$};
 \node at (-0.065,0) {$s_2$};
 \node at (0.165,0.055) {$s'$};
\draw (-0.05,0.1) -- (0.05,0.1);
\draw (0.05,0.1) -- (0.05,0.075);
\draw (-0.05,-0.0) -- (0.05,-0.0);
\draw (0.05,0) -- (0.05,0.025);
\draw  (0.08,0.05) -- (0.15,0.05);
\end{tikzpicture}
}
\newcommand{\mest}{\begin{tikzpicture}[very thick,scale=10] 
\node at (0.05,0.0) [fill=colzeta,draw,rectangle, drop shadow] (n2){$\hat{\zeta}$};
\draw (-0.05,0.0) -- (0.025,0.0);
 \node at (-0.065,0.0) {$s$};
 \node at (0.225,0.0) {\mnp{1in}{\fontsize{7}{7}\selectfont \sffamily \mnp{.35in}{Real Number \\ output\\
in $[0,1]$}}};
\end{tikzpicture}
}
\newcommand{\VSP}{\rule{0pt}{2ex}}
\begin{tabular}{m{2.85in} m{3.75in}}\hline
\VSP\bf Stream Operation & \VSP \bf Algorithmic Procedure (Pseudocode) \\\hline

\mnp{2.75in}{\bf $\blacksquare$ Independent
Stream Copy$^\dag$\\}

\mnp{1in}{
\scpy
}\hfill
\mnp{1.5in}{
\raggedright\bf \sffamily \small Generate an independent sample path from the same hidden stochastic source. }
 &

\begin{list}{\labelitemi}{\leftmargin=0.75em} \setlength{\itemsep}{\ITEMSEP}%
\item[1] Generate stream $\omega_0$ from FWN
 \item[2] Read current symbol $\sigma_1$ from $s_1$, and $\sigma_2$ from
$\omega_0$
\item[3] If $\sigma_1= \sigma_2$, then write $\sigma_1$  to output $s'$
\item[4] Move read positions one step to right, and go to step 1
\end{list}
{\itshape This operation is required internally in  stream inversion.}
 \\\hline

\mnp{2.75in}{
\bf $\blacksquare$ Stream 
Inversion$^\dag$\\}

\mnp{1in}{
\sinv}
\hfill
\mnp{1.5in}{
\raggedright \bf \sffamily   \small
Generate sample path from  inverse model of  hidden  source. }
 &

\begin{list}{\labelitemi}{\leftmargin=0.75em} \setlength{\itemsep}{\ITEMSEP}%
\item[1] Generate $\vert \Sigma \vert-1$ independent copies of $s_1$: $s_1,\cdots, s_{\vert\Sigma\vert -1}$
 \item[2] Read current symbols $\sigma_i$ from $s_i$ ($i=1,\cdots,\vert\Sigma\vert-1$)
\item[3] If $\sigma_i\neq  \sigma_j$ for all distinct  $i,j$, then write 
$\Sigma \setminus \bigcup_{i=1}^{\vert \Sigma \vert-1}\sigma_i$
to output $s'$
\item[4] Move read positions one step to right, and go to step 1
\vspace{-5pt}
\end{list}
 \\\hline
\mnp{2.75in}{
\vspace{5pt}

\bf$\blacksquare$  Stream Summation$^\dag$\\}

\mnp{1in}{\ssum}
\hfill
\mnp{1.5in}{\raggedright\bf \sffamily  \small
Generating sample path from  sum of hidden sources.}
 &

\begin{list}{\labelitemi}{\leftmargin=0.75em}  \setlength{\itemsep}{\ITEMSEP}%
 \item[1] Read current symbols $\sigma_i$ from $s_i$ ($i=1,2$)
\item[2] If $\sigma_1 =  \sigma_2$, then write 
to output $s'$
\item[3] Move read positions one step to right, and go to step 1
\end{list}
\\\hline
\mnp{2.75in}{
\bf$\blacksquare$  Deviation
from FWN$^\ddag$\\}

\mnp{1.2in}{\mest}
\hfill
\mnp{1.5in}{\raggedright\bf \sffamily  \small
Estimating the deviation of a symbolic stream from FWN. 
}
\vspace{2pt}

(Symbolic derivatives (Definition~\ref{defsymderivative}) in  Section~\ref{sec-pfsa} formalizes $\phi^s(\cdot)$. If $s$ is generated by a FWN process, then $\phi^s(x) \rightarrow \mathcal{U}_\Sigma$ for any $x \in \Sigma^\star$, and hence $\hat{\zeta}(s,\ell) \rightarrow 0$.)
 &\cgatherd[-3pt][-5pt]{
\hat{\zeta}(s,\ell)=
\frac{\vert \Sigma \vert -1}{\vert \Sigma\vert}\sum_{x: \vert x \vert \leqq \ell} \frac{\vert \vert \phi^s(x) - \mathcal{U}_\Sigma \vert \vert_\infty}{\vert \Sigma \vert^{2\vert x \vert}} 
 ,\ \textrm{where}}
\begin{itemize} \setlength{\itemsep}{\ITEMSEP}%
 \item $\vert \Sigma \vert$ is the alphabet size, $\vert x\vert$ is the length of the string $x$
\item $\ell$ is the maximum length of strings upto which the sum is evaluated. For a given  $\epsilon^\star$, we choose $\ell = \ln(1/\epsilon^\star)/\ln (\vert \Sigma \vert)$ (See Proposition~\ref{prophiddenmodeltwo})
\item $\mathcal{U}_\Sigma$ is the uniform probability vector of length $\vert\Sigma\vert$
\item For $\sigma_i \in \Sigma$, $\phi^s(x) \big \vert_i = \frac{\txt{\# of occurrences of  $x\sigma_i$ in string $s$}}{\txt{\# of occurrences of  $x$ in string $s$}}$ 
\end{itemize}
 \\\hline
\end{tabular}

\flushleft {\footnotesize $^\dag$See Section~\ref{sec-stropr} for proof of correctness }
\vspace{-2pt}
\flushleft {\footnotesize $^\ddag$See Definition~\ref{defdevest} and Propositions~\ref{prophiddenmodel} and \ref{prophiddenmodeltwo} in Section~\ref{sec-stropr}}
\vspace{-2pt}
\flushleft{\footnotesize $^\textrm{\textsection}$ Symbolic derivatives underlie the rigorous proofs. However, for the actual implementation, they are only needed in the final step to compute deviation from FWN }
\vspace{-8pt}
\end{table*}
\fi
Since a data stream is perfectly annihilated by a correct realization of its anti-stream,  any deviation of the collision product from noise quantifies statistical dissimilarity.  Using this causal similarity metric, we can cluster streams, classify them, or identify stream segments that are unusual or different. The algorithms are linear in input data, implying they can be applied efficiently to streams in near-real time. Importantly, data smashing can be applied without understanding where the streams were generated, how they are encoded, and what they represent. 

Ultimately, from a collection of data streams and their pairwise similarities, it is possible to automatically ``back out'' the underlying metric embedding of the data, revealing its hidden structure for use with traditional machine learning methods.

Dependence across data streams is often quantified using  mutual information~\cite{CT91}. However, mutual information and data smashing are distinct concepts. The former measures dependence between streams; the latter computes a distance between the generative processes themselves.  Two sequences of  independent coin-flips necessarily have zero mutual information, but data smashing will identify the streams as similar; being generated by the same stochastic process. Moreover, smashing only works correctly if the streams are independent or nearly so (See Section~\ref{sec-addla}).

Similarity computed  via data smashing is clearly a function of the statistical information buried in the input streams. However,  it might not be easy to find the right statistical tool, that reveals this hidden information, particularly without domain knowledge, or without first  constructing a good system model (See Section~\ref{sec-addlb} for an example where smashing reveals non-trivial categories missed by simple statistical measures). We describe in detail the process of computing anti-streams, and the process of comparing information. In Section~\ref{sec-pfsa}-\ref{sec-stropr}  we provide theoretical bounds on the confidence levels, minimal data lengths required for reliable analysis, and scalability of the process as function of the signal encodings. 
%
%

We have  limitations. 
Data smashing is not directly applicable to learning tasks that do not depend or require a notion of similarity, $e.g.$, identifying a specific time instant at which some event of interest transpired within a data set, or predicting the next step in a time series.
Even with the problems to which smashing is applicable, we do not claim  strictly superior quantitative performance to the state-of-art in any and all applications; carefully chosen approaches tuned to specific problems can certainly do as well, or better. Our claim is not that we uniformly outperform existing methods, but that we are on par, as evidenced in multiple example applications; yet do so without requiring expert knowledge, or a training set. Additionally, technical reasons preclude applicability to data from strictly deterministic systems (See section on Limitations \& Assumptions).
\section{The Hidden Models}
The notion of a universal comparison metric makes sense only in the context of a featureless approach, where one considers pairwise similarity (or dissimilarity)  between   individual measurement sets. However, while the advantage  of considering the notion of similarity between data sets instead of between feature vectors has been recognized~\cite{duinexperiments1997,mottlfeatureless2001,kalskadissimilarity2002}, the definition of similarity measures  has remained intrinsically heuristic and application dependent, with the possibility of  a universal  metric  been summarily rejected. We show that such universal comparison is indeed realizable, at least under some general assumptions on the nature of the generating process.
 
We consider sequential observations, $e.g.$, time series of sensor data. The first step is mapping the possibly continuous-valued sensory observations to discrete symbols via pre-specified quantization of the data range (See Section~\ref{sec-quant} and Fig.~\ref{figS1}). Each symbol represents a slice of the data range, and the total number of slices define the symbol alphabet $\Sigma$ (where $\vert \Sigma\vert$ denotes the alphabet size). The coarsest quantization has a binary alphabet consisting of say $0$ and $1$ (it is not important what symbols we use, we can as well represent the letters of the alphabet with $a$ and $b$), but finer quantizations with larger alphabets are also possible. An observed data stream is thus mapped to a symbolic sequence over this pre-specified alphabet. We assume that the symbol alphabet and its interpretation is fixed for a particular task. 

Quantization  involves some information loss which can be reduced with finer alphabets at the expense of increased computational complexity (See Section~\ref{sec-quant}). We use quantization schemes (See Fig. \ref{figS1}) which require no domain expertise.
%
\subsection{Inverting and combining hidden models} 
Quantized Stochastic Processes (QSPs) which capture the statistical structure of  symbolic streams can be  modeled using probabilistic automata, provided the processes are ergodic and stationary~\cite{CL12,crutchfield87,crutchfield94}.  
For the purpose of computing our similarity metric, we require that the number of states in the automata  be finite ($i.e.$ we only assume the existence of a generative Probabilistic Finite State Automata (PFSA)); we do not attempt to construct explicit models or require knowledge of either the exact number of states or  any explicit  bound thereof (See Fig.~\ref{fig2}).

A slightly restricted subset of the space of all PFSA over a fixed alphabet admits an Abelian group structure (See Section~\ref{sec-algstruct}); wherein the operations of commutative addition and inversion are well-defined. A trivial example of an Abelian group is the set of reals with the usual addition operation; addition of real numbers is  commutative and each real number $a$ has a unique inverse $-a$, which when summed produce the unique identity $0$. We have previously discussed the Abelian   group structure on PFSAs in the context of model selection~\cite{CWR10}. Here, we show that key group operations, necessary for classification,  can be carried out on the observed sequences alone, without any state synchronization or reference to the hidden generators of the sequences.

Existence of a group structure implies that given  PFSAs $G$ and $H$, sums $G+H, G-H$, and  unique  inverses $-G$ and $-H$ are well-defined. Individual symbols have no notion of a ``sign'', and hence the models $G$ and $-G$ are \textit{not} generators of sign-inverted sequences which would not make sense as our generated sequences   are  symbol streams. For example, the anti-stream of  a sequence $10111$ is not $-1 \ 0 \ -1 \ -1 \  -1$, but a fragment that has inverted statistical properties in terms of the occurrence patterns of the symbols $0$ and $1$ (See Table~\ref{tab1}).  For  a PFSA $G$, the unique inverse $-G$ is the PFSA which when added to $G$ yields the group identity $W=G+(-G)$, $i.e.$, the zero model. Note, the zero model $W$ is characterized by the property  that for any arbitrary PFSA $H$ in the group, we have $H+W=W+H=H$. 

For any fixed alphabet size, the zero model  is the unique  single-state PFSA (up to minimal description~\cite{CR08}) that generates  symbols as consecutive  realizations of independent random variables with uniform distribution over the symbol alphabet.
Thus $W$ generates flat white noise (FWN), and the  entropy rate of FWN achieves the theoretical upper bound among the sequences generated by arbitrary PFSA in the model space. Two PFSAs $G,H$ are identical if and only if $G+(-H)=W$.
%
%
\subsection{Metric Structure on Model Space} 
In addition to the Abelian group, the PFSA space admits a metric structure (See Section~\ref{sec-pfsa}). 
The distance between two models thus can be  interpreted as the deviation of their group-theoretic difference from a FWN process. Information annihilation exploits the possibility of estimating  causal similarity between observed data streams by estimating this distance from the observed sequences alone without requiring the  models themselves.

We can estimate the distance of the hidden generative model from FWN given only  an observed stream $s$. This is achieved by the  function $\hat{\zeta}$ (See Table~\ref{tab1}, row 4).
Intuitively, given an  observed sequence fragment $x$, we first compute the deviation of the distribution of the next
 symbol from the uniform distribution over the alphabet.  $\hat{\zeta}(s,\ell)$ is the sum of these 
deviations for all historical fragments $x$  with length up to $\ell$, weighted by $1/\vert \Sigma \vert^{2\vert x \vert}$. The weighted sum ensures that deviation of the distributions for longer  $x$ have smaller contribution to $\hat{\zeta}(s,\ell)$, which addresses the issue that the occurrence frequencies of longer sequences are  more variable.
%
\iffigs
\begin{table*}[t]
\centering 
\captionN{Application Problems, \& Results$^\ddag$}\label{tab2}
\vskip .4em
\input{Figures/sc_tab2b.tex}
\flushleft{\footnotesize $^\star$~See Section~\ref{sec-quant} for details on choosing quantization schemes}
\end{table*}
\fi
\section{Key Insight: The Information Annihilation Principle}
Our key insight is the following: two sets of sequential observations have the same generative process if the \textit{inverted} copy of one can \textit{annihilate} the statistical information contained in the other. We claim, that given two symbol streams $s_1$ and $s_2$, we can check if the underlying PFSAs (say $G_1,G_2$) satisfy the \textit{annihilation equality}:
$G_1+(-G_2)=W $ 
without explicitly knowing or constructing the models themselves. 

Data smashing is predicated on being able to invert and sum streams, and to compare streams to noise. Inversion generates a stream $s'$ given a stream $s$, such that if PFSA $G$ is the source for $s$, then $-G$ is the source for $s'$. Summation collides two streams: Given streams $s_1$ and $s_2$, generate a new stream $s'$ which is a realization of FWN if and only if the hidden models $G_1,G_2$ satisfy $G_1+G_2=W$. Finally, deviation of a stream $s$ from that generated by a FWN process can be calculated directly.

 Importantly, for a stream $s$ (with generator $G$), the inverted stream $s'$ is not unique.
Any symbol stream generated from the inverse model $-G$ qualifies as an inverse for $s$; thus anti-streams are non-unique. What is indeed  unique  is the generating inverse PFSA model. Since, our technique compares the hidden stochastic processes and not their possibly  non-unique realizations, the non-uniqueness of anti-streams is not problematic.

%

Despite the possibility of mis-synchronization between hidden model states, applicability of the algorithms shown in Table~\ref{tab1} for disambiguation of hidden dynamics is valid. We show in Section~\ref{sec-stropr}   that the algorithms evaluate  distinct models to be  distinct, and nearly identical hidden models to be nearly identical.

Estimating the deviation of a stream from FWN is straightforward (as specified by $\hat{\zeta}(s,\ell)$ in Table~\ref{tab1}, row 4). All subsequences of a given length must necessarily occur with the same frequency for a FWN process; and we simply estimate the deviation from this behavior in the observed sequence. The other two tasks are carried out via selective erasure of symbols from the input stream(s) (See Table~\ref{tab1}, rows 1-3). For example, 
 summation of streams is realized as follows: given two streams $s_1,s_2$, we read a symbol from each stream, and if they match then we copy it to our output, and ignore the  symbols read when they do not match. 

Thus, data smashing  allows us to manipulate streams via selective erasure, to  estimate a distance between the hidden stochastic sources.  
Specifically, we estimate the degree to which the sum of a stream and its anti-stream brings the entropy rate of the resultant stream close to its theoretical upper bound. 
%
\iffigs
\begin{figure*}[t] 
\vspace{-10pt}

\centering 
\newcommand{\convfig}[6]{%
\begin{loglogaxis}[
font=\bf\fontsize{6}{6}\selectfont,
axis line style={black, opacity=0.5,  thick, rounded corners=0pt},
axis on top=true, 
scale=1.2,
enlargelimits=false, 
width=#1, 
height=#2,     
xmin=#3,xmax=#4, 
semithick,
scale only axis,
 axis background/.style={
shade,top color=gray!30,bottom color=white},
xlabel={\# of symbols read},
yticklabel style={xshift=.025in,}, ylabel style={ align=center, yshift=-.07in},
ylabel={Self-annihilation \\ Error},
 #6, title style={yshift=-.1in},];
\input{#5}\end{loglogaxis}}
\begin{tikzpicture}[font=\bf \fontsize{6}{6}\selectfont]
\begin{loglogaxis}[
font=\bf\fontsize{6}{6}\selectfont,
axis line style={black, opacity=0.5,  thick, rounded corners=0pt},
axis on top=true, 
scale=1.2,
enlargelimits=false, 
width=.9in, 
height=.75in,     
xmin=100,xmax=4000, 
semithick,
scale only axis,
 axis background/.style={
shade,top color=gray!30,bottom color=white},
xlabel={\# of symbols read},
yticklabel style={xshift=.025in,}, ylabel style={ align=center, yshift=-.07in},
ylabel={Self-annihilation \\ Error},
 name=X1,title={(i) EEG }, ylabel style={yshift=-.1in}, xlabel style={yshift=.05in}, title style={yshift=-.1in},];
\input{Figures/convEEG.tex}\end{loglogaxis}
\begin{loglogaxis}[
font=\bf\fontsize{6}{6}\selectfont,
axis line style={black, opacity=0.5,  thick, rounded corners=0pt},
axis on top=true, 
scale=1.2,
enlargelimits=false, 
width=.9in, 
height=.75in,     
xmin=1000,xmax=121000, 
semithick,
scale only axis,
 axis background/.style={
shade,top color=gray!30,bottom color=white},
xlabel={\# of symbols read},
yticklabel style={xshift=.025in,}, ylabel style={ align=center, yshift=-.07in},
ylabel={Self-annihilation \\ Error},
 name=X2,   at=(X1.east),xshift=-.075in, anchor=west, xshift=.45in , ylabel={} ,title={(ii) Heart sound}, xlabel style={yshift=.05in}, title style={yshift=-.1in},];
\input{Figures/convheart.tex}\end{loglogaxis}
\begin{loglogaxis}[
font=\bf\fontsize{6}{6}\selectfont,
axis line style={black, opacity=0.5,  thick, rounded corners=0pt},
axis on top=true, 
scale=1.2,
enlargelimits=false, 
width=.9in, 
height=.75in,     
xmin=20,xmax=190, 
semithick,
scale only axis,
 axis background/.style={
shade,top color=gray!30,bottom color=white},
xlabel={\# of symbols read},
yticklabel style={xshift=.025in,}, ylabel style={ align=center, yshift=-.07in},
ylabel={Self-annihilation \\ Error},
 name=X3,   at=(X2.east),xshift=-.075in, anchor=west , xshift=.45in , ylabel={},title={(iii) Photometry}, xlabel style={yshift=.05in}, title style={yshift=-.1in},];
\input{Figures/convogle.tex}\end{loglogaxis}
\node [anchor=north west ] at ([yshift=.35in, xshift=-.3in]X1.north west) (T1xx){ \large  A \small Covergence of Self-Annihilation Error};
\node [anchor=north west ] at ([yshift=.35in, xshift=4.5in]X1.north west) (T1xx){ \large  B \small Computation time for annihilation};

 \pgfplotsset{every axis legend/.append style={
at={(0.35,0.35)},
anchor=south}}
\begin{axis}[ name=X4, at=(X3.east), xshift=.75in, anchor=west,legend cell align=left,legend style={draw=none,fill=none,,text opacity=1,font= \fontsize{6}{6}\selectfont},
axis line style={black, opacity=0.5,  thick, rounded corners=0pt},
axis on top=true, 
scale=1,grid style={dashed, gray!30},
enlargelimits=false, 
width=1.6in, 
height=.9in,     
xmin=100,xmax=100000, 
ymin=0,ymax=17,
semithick,grid, 
scale only axis,
 axis background/.style={
shade,top color=gray!30,bottom color=white},
xlabel={Symbol Length [No. of symbols]},xlabel style={yshift=.05in},
yticklabel style={xshift=.025in},
ylabel={Computation Time [ms]},
ylabel style={yshift=-.3in},
    scaled x ticks = false,
      x tick label style={/pgf/number format/fixed,
      /pgf/number format/1000 sep = \thinspace 
      }
  ];
\input{Figures/time1.txt}
\input{Figures/time2.txt}
\input{Figures/time3.txt}
\input{Figures/time4.txt}
\legend{$\vert \Sigma\vert=2$ ($2$ states) , $\vert \Sigma\vert=2$ ($4$ states) , $\vert \Sigma\vert=3$ , $\vert \Sigma\vert=4$ }
\end{axis}

\end{tikzpicture} 
\vspace{-5pt} 

\captionN{\textbf{Computational complexity and convergence rates for information annihilation.} (A) Illustrates exponential convergence of the self-annihilation error for a small set of data series for different applications (plate (i) for EEG data, plate (ii) for heart sound recordings, and plate (iii) for photometry). (B) Computation times for carrying out annihilation using the circuit shown in Fig.~\ref{fig2}D as a function of the length of  input streams for different alphabet sizes (and for diffrent number of states in the hidden models). Note that the asymptotic time complexity of obtaining the similarity distances scales as $O(\vert \Sigma \vert n)$, where $n$ is the length of the shorter of the two input streams. 
}\label{fig4}
\end{figure*}
\fi

\iffigs
 \begin{figure*}[t]
\flushleft\input{Figures/sc_figthreeE.tex}
\vspace{-105pt}
\captionN{\textbf{Data smashing applications.} Pairwise distance matrices, identified clusters and 3D projections of  Euclidean embeddings for epileptic pathology identification (shown in (A)), identification of heart murmur (shown in (B)), and  classification of variable stars from photometry (shown in (C)). In these applications, the relevant clusters are found unsupervised. 
}\label{fig3}
\end{figure*}
\fi

\subsection{Contrast with Feature-based State of Art}
Contemporary research in 
machine learning is dominated by the search for good ``features''~\cite{Duda00}, which are typically understood to be heuristically  chosen discriminative attributes characterizing objects or phenomena of interest. 
Finding such attributes is not easy~\cite{brum11,bar11}. Moreover, the number of  characterizing features $i.e.$ the size of the feature set, needs to be relatively small to avoid  intractability of the subsequent learning algorithms.  Additionally, their heuristic definition precludes any notion of optimality; it is impossible to quantify the quality of a given feature set in any absolute terms; we can only compare how it performs in the context of a specific  task  against a few selected variations.  

In addition to the heuristic nature of feature selection, machine learning algorithms typically necessitate  the choice of a distance metric in the feature space. For example, the classic  ``nearest neighbor'' k-NN classifier~\cite{CH67} requires definition of proximity, and the k-means algorithm~\cite{Mac67} depends on pairwise distances  in the feature space for clustering. 
%
To side-step the heuristic metric problem, recent approaches often learn appropriate metrics directly from data, attempting to ``back out'' a metric from side information or labeled constraints~\cite{yangan2007}. Unsupervised approaches use  dimensionality reduction and embedding strategies to uncover the geometric structure of geodesics in the feature space ($e.g.$ see manifold learning~\cite{tenenbauma2000,roweisnonlinear2000,seungcognition.2000}).
However, automatically inferred data geometry in the feature space is, again, strongly dependent on the initial choice of features. Since Euclidean distances between feature vectors are often misleading~\cite{tenenbauma2000},  heuristic features make it impossible to conceive of  a task-independent universal metric.

%
%

In contrast,  smashing  is based on an application-independent notion of similarity between quantized sample paths observed from hidden stochastic processes. Our universal metric quantifies the degree to which the summation of the inverted copy of any one  stream to the other annihilates the existing statistical dependencies, leaving behind flat white noise. We  circumvent the need for features altogether (See Fig.~\ref{fig1}B) and do not require training.

%
%
Despite the fact that the estimation of similarities between two data streams is performed in absence of the knowledge of the underlying source structure or its parameters, we establish that this universal metric is causal, $i.e.$, with sufficient data it converges to a well-defined distance between the hidden stochastic sources themselves, without ever knowing them explicitly.
\subsection{Self-annihilation Test for Data-sufficiency Check} Statistical process characteristics dictate the  amount of data  required for estimation of  the proposed  distance.
 With no access to   the hidden models, we cannot estimate the required data length   a priori; however it is possible to check for data-sufficiency for a specified error threshold via self-annihilation. Since the proposed metric is causal, the distance between two independent samples from  the same source always converges to zero. We estimate the degree of  self-annihilation achieved in order to determine data sufficiency; $i.e.$, a stream is sufficiently long if it can sufficiently annihilate an inverted self-copy to FWN.

The self-annihilation based data-sufficiency  test consists of two steps: given an observed symbolic sequence $s$, we first generate an independent copy (say $s'$). This is the independent stream copy operation (See Table~\ref{tab1}, row 1), which  can be carried out via selective symbol erasure without any knowledge of the source itself. Once we have $s$ and $s'$, we check if 
the inverted version of one annihilates the other to a pre-specified degree.
In particular, we generate $s''$ from $s$ via stream inversion, and use stream summation of $s'$ and $s''$ to produce the final output stream $s'''$, and check if $\hat{\zeta}(s''',\ell)$ is less than some specified threshold $\epsilon^\star >0$.
We show  that considering only histories up to a length $\ell= \frac{\ln(1/\epsilon^\star)}{\ln(\vert \Sigma \vert)}$ in the computation of $\hat{\zeta}(s''',\ell)$ is sufficient (See Section~\ref{sec-stropr}).

 
The self-annihilation error is also useful to rank the effectiveness of different quantization schemes. Better quantization  schemes ($e.g.$ ternary instead of binary) will be able to  produce better self-annihilation while maintaining the ability to discriminate different streams (See Section~\ref{sec-quant}).
\section{Feature-free Classification and Clustering}
Given $n$ data streams $s_1,\cdots,s_n$, we construct a  matrix $E$, such that $E_{ij}$ represents the estimated  distance between the streams $s_i, s_j$.
Thus, the diagonal elements of $E$ are the self-annihilation errors, while the off-diagonal elements represent inter-stream similarity estimates (See  Fig.~\ref{fig2}D for the basic annihilation circuit).
 Given a positive threshold $\epsilon^\star >0$, the self-annihilation tests are passed if $\epsilon_{kk} \leqq \epsilon^\star$ ($k=i,j$), and for sufficient data the streams $s_i,s_j$ have identical sources with  high probability if and only if $\epsilon_{ij} \leqq \epsilon^\star$. Once $E$ is constructed, we can determine clusters by rearranging $E$ into prominent diagonal blocks. Any standard technique~\cite{Ward63} can be used for such clustering; information annihilation is only used to find the causal  distances between observed data streams, and the resultant distance matrix can then used as input to state-of-the-art clustering methodologies, or finding geometric structures (such as lower dimensional embedding manifolds~\cite{tenenbauma2000}) induced by the similarity metric on the data sources.

 The matrix $H$, obtained from $E$ by setting the diagonal entries to zero, estimates a distance matrix. 
An Euclidean   embedding~\cite{SS85} of $H$ then leads to deeper insight into the geometry of the space of the hidden generators, $e.g.$, in the case of the EEG data, the  time series' describe a one-dimensional manifold (a curve), with  data from  similar  phenomena clustered together along the curve (See Fig.~\ref{fig3}A(ii)).

\subsection{Computational Complexity \& Data Requirements}
 The asymptotic time complexity 
of carrying out the stream operations  scales linearly with input length, and the granularity of the alphabet (See Section~\ref{sec-stropr} and  Fig.~\ref{fig4}B for illustration of the linear time complexity of  estimating  inter-stream similarity).

To pass the self-annihilation test, a data stream must be sufficiently long; and  the required length $\vert s \vert $ of the  input $s$ with a specified threshold $\epsilon^\star$ is dictated by the characteristics of the generating process. Selective erasure in annihilation (See Table~\ref{tab1}) implies that the output  tested for being FWN is shorter compared to the input stream, and the expected shortening ratio $\beta$ can be explicitly computed (See Section~\ref{sec-stropr}).We refer to  $\beta$ as the \textit{annihilation efficiency}, since the convergence rate of the self-annihilation error scales as $1/\sqrt{\beta\vert s\vert}$. In other words, the required length $\vert s \vert$  of the data stream to achieve a self-annihilation error of $\epsilon^\star$ scales as  $1 / \beta (\epsilon^\star)^2$.
It is important to note that our analysis shows that the annihilation efficiency is independent of the descriptional complexity of the process, $e.g.$, in Fig.~\ref{fig4C} the self-annihilation error for  a simpler two state process converges faster to a  four state process. However the convergence rate always scales as $O(1/\sqrt{\vert s\vert})$ as dictated by the the Central Limit Theorem (CLT)~\cite{fell45}.
\subsection{Limitations \& Assumptions}
 Data smashing is not useful in problems which do not require a notion of similarity, $e.g.$,  predicting the future course of a time series, or
analyzing a data set to pinpoint the occurrence time of an event of interest. 

For problems to which smashing is applicable, we  implicitly assume the \textit{existence} of  PFSA generators; although we never find these models explicitly. It follows that what we actually assume is not any particular modeling framework, but that the systems of interest satisfy the properties of  ergodicity, stationarity, and have a finite (but not a priori bounded) number of  states (See Section~\ref{sec-pfsa}). In practice,  our technique performs well even if these properties are  only approximately satisfied ($e.g.$ quasi-stationarity instead of stationarity, see example in Section~\ref{sec-addlb}).  The algebraic structure of the space of PFSAs (in particular, existence of  unique group inverses) is key to the information annihilation principle; however we argue that any quantized ergodic stationary stochastic process is indeed representable as a probabilistic automata (See Section~\ref{sec-pfsa}).

Data smashing is not applicable to data from strictly deterministic  systems. Such  systems are representable by probabilistic automata; however  transitions occur with probabilities which are either  $0$ or $1$. PFSAs with zero-probability transitions are non-invertible, which invalidates  the underlying theoretical guarantees (See Section~\ref{sec-algstruct}). 
Similarly, data streams in which some alphabet symbol is exceedingly rare would be difficult to invert (See Section~\ref{sec-stropr} for the notion of annihilation efficiency).

Symbolization  invariably introduces quantization  error.  This can be made  small by using larger alphabets.  However, larger alphabet sizes  demand 
 longer observed sequences (See Section~\ref{sec-stropr}, Fig.~\ref{figS5}), implying that the length of observation limits the quantization granularity, and in the process limits the 
degree to which the quantization error can be mitigated. Importantly, with  coarse quantizations distinct processes may evaluate to be similar. However, identical processes will still evaluate to be identical (or nearly so), provided the streams pass the self-annihilation test. The self-annihilation test thus offers an application-independent way to compare and rank quantization schemes (See Section~\ref{sec-quant}).

The algorithmic steps (See Table~\ref{tab1})  require no synchronization (we can start reading the streams anywhere), implying that non-equal length of time-series,  and   phase mismatches are of no consequence.

\section{Application Examples}\label{sec-appl}
Data smashing begins with  quantizing  streams to symbolic sequences, followed by the use of  the annihilation circuit (Fig.~\ref{fig2}D) to compute pairwise causal similarities.  Details of the quantization schemes, computed distance matrices, and identified clusters and Euclidean embeddings are summarized  in Table~\ref{tab2} and Fig.~\ref{fig3}. 

Our first application is classification of brain electrical activity from
different physiological and pathological brain states~\cite{ALM01}.
We used sets of electroencephalographic (EEG) data series consisting of  surface EEG recordings from healthy volunteers
with eyes closed and  open, and intracranial  recordings from epilepsy patients during  seizure free
intervals from within and from outside the seizure generating area, as well as intracranial  recordings of seizures. 

 Starting with the data series    of  electric potentials, we generated sequences of  relative changes between consecutive  values before quantization. This step  allows  a common alphabet for sequences  with wide variability in the sequence  mean  values. 

The distance matrix from  pairwise smashing yielded clear clusters corresponding to seizure, normal eyes open (EO), normal eyes closed (EC) and epileptic pathology in non-seizure conditions. (See Fig.~\ref{fig3}A,  seizures not shown due to large differences from the rest). 

Embedding the distance matrix (See Fig.~\ref{fig3}A, plate (i)) yields a one-dimensional manifold (a curve), with contiguous segments corresponding to  different brain states, $e.g.$,  right hand side of plane A  correspond to epileptic pathology. This provides a particularly insightful picture, which eludes  complex non-linear modeling\cite{ALM01}.

Next we classify cardiac rhythms from noisy heat-sound data recorded using a digital stethoscope~\cite{ps11}. We analyzed $65$ data series (ignoring the labels) corresponding to healthy rhythms and murmur, to verify if we could identify clusters without supervision that  correspond to the expert-assigned  labels. 

We found $11$ clusters in the distance matrix (See Fig.~\ref{fig3}B), $4$ of which consisted of mainly data with murmur (as determined by the expert labels), and the rest consisting of mainly healthy rhythms (See Fig.~\ref{fig3}B, plate (iv)). Classification precision for murmur is noted in Table~2 ($75.2\%$).  Embedding of the  distance matrix revealed a two dimensional manifold (See Fig.~\ref{fig3}B, plate (iii)).

Our next problem is the classification of  variable stars using   light intensity  series (photometry) from  the Optical Gravitational Lensing Experiment (OGLE) survey~\cite{ogle06}. Supervised classification of photometry proceeds by first ``folding'' each light-curve to its known period to correct phase mismatches. 
 In our first analysis, we started with  folded light-curves; and generated data series of the relative changes between consecutive brightness values in the curves before quantization, which allows for the use of a common alphabet for light curves with wide variability in the mean brightness values. 
 Using data for Cepheids and RRLs (3426 Cepheids, 7273 RRL), we obtained a classification accuracy of $99.8\%$ which marginally outperforms the state of art (See Table~\ref{tab2}). Clear clusters (obtained unsupervised) corresponding to the two classes can be seen in the  computed distance matrix (See Fig.~\ref{fig3}C, plate (i)), and the 3D projection of its Euclidean embedding (See Fig.~\ref{fig3}C, plate (ii)).
The  3D embedding was very nearly  constrained within a 2D manifold (See Fig. \ref{fig3}C plate (ii)).

Additionally, in our second analysis, we asked if data smashing can  work without knowledge of the period of the variable star; skipping the folding step.
Smashing  raw photometry data  yielded a classification accuracy of $94.3\%$ for the two classes (See Table~\ref{tab2}). This direct approach is beyond state of the art techniques.

Our fourth application is biometric authentication using visually evoked EEG potentials (VEP). The public database used~\cite{eega}. considered 122 subjects, each of whom was exposed to  pictures of objects chosen from the standardized Snodgrass  set~\cite{snod80}. 

Note that while this application is supervised (since we are not attempting to find clusters  unsupervised), no actual training  is involved; we merely mark the randomly chosen subject-specific set of data series as the  library set representing each individual subject.
 If ``unknown'' test data series is smashed  against each element  of each of  the libraries corresponding to the individual subjects,  we expected that the data series from the same subject will annihilate each other correctly, while those from different subjects will fail to do so to the same extent. 
 We outperformed the state of  art for both kNN and SVM based approaches (See Table~\ref{tab2}). 

Our fifth application is text independent  speaker  identification  using the ELSDSR database~\cite{elsdsr}, which includes  recording from 23 speakers (9 female, and 14 male, with possibly non-native accents). As before, training  involved  specifying the library series for each speaker. We computed 
the distance matrix by smashing the library data series against each other, and trained a simple kNN on the Euclidean embedding of the distance matrix. The test data then yielded a classification accuracy of $80.2\%$, which beat the state of art figure of $73.73\%$ for $2s$ snippets of recording data~\cite{elsdsr2} (See Table~\ref{tab2}).

In the  suceeding sections, we develop the mathematical details of the information annihilatin principle, and establish the correctness of the data smashing algorithm. Section~\ref{sec-pfsa} presents the theory  of  probabilistic automata  as  a modeling framework for ergodic stationary quantized stochastic processes. Section~\ref{sec-algstruct} describes the relevant algebraic structures, including that of an Abelian group, definable on  the space of probabilistic automata. This is central to the notion of anti-streams. Section~\ref{sec-stropr} then establishes that the stream operations delineated in Table~\ref{tab1} are indeed correct. Section~\ref{sec-quant} discusses quantization schemes; specifically describing how to choose the granularity of the quantization. Section~\ref{sec-addl} expounds the differences between the data smashing approach and some specific standard notions often used to quantify statistical dependencies, $e.g.$ mutual information between data streams. We also discuss a specific example to illustrate that simple statistical features may miss important dynamical artifacts in data, which is easily revealed via data smashing. The paper is summarized and concluded in Section~\ref{sec-conc}.

\section{Stochastic Processes \& Probabilistic Automata}\label{sec-pfsa}
To establish the correctness of the data smashing algorithm, we  first establish  the possibility of using probabilistic automata to model stationary, ergodic processes. Our automata models~\cite{CL12} are distinct to those reported in the literature~\cite{P71,VTCC05}. We include a brief overview here for the sake of completeness.

\begin{notn} 
 $\Sigma$  denotes a finite alphabet  of symbols. The set of all finite but possibly unbounded strings on $\Sigma$ is denoted by $\Sigma^\star$~\cite{HMU01}. The set of finite strings over $\Sigma$ form a concatenative monoid, with the empty word $\lambda$ as identity. 
The set of strictly infinite strings on $\Sigma$ is denoted as $\Sigma^\omega$, where $\omega$ denotes the first transfinite cardinal. 
For a string $x$,  $\vert x \vert$ denotes its length, and for a set $A$,    $\vert A \vert$ denotes its cardinality. 
\end{notn}

\begin{defn}[QSP]\label{defQSP}
A QSP $\mathcal{H}$ is a discrete time $\Sigma$-valued strictly stationary, ergodic stochastic process, $i.e.$ 
\cgather{
\mathcal{H} = \left \{ X_t: X_t \textrm{ is a $\Sigma$-valued random variable}, t \in \mathbb{N}\cup \{0\} \right \}
} 
A  process is ergodic if  moments may be calculated from a sufficiently long realization, and strictly stationary if moments are time-invariant.
\end{defn}
We next formalize the connection of QSPs to PFSA generators.
  We develop the theory assuming multiple realizations of the QSP $\mathcal{H}$, and  fixed initial conditions. Using ergodicity, we will be then able to apply our construction to a single sufficiently long realization, where  initial conditions cease to matter.
\begin{defn}[$\sigma$-Algebra On Infinite Strings]
 For the set of infinite strings on  $\Sigma$, we define $\mathfrak{B}$ to be the smallest $\sigma$-algebra generated by the family of sets $\{  x \Sigma^\omega : x \in \Sigma^\star\}$.
\end{defn}
\begin{lem}\label{QSPtoProb}
Every QSP  induces a  probability space $(\Sigma^\omega,\mathfrak{B},\mu)$.
\end{lem} 
\begin{proof}
Assuming stationarity, we can construct a probability measure $\mu: \mathfrak{B} \rightarrow [0,1]$  by defining
for any sequence $x\in \Sigma^\star\setminus \{\lambda\}$, and a sufficiently large number of realizations $N_R$ (assuming ergodicity):
\cgathers{
 \mu( x \Sigma^\omega) = \lim_{N_R \rightarrow \infty}\frac{ \textrm{\small \#  of initial occurrences of $x$}}{\begin{array}{c} \textrm{\small \#  of initial occurrences} \\\textrm{\small of all sequences of length  $\vert x \vert$}\end{array}}
}
and extending the measure to elements of $\mathfrak{B} \setminus B$ via at most countable sums. Thus $\mu(\Sigma^\omega) = \sum_{x \in \Sigma^\star} \mu( x \Sigma^\omega) = 1$, and for the null word  $\mu(\lambda \Sigma^\omega) =  \mu(\Sigma^\omega) = 1$.
\end{proof}

 \begin{notn}
 For notational brevity, we denote $\mu( x \Sigma^\omega)$ as $Pr(x)$.
 \end{notn}

Classically,  automaton states are equivalence classes for the  Nerode relation;  two strings are  equivalent if and only if any finite extension of the strings is either both in the language under consideration, or neither are~\cite{HMU01}. We use a probabilistic extension~\cite{CR08}.

\begin{defn}[Probabilistic Nerode Equivalence Relation]\label{defnerode} $(\Sigma^\omega,\mathfrak{B},\mu)$ induces an equivalence relation $\sim_{N}$ on the set of finite strings $\Sigma^\star$ as:
\cgather{
\forall x,y \in \Sigma^\star, 
\smash{x \sim_{N} y \iff \forall z \in \Sigma^\star \bigg (\big (} Pr(xz)=Pr(yz)=0 \big )  \notag \\  \bigvee \big \vert  Pr(xz)/Pr(x)- Pr(yz)/Pr(y)\big \vert =0\bigg )
}
\end{defn}

\begin{notn}
For $x \in \Sigma^\star$,  the equivalence class of $x$ is  $[x]$.
\end{notn}

It is easy to see   that $\sim_{N}$ is right invariant, $i.e.$ 
\cgather{
x \sim_{N} y \Rightarrow \forall z \in \Sigma^\star, xz \sim_{N} yz 
}
A right-invariant equivalence on $\Sigma^\star$ always induces an automaton structure; and hence the probabilistic Nerode relation induces  a probabilistic automaton: states are equivalence classes of $\sim_{N}$, and the transition structure arises as follows: For states $q_i,q_j$, and  $x \in \Sigma^\star$,
\begin{gather}
([x]=q ) \wedge ([x \sigma ] = q' 
 )\Rightarrow q \xrightarrow{\sigma} q'
\end{gather}
Before formalizing the above construction, we introduce the notion of probabilistic automata with initial, but no final, states.


\begin{defn}[Initial-Marked PFSA]\label{defpfsa} An initial marked probabilistic finite state automaton (a Initial-Marked PFSA)   is a quintuple $(Q,\Sigma,\delta,\pitilde,q_0)$, where $Q$ is a finite  state set, $\Sigma$ is the alphabet,  $\delta:Q \times \Sigma \rightarrow Q$ is the  state transition function,  $\pitilde : Q \times \Sigma \rightarrow [0,1]$  specifies the conditional symbol-generation probabilities, and $q_0\in Q$ is the initial state.  $\delta$ and $\pitilde$ are recursively extended to arbitrary  $y=\sigma x \in \Sigma^\star$ as follows:
\cgather{
\forall q \in Q, \delta(q,\lambda) = q\\
\delta(q,\sigma x) = \delta(\delta(q,\sigma),x)\\
\forall q \in Q, \pitilde(q,\lambda) = 1\\
\pitilde(q,\sigma x) = \pitilde(q,\sigma)\pitilde(\delta(q,\sigma),x)
}
Additionally, we impose  that for  distinct states $q_i,q_j \in Q$, there exists a string $x \in \Sigma^\star$, such that $\delta(q_i,x) = q_j$, and $\pitilde(q_i,x) > 0$.
\end{defn}

Note that the probability of the null word is unity from each state.

If the current state and  the next symbol is  specified, our next  state is fixed; similar to Probabilistic Deterministic Automata~\cite{Gavalda06}. However, unlike the latter, we lack final states in the model. Additionally, we assume our graphs to be strongly connected.

Later we will remove initial state dependence using ergodicity. 
Next we formalize how  a PFSA arises  from  a  QSP.

\begin{lem}[PFSA Generator]\label{lemPFSAgen}
Every Initial-Marked PFSA $G=(Q,\Sigma,\delta,\pitilde,q_0)$ induces a unique probability measure $\mu_G$ on the measurable space  $(\Sigma^\omega,\mathfrak{B})$.
\end{lem}
\begin{proof}
Define   set function $\mu_G$ on the measurable space  $(\Sigma^\omega,\mathfrak{B})$:
\cgather{
\mu_G(\varnothing) \triangleq 0\\
\forall x \in \Sigma^\star, \mu_G(x\Sigma^\omega) \triangleq \delta(q_0,x)\\
\forall x,y \in \Sigma^\star, \mu_G(\{x,y\}\Sigma^\omega) \triangleq \mu_G(x\Sigma^\omega)+ \mu_G(y\Sigma^\omega)
}
Countable additivity of  $\mu_G$ is immediate, and  (See Definition~\ref{defpfsa}):
\cgather{
\mu_G(\Sigma^\omega) = \mu_G(\lambda\Sigma^\omega) = \delta(q_0, \lambda) = 1
}
implying that $(\Sigma^\omega,\mathfrak{B}, \mu_G)$ is a probability space.
\end{proof}

We refer to $(\Sigma^\omega,\mathfrak{B}, \mu_G)$ as the probability space generated by the Initial-Marked PFSA $G$. 

\begin{lem}[Probability Space To PFSA]\label{lemPROB2PFSA}
If the probabilistic Nerode relation corresponding to a  probability space $(\Sigma^\omega,\mathfrak{B}, \mu)$ has a finite index, then the latter has an  initial-marked PFSA generator.
\end{lem}
\begin{proof}
Let  $Q$ be the set of equivalence classes of the probabilistic Nerode relation (Definition~\ref{defnerode}),  and define functions $\delta:Q \times \Sigma \rightarrow Q$, $\pitilde:Q \times \Sigma \rightarrow [0,1]$ as:
\calign{
& \delta([x],\sigma) = [x\sigma]\\
& \pitilde([x],\sigma) = \frac{Pr(x'\sigma)}{Pr(x')} \textrm{ for any choice of } x' \in [x] \label{eqpit}
}
where we  extend $\delta,\pitilde$  recursively  to   $y=\sigma x \in \Sigma^\star$ as 
\cgather{\delta(q,\sigma x) = \delta(\delta(q,\sigma),x)\\\pitilde(q,\sigma x) = \pitilde(q,\sigma)\pitilde(\delta(q,\sigma),x)}
For verifying the null-word probability, choose a $x \in \Sigma^\star$ such that $[x] = q$ for some $q \in Q$. Then, from  Eq.~\eqref{eqpit}, we have:
\cgather{
 \displaystyle \pitilde(q,\lambda)= \frac{Pr(x'\lambda)}{Pr(x')} \textrm{ for any  } x' \in [x] \Rightarrow \pitilde(q,\lambda)
=  \frac{Pr(x')}{Pr(x')} = 1
}
Finite index of $\sim_{N}$ implies  $\vert Q\vert < \infty$, and hence denoting  $[\lambda]$ as $q_0$, we conclude:  $G=(Q,\Sigma,\delta,\pitilde,q_0)$ is an Initial-Marked PFSA.  Lemma~\ref{lemPFSAgen} implies that $G$ generates $(\Sigma^\omega,\mathfrak{B}, \mu)$,  which completes the proof.
\end{proof}

The above construction yields a \textit{minimal realization} for the Initial-Marked PFSA, unique up to state renaming. 


\begin{lem}[QSP to PFSA]\label{lemQSP2PFSA}
Any QSP  with a finite index Nerode equivalence is generated by an Initial-Marked PFSA.
\end{lem}
\begin{proof}
Follows immediately from Lemma~\ref{QSPtoProb} (QSP to Probability Space) and Lemma~\ref{lemPROB2PFSA} (Probability Space to PFSA generator).
\end{proof}

\subsection{Canonical Representations}
 We have defined a QSP as both ergodic and stationary, whereas the Initial-Marked PFSAs have a designated initial state. Next we introduce  canonical representations to remove initial-state dependence. We use $\Pitilde$ to denote the matrix representation of  $\pitilde$, $i.e.$, $\Pitilde_{ij} = \pitilde(q_i,\sigma_j)$,  $q_i \in Q, \sigma_j \in \Sigma$. We need the notion of transformation matrices $\Gamma_\sigma$.

\begin{defn}[Transformation Matrices]\label{defGamma}
 For an initial-marked PFSA $G=(Q,\Sigma,\delta,\pitilde,q_0)$, the symbol-specific transformation matrices $\Gamma_\sigma \in \{0,1\}^{\vert Q \vert \times \vert Q \vert}$ are:
\cgather{
\Gamma_\sigma \big \vert_{ij} = \begin{cases}
                                 \pitilde(q_i,\sigma), & \textrm{if } \delta(q_i,\sigma) = q_j \\
				 0, & \textrm{otherwise}
                                \end{cases}
}
\end{defn}
Transformation matrices have a single non-zero entry per row, reflecting our generation rule that given a state and a generated symbol, the next state is fixed. 

First, we note that, given an initial-marked PFSA $G$,  we can associate a probability distribution $\wp_x$ over the states of $G$  for each $x \in \Sigma^\star$ in the following sense:
if $x=\sigma_{r_1}\cdots \sigma_{r_m} \in \Sigma^\star$, then we have:
\cgather{
\wp_x = \wp_{\sigma_{r_1}\cdots \sigma_{r_m}} = \underbrace{\frac{1}{\vert \vert \wp_\lambda \prod_{j=1}^m \Gamma_{\sigma_{r_j}}  \vert \vert_1 }}_{\textrm{Normalizing factor}}\wp_\lambda \prod_{j=1}^m \Gamma_{\sigma_{r_j}}
}
where $\wp_\lambda$ is the stationary distribution over the states of $G$.  Note that there may exist more than one string that leads to a distribution $\wp_x$, beginning from the stationary distribution $\wp_\lambda$. Thus, $\wp_x$   is an equivalence class of strings, $i.e.$, $x$ is  not unique. 
\begin{defn}[Canonical Representation]\label{defcanon}

 An  initial-marked PFSA $G=(Q,\Sigma,\delta,\pitilde,q_0)$ uniquely induces a canonical representation $(Q^C,\Sigma,\delta^C,\pitilde^C)$, where $Q^C$ is  a subset of the  set of probability distributions over  $Q$, and  $\delta^C: Q^C \times \Sigma \rightarrow Q^C$,  $\pitilde^C: Q^C \times \Sigma \rightarrow [0,1]$ are constructed  as follows:
\begin{enumerate}
 \item Construct the stationary distribution on $Q$ using the transition probabilities of the Markov Chain induced by $G$, and include this as the first element $\wp_\lambda$ of $Q^C$. Note that the transition matrix for $G$ is the row-stochastic matrix $M \in [0,1]^{\vert Q \vert \times \vert Q \vert}$, with
$
M_{ij} = \sum_{\sigma: \delta(q_i,\sigma)=q_j}\pitilde(q_i,\sigma)
$, and hence $\wp_\lambda$ satisfies:
\cgather{
\wp_\lambda M = \wp_\lambda
}
\item Define  $\delta^C$ and $\pitilde^C$ recursively:
\calign{
&\delta^C(\wp_x, \sigma) = \frac{1}{\vert \vert \wp_x \Gamma_\sigma \vert \vert_1}\wp_x \Gamma_\sigma \triangleq \wp_{x\sigma}\\
&\pitilde^C(\wp_x,\sigma) = \wp_x \Pitilde
}
\end{enumerate}
\end{defn}
For a QSP $\mathcal{H}$, the canonical representation is denoted as $\mathcal{C}_\mathcal{H}$.
\begin{lem}[Properties of Canonical Representation]\label{lemcanonstruc}
 Given an initial-marked PFSA $G=(Q,\Sigma,\delta,\pitilde,q_0)$:
\begin{enumerate}
\item The canonical representation is independent of the initial state.
\item The canonical representation $(Q^C,\Sigma,\delta^C,\pitilde^C)$ contains a copy of $G$ in the sense that there exists a set of states $Q' \subset Q^C$, such that there exists a one-to-one map $\zeta:Q \rightarrow Q'$, with:
\cgather{\forall q \in Q, \forall \sigma \in \Sigma, \left \{ \begin{array}{l}
\pitilde(q,\sigma) = \pitilde^C(\zeta(q),\sigma) \\
\delta(q,\sigma) = \delta^C(\zeta(q),\sigma)\end{array}\right.
}
\item If during the  construction (beginning with $\wp_\lambda$) we  encounter $\wp_x = \zeta(q)$ for some $x \in \Sigma^\star$,  $q \in Q$ and any map $\zeta$ as defined in (2), then we stay within the graph of the copy of the  initial-marked PFSA  for all right extensions of $x$.
\end{enumerate}

\end{lem}
\begin{proof}
(1) follows  the ergodicity of QSPs, which makes $\wp_\lambda$ independent of the initial state in the initial-marked PFSA.

(2) The canonical representation  subsumes the initial-marked representation in the  sense that the states of the latter may themselves be seen as degenerate distributions over $Q$, $i.e.$, by letting 
\cgather{\label{eqE}
\mathcal{E}=\big \{ e^i \in [0 \ 1]^{\vert Q \vert } , i= 1,\cdots, \vert Q \vert \big \}
}
  denote  the set of distributions satisfying:
\cgather{
e^i\vert_j = \begin{cases}
              1, & \textrm{if } i=j\\
              0, & \textrm{otherwise}
             \end{cases}
}
(3) follows from the strong connectivity of $G$.
\end{proof}



Lemma~\ref{lemcanonstruc} implies that initial states are unimportant;   we may denote the initial-marked PFSA induced by a QSP $\mathcal{H}$, with the initial marking removed, as $\mathcal{P}_\mathcal{H}$, and refer to it simply as a ``PFSA''.  States in $\mathcal{P}_\mathcal{H}$ are representable as  states in $\mathcal{C}_\mathcal{H}$ as elements of $\mathcal{E}$. Next we show that  we always  encounter a state arbitrarily close to some element in  $\mathcal{E}$ (See Eq.~\eqref{eqE}) in the canonical construction starting from the stationary distribution $\wp_\lambda$ on the states of $\mathcal{P}_\mathcal{H}$.
%

Next we introduce the notion of $\epsilon$-synchronization of probabilistic automata (See Figure~\ref{figsync}). Synchronization of automata is fixing or determining the current state. Not all PFSAs are synchronizable, but all are $\epsilon$-synchronizable~\cite{CL12}.
\begin{figure}[t]
\vspace{-5pt}
\centering 
\newcommand{\Vdist}{.5in}
\def\COL{black!50}
\def\COLA{Red4!20!black}
\def\COLB{Red4!20!black}
\def\COLC{Red4!20!black}
\def\GCOL{gray!0}

\definecolor{blcol}{RGB}{150,150,250}
\definecolor{murcol}{RGB}{255,150,150}
\definecolor{healcol}{RGB}{150,235,150}
\definecolor{colzeta}{RGB}{250,250,200}
\definecolor{leftpcol}{RGB}{200,255,200}
\definecolor{rightpcol}{RGB}{200,200,255}
\definecolor{boxcol}{RGB}{210,200,200}
\definecolor{linecol}{RGB}{200,180,180}
\definecolor{colplus}{RGB}{235,220,220}
\definecolor{colinv}{RGB}{200,220,255}
\definecolor{cof}{RGB}{219,144,71}
\definecolor{pur}{RGB}{200,200,200}
\definecolor{greeo}{RGB}{91,173,69}
\definecolor{greet}{RGB}{52,111,72}
 \definecolor{nodecol}{RGB}{180,180,220}
 \definecolor{nodeedge}{RGB}{140,148,155}
 \definecolor{nodeedgeb}{RGB}{240,148,155}
  \definecolor{nodecolb}{RGB}{220,180,180}
  \definecolor{nodecolc}{RGB}{180,220,180}
  \definecolor{nodecolcD}{RGB}{100,160,100}
  \definecolor{nodecolW}{RGB}{190,190,190}
  \definecolor{edgecol}{RGB}{60,60,80}
 \definecolor{nodecolD}{RGB}{140,140,180}
  \definecolor{nodecolbD}{RGB}{180,140,140}
\tikzset{srule/.style={ semithick, opacity=.8, Red1, text opacity=1, text=gray}}
\tikzset{axstyle/.style={ black, opacity=1,  thick, rounded corners=0pt}}
\tikzset{%
fshadow/.style={      preaction={
         fill=black,opacity=.2,
         path fading=circle with fuzzy edge 20 percent,
         transform canvas={xshift=1mm,yshift=-1mm}
       }}
}
\tikzset{oplus/.style={path picture={%
      \draw[black]
       (path picture bounding box.south) -- (path picture bounding box.north) 
       (path picture bounding box.west) -- (path picture bounding box.east);
      }}} 
\tikzset{%
  highlight/.style={draw=Red4,rectangle,rounded corners=1pt, opacity=1,fill=gray!10,thick,inner sep=-1.25pt,on background layer,fill opacity=.3}
}
\tikzset{%
  highlightg/.style={draw=DodgerBlue4,rectangle,rounded corners=1pt, opacity=1,fill=gray!10,thick,inner sep=-1.25pt,on background layer,fill opacity=.3}
}
%
\def\SCALE{1.5}
\newcommand{\Gone}{%
\begin{tikzpicture}[->,>=stealth',shorten >=1pt,auto,node distance=1.2cm,
                    semithick,scale=\SCALE,font=\bf\fontsize{6}{6}\selectfont]
  \tikzstyle{every state}=[fill=nodecol,draw=nodeedge,text=black,minimum size=3, text width=2,scale=\SCALE,fshadow]
  \node[state] (A)         []           {$\mspace{-4mu}q_0$};
  \node[state]         (B) [right of=A] {$\mspace{-4mu}q_1$};
  \path (A) edge   [draw=edgecol,bend left]           node {$\sigma_1\vert 0.15$} (B)
        (A) edge [draw=edgecol,in=120,out=60,loop,left] node [xshift=-.1in,yshift=-.1in]{$\sigma_0\vert 0.85$} (A)
        (B) edge [draw=edgecol,in=60,out=120,loop,right] node [xshift=.1in,yshift=-.1in]{$\sigma_0\vert 0.25$} (B)
            edge   [draw=edgecol,bend left]           node {$\sigma_1\vert 0.75$} (A);
\end{tikzpicture}
}
\newcommand{\Gtwo}{%
\begin{tikzpicture}[->,>=stealth',shorten >=1pt,auto,node distance=1.2cm,
                    semithick,scale=\SCALE,font=\bf\fontsize{6}{6}\selectfont]
  \tikzstyle{every state}=[fill=nodecolb,draw=nodeedgeb,text=black,minimum size=3, text width=2,scale=\SCALE,fshadow]
  \node[state] (A)         []           {$\mspace{-4mu}q_0$};
  \node[state]         (B) [right of=A] {$\mspace{-4mu}q_1$};
  \path (A) edge   [draw=edgecol,bend left]           node {$\sigma_1\vert 0.15$} (B)
        (A) edge [draw=edgecol,in=120,out=60,loop,left] node [xshift=-.1in,yshift=-.1in]{$\sigma_0\vert 0.85$} (A)
        (B) edge [draw=edgecol,in=60,out=120,loop,right] node [xshift=.1in,yshift=-.1in]{$\sigma_1\vert 0.75$} (B)
            edge   [draw=edgecol,bend left]           node {$\sigma_0\vert 0.25$} (A);
\end{tikzpicture}
}

\begin{tikzpicture}[scale=1.5,font=\bf \sffamily \fontsize{7}{7}\selectfont]

\node [] (A)  {\Gtwo};
\node [] (aa) at (A.south)  {(a) Synchronizable};

\node [anchor=west] (B) at ([xshift=.02in]A.east)  {\Gone};

\node [] (bb) at (B.south)  {(b) Non-synchronizable};

\end{tikzpicture}
\vspace{-10pt}

\captionN{{\bf Synchronizable and non-synchronizable machines.} Synchronization is determination of the current state from observed past symbols. Not all PFSAs are synchronizable, $e.g.$, while the top machine is synchronizable, the bottom one is not. Note that  a history of just one symbol suffices to determine the current state in the synchronizable machine (top), while no finite history can do the same in the non-synchronizable machine (bottom). A $\epsilon$-synchronizing string always exists~\cite{CL12} for a PFSA, which is not true for deterministic automata~\cite{BICP99},\cite{Ito84}.
}\label{figsync} 
\vspace{-15pt}

\end{figure}

\begin{defn}[$\epsilon$-synchronizing Strings]\label{defepsilonsynchro}
  A  string $x\in \Sigma^\star$ is $\epsilon$-synchronizing for a PFSA if:
\cgather{
\exists\bvec  \in \mathcal{E}, \vert \vert \wp_x -\bvec  \vert \vert_\infty \leqq \epsilon
}
\end{defn}

We next introduce the notion of symbolic derivatives:
Note that, PFSA states  are not  observable; we observe  symbols generated from hidden states. A symbolic derivative at a given string  specifies the distribution of the next symbol over the alphabet.

\begin{notn}
We denote the set of probability distributions 
over a finite  set of cardinality $k$ as $\mathscr{D}(k)$. 
\end{notn}

\begin{defn}[Symbolic Count Function]\label{defcount}
 For a string $s$ over  $\Sigma$, the count function $\#^s: \Sigma^\star \rightarrow \mathbb{N}\cup \{0\}$,  counts the number of times a particular substring occurs in $s$. The count is overlapping, $i.e.$, in a string $s=0001$, we count the number of occurrences of $00$s as $\underline{00}01$ and $0\underline{00}1$, implying $\#^s 00 =2$.
\end{defn}

\begin{defn}[Symbolic Derivative]\label{defsymderivative}
 For a string $s$  generated by a QSP over $\Sigma$, the symbolic derivative  $\phi^s:\Sigma^\star \rightarrow \mathscr{D}(\vert \Sigma\vert -1)$ is defined:
\vspace{-5pt}
\cgather{
\phi^s(x) \big \vert_i = \frac{\#^s x\sigma_i}{\sum_{\sigma_i \in \Sigma }\#^s x\sigma_i}
}
Thus,  $\forall x \in \Sigma^\star, \phi^s(x)$ is a probability distribution over $\Sigma$. $\phi^s(x)$ is referred to as the symbolic derivative at $x$.
\end{defn}

Note that  $\forall q_i \in Q$, $\pitilde$  induces a  probability distribution over $\Sigma$ as  $[\pitilde(q_i,\sigma_1), \cdots , \pitilde(q_i,\sigma_{\vert \Sigma \vert})]$. We denote this as $\pitilde(q_i,\cdot)$.

We next show that the symbolic derivative at $x$ can be used to estimate this distribution for $q_i = [x]$, provided $x$ is $\epsilon$-synchronizing.
\begin{prop}[$\epsilon$-Convergence]\label{thmsymderiv} If $x \in \Sigma^\star$ is $\epsilon$-synchronizing, then:
 \cgather{
\forall \epsilon > 0,  \lim_{\vert s \vert \rightarrow \infty}\vert \vert \phi^s(x) -\pitilde([x],\cdot)\vert \vert_\infty \leqq_{a.s} \epsilon\label{eqsync2} 
}
\end{prop}
\begin{proof}
 We use the Glivenko-Cantelli theorem~\cite{Fl70} on  uniform convergence of empirical distributions. Since $x$ is $\epsilon$-synchronizing:
\cgather{
\forall \epsilon > 0, \exists\bvec  \in \mathcal{E}, \vert \vert \wp_x -\bvec  \vert \vert_\infty \leqq \epsilon
}
Recall that $\mathcal{E}=\big \{ e^i \in [0 \ 1]^{\vert Q \vert } , i= 1,\cdots, \vert Q \vert \big \}$  denotes  the set of distributions over $Q$ satisfying:
\cgather{
e^i\vert_j = \begin{cases}
              1, & \textrm{if } i=j\\
              0, & \textrm{otherwise}
             \end{cases}
}
Let  $x$ $\epsilon$-synchronize to $q \in Q$. Thus, when we encounter $x$ while 
reading  $s$, we are guaranteed to be distributed over  $Q$ as $\wp_x$, where:
\cgather{
 \vert \vert \wp_x -\bvec  \vert \vert_\infty \leqq \epsilon
\Rightarrow \wp_x = \alpha \bvec +(1-\alpha) u
}
where $ \alpha \in [0,1]$, $\alpha \geqq 1 - \epsilon$, and $u$ is an unknown distribution over $Q$. Defining $A_\alpha = \alpha \pitilde(q,\cdot) + (1-\alpha) \sum_{j=1}^{\vert Q\vert}u_j \pitilde(q_j,\cdot)
$, we note that $\phi^s(x)$ is an empirical distribution for $A_\alpha$, implying:
 \caligns{
 &\lim_{\vert s \vert \rightarrow \infty}\vert \vert \phi^s(x) - \pitilde(q,\cdot) \vert \vert_\infty 
 =\lim_{\vert s \vert \rightarrow \infty}
\vert \vert \phi^s(x) - A_\alpha + A_\alpha 
- \pitilde(q,\cdot) \vert \vert_\infty \\
 & \leqq \overbrace{\lim_{\vert s \vert \rightarrow \infty}   \vert \vert \phi^s(x) -A_\alpha \vert\vert_\infty }^{\textrm{\scriptsize  a.s. $0$ by Glivenko-Cantelli}}  +  \lim_{\vert s \vert \rightarrow \infty}\vert \vert  A_\alpha - \pitilde(q,\cdot) \vert \vert_\infty \\
&\leqq_{a.s}  (1-\alpha) \left ( \vert \vert\pitilde(q,\cdot) - u \vert\vert_\infty  \right ) \leqq_{a.s}  \epsilon
 }
This completes the proof.
\end{proof}

The notion of canonical representations, along with that of the symbolic derivatives will be used to establish the correctness of the 
stream operations in Section~\ref{sec-stropr}. Note that the canonical representation is free from the notion  of initial states; intuitively this translates to our ability to carry out the stream operations (Table 1, main text) without knowledge of the initial states of the hidden models. The notion of the symbolic derivatives, along with Proposition~\ref{thmsymderiv} establishes that if the derivatives computed from two sufficiently long observed sequences $s_1,s_2$ match up closely, then the underlying generative PFSAs are also close. The detailed formulation in \cite{CL12} proves that we can conclude that the distance between these underlying models is small with a high probability (in the PAC sense).

We also need to briefly describe the concept of a metric on the space of probabilistic automata established in \cite{CL12}.
\begin{prop}[Metric For Probabilistic Automata]\label{propmetric}
 For two strongly connected PFSAs $G_1,G_2$, denote the symbolic derivative at $x\in \Sigma^\star$ as $\phi^s_{G_1}(x)$ and $\phi^s_{G_2}(x)$ respectively. Then, 
\cgathers{
\Theta(G_1,G_2) =\frac{\vert \Sigma \vert -1}{\vert \Sigma\vert} \lim_{\begin{subarray}{c}\vert s_1\vert \rightarrow \infty, \\ \vert s_2\vert \rightarrow \infty\end{subarray}}
\sum_{x \in \Sigma^\star}\left \{  \frac{\vert \vert \phi^{s_1}_{G_1}(x) - \phi^{s_2}_{G_2}(x)  \vert \vert_\infty}{\vert \Sigma \vert^{2\vert x \vert}}\right \} 
}
defines  a metric on the space of probabilistic automata on $\Sigma$.
\end{prop}
\begin{IEEEproof}
 The above metric is slightly different from the one introduced in \cite{CL12}. However, the proof of the metric properties follows almost identically.
\end{IEEEproof}
The following result is immediate, and justifies the expression given in Table 1 of main text (Row 4).

\begin{cor}[For Proposition~\ref{propmetric}]\label{cormetricbounds}
 For any two PFSA $G_1,G_2$:
\cgather{
0 \leqq \Theta(G_1,G_2) \leqq 1
}
\end{cor}
\begin{IEEEproof}
 The lower bound is immediate by setting $G_1=G_2$. For the upper bound, we note:
\cgathers{
\Theta(G_1,G_2) =\frac{\vert \Sigma \vert -1}{\vert \Sigma\vert} \lim_{\begin{subarray}{c}\vert s_1\vert \rightarrow \infty, \\ \vert s_2\vert \rightarrow \infty\end{subarray}}
\sum_{x \in \Sigma^\star}\left \{  \frac{\vert \vert \phi^{s_1}_{G_1}(x) - \phi^{s_2}_{G_2}(x)  \vert \vert_\infty}{\vert \Sigma \vert^{2\vert x \vert}}\right \} \\
\leqq \frac{\vert \Sigma \vert -1}{\vert \Sigma\vert} \lim_{\begin{subarray}{c}\vert s_1\vert \rightarrow \infty, \\ \vert s_2\vert \rightarrow \infty\end{subarray}}
\sum_{x \in \Sigma^\star}\left \{  \frac{\max \left( \vert \vert \phi^{s_1}_{G_1}(x) - \phi^{s_2}_{G_2}(x)  \vert \vert_\infty\right )}{\vert \Sigma \vert^{2\vert x \vert}}\right \}\\
= \frac{\vert \Sigma \vert -1}{\vert \Sigma\vert} 
\sum_{x \in \Sigma^\star} \frac{1}{\vert \Sigma \vert^{2\vert x \vert}} = \frac{\vert \Sigma \vert -1}{\vert \Sigma\vert} 
\sum_{k=0}^\infty \frac{\vert \Sigma \vert^{k}}{\vert \Sigma \vert^{2k}}= \frac{\vert \Sigma \vert -1}{\vert \Sigma\vert} 
\sum_{k=0}^\infty \frac{1}{\vert \Sigma \vert^{k}}
}
where the last two steps follow from the fact that there are $\vert \Sigma \vert^{\vert x \vert }$ strings of length $\vert x \vert$, which allows us to replace the sum over $x \in \Sigma^\star$ to over $k=\vert x \vert$. Finally, noting that 
$\frac{\vert \Sigma \vert -1}{\vert \Sigma\vert} 
\sum_{k=0}^\infty \frac{1}{\vert \Sigma \vert^{k}} =1$, completes the proof.
\end{IEEEproof}

Next, we elucidate the relevant algebraic structures on the space of PFSA.

\section{Algebraic Structures On PFSA Space}\label{sec-algstruct}

 The material presented in this section  is reproduced from the first author's previous work~\cite{CWR10}, and is included here for the sake of completeness.

The formulation in Section~\ref{sec-pfsa} indicates that a symbolic dynamical process has a probabilistic finite state description if and only if the corresponding Nerode equivalence has a finite index.
%
\begin{defn}[Space of PFSA]\label{def:allmacandmeas}
The space of all PFSA over a given symbol alphabet is denoted by $\mathscr{A}$ and the space of all probability measures  $p$ inducing a finite-index probabilistic Nerode equivalence on the corresponding measure space $(\Sigma^{\omega},\mathfrak{B}_\Sigma,p)$
is denoted by $\Q$.
\end{defn}
As expected, there is a close relationship between $\mathscr{A}$ and $\Q$, which is made explicit in the sequel.
\begin{defn}[PFSA Map $\mathds{H}$] \label{def:H}
Let $p\in\mathscr{P}$ and $G=(Q,\Sigma,\delta,q_0,\widetilde{\Pi})\in\mathscr{A}$. The map $\mathds{H}:\mathscr{A}\rightarrow\mathscr{P}$ is defined as $\mathds{H}(G)=p$ such that the following condition is satisfied:
\cgather{ \label{equ:H}
\forall x = \sigma_1\cdots\sigma_r\in\Sigma^\star, \\ 
p(x)=\widetilde{\Pi}(q_0,\sigma_1) \prod_{k=1}^{r-1} \widetilde{\Pi}(\delta^\star(q_0,\sigma_1\cdots\sigma_k),\sigma_{k+1})
}
 where
$r\in\mathbb{N}$, the set of positive integers.
\end{defn}
\begin{defn}[Right Inverse $\mathds{H}_{-1}$]
The right inverse of the map $\mathds{H}$ is denoted by $\mathds{H}_{-1}:\mathscr{P}\rightarrow\mathscr{A}$ such that
\cgather{
 \forall p \in \Q, \ \mathds{H}(\mathds{H}_{-1} (p)) = p
}
\end{defn}
An explicit construction of  $\mathds{H}_{-1}$ is reported in~\cite{CR08};  we only require that such a map exists.
\begin{defn}[Perfect Encoding]\label{def:encodingM}
Given an alphabet $\Sigma$, a PFSA $G=(Q,\Sigma,\delta,q_0,\widetilde{\Pi})$ is said to be a perfect encoding of the measure space $(\Sigma^{\omega},\mathscr{B}_{\Sigma},p)$ if $p=\mathds{H}(G)$.
\end{defn}
There are possibly many PFSA realizations that encode the same probability measure on $\mathscr{B}_{\Sigma}$ due to existence of non-minimal realizations and state relabeling; neither of them affect the underlying encoded measure. From this perspective, a notion of PFSA equivalence is introduced as follows:
\begin{defn}[PFSA Equivalence]\label{def:equPFSA}
 Two PFSA $G_1$ and $G_2$ are defined to be equivalent if $\mathds{H}(G_1) = \mathds{H}(G_2)$. In this case, we say $G_1 = G_2$.
\end{defn}
%
In the sequel, a PFSA $G$ implies the equivalence class of $G$, $i.e.$,   $\{P\in\mathscr{A}:\mathds{H}(P) = \mathds{H}(G)\}$. 
%
\begin{defn}[Structural Equivalence] \label{def:sameStruct}
Two PFSA $G_i=(Q_i,\Sigma,\delta_i,q_0^i,\widetilde{\Pi}_i)\in\mathscr{A}$, $i=1,2$, are defined to have the equivalent (or identical) structure if $Q_1=Q_2,q_0^1 = q_0^2$ and $\delta_1(q,\sigma)=\delta_2(q,\sigma),\forall q\in Q_1 \ \forall\sigma\in\Sigma$.
\end{defn}
\begin{defn}[Synchronous Composition of PFSA] \label{def:product}
The binary operation of synchronous composition of two PFSA $G_i=(Q_i,\Sigma,\delta,q_0^{(i)},\widetilde{\Pi}_i)\in\mathscr{A}$ where $i=1,2$, denoted by $\sync: \mathscr{A}\times\mathscr{A}\rightarrow\mathscr{A}$ is defined as 
\cgather{
G_1 \sync G_2 = \left  (Q_1\times Q_2,\Sigma,\delta',(q_0^{(1)},q_0^{(2)}),\widetilde{\Pi}'\right ) 
\intertext{where $\delta'$ and $\widetilde{\Pi}'$ is computed as follows:}
 \forall q_i\in Q_1,q_j\in Q_2,\sigma\in\Sigma,  \left \{ \begin{array}{l} \delta' ( (q_i,q_j),\sigma) = \left( \delta_1(q_i,\sigma),\delta_2(q_j,\sigma) \right)\\
 \widetilde{\Pi}'((q_i,q_j),\sigma)=\widetilde{\Pi}_1 (q_i,\sigma)
\end{array}\right.
}
\end{defn}
In general,  $\sync$ $i.e.$  synchronous composition is non-commutative.
\begin{prop}[Synchronous Composition of PFSA] \label{pro:productSameMeasure}
Let $G_1, G_2\in\mathscr{A}$. Then,
 $\mathds{H}(G_1)=\mathds{H}(G_1 \sync G_2)$ and therefore $G_1=G_1 \sync G_2$ in the sense of Definition~\ref{def:equPFSA}.
\end{prop}
\begin{IEEEproof}
See Theorem 4.5 in~\cite{CR08}.
\end{IEEEproof}
\begin{figure*}[t]
\centering
 \input{Figures/figtwo_old.tex}
\captionN{Addition of arbitrary PFSAs with the same alphabet, using non-minimal realizations to equate structures (via synchronous composition) }\label{figS3}
\end{figure*}

Synchronous composition of PFSA allows transformation of PFSA with disparate structures to non-minimal descriptions that have the
same underlying graphs. This assertion is crucial for the development in the sequel, since any binary operation defined for two
PFSA with an identical structure can be extended to the general case on account of Definition~\ref{def:product} and Proposition~\ref{pro:productSameMeasure}.

Next we show that a restricted PFSA  subspace  can be assigned the algebraic structure of an Abelian group. We first construct the Abelian group on a subspace of probability measures, and then induce the group structure on this subspace of PFSA via the isomorphism between the two spaces.

\begin{defn}[Restricted PFSA Space]\label{defnA+}
Let $\mathscr{A}^+=\{ G=(Q,\Sigma,\delta,q_0,\widetilde{\Pi}): \widetilde{\Pi}(q,\sigma)>0 \ \forall q\in Q \ \forall \sigma\in\Sigma\}$ that is a proper subset of $\mathscr{A}$. It follows that the transition map of any PFSA in the subset $\mathscr{A}^+$ is a total function. We restrict the map $\mathds{H}:\mathscr{A}\rightarrow\mathscr{P}$ on a smaller domain $\mathscr{A}^+$, that is, $\mathds{H}^+:\mathscr{A}^+\rightarrow\mathscr{P}^+$, i.e., $\mathds{H}^+=\mathds{H}\arrowvert_{\mathscr{A}^+}$.
\end{defn}
\begin{defn}[Restricted Probability Measure]\label{defnP+}
Let $\mathscr{P}^+\triangleq\{p\in\mathscr{P}:p(x)\neq 0,\forall x\in\Sigma^\star\}$ that is a proper subset of $\mathscr{P}$. Each element of $\mathscr{P}^+$ is a probability measure that assigns a non-zero probability to each string on $\mathfrak{B}_\Sigma$.
Similar to Definition~\ref{defnA+}, we  restrict $\mathds{H}_{-1}$ on  $\mathscr{P}^+$, i.e., $\mathds{H}_{-1}^+=\mathds{H}_{-1}\arrowvert_{\mathscr{P}^+}$.
\end{defn}

Since we do not distinguish PFSA in the same equivalence class (See Definition~\ref{def:equPFSA}), we have the following result.
\begin{prop}[Isomorphism of $\mathds{H}^+$] \label{pro:isom}
The map $\mathds{H}^+$ is an isomorphism between the spaces $\mathscr{A}^+$ and $\Q^+$, and its inverse is $\mathds{H}^+_{-1}$.
\end{prop}
\begin{IEEEproof}
Immediate from preceding discussion.
\end{IEEEproof}

\begin{defn}[Abelian Operation on $\mathscr{P^+}$] \label{def:addition}
 The addition operation $\oplus: \mathscr{P^+} \times \mathscr{P^+} \rightarrow \mathscr{P^+}$ is defined by $p_3\triangleq p_1 \oplus p_2, \forall p_1,p_2 \in \mathscr{P^+}$ such that
\begin{enumerate}
  \item $p_3(\epsilon)=1$.
  \item $\forall x \in \Sigma^\star$ and $\tau\in\Sigma$, $\frac{p_3(x\tau)}{p_3(x)}=\frac{p_1(x\tau)p_2(x\tau)}{\sum_{\alpha\in\Sigma} p_1(x\alpha)p_2(x\alpha)}$
 \end{enumerate}
\end{defn}
$p_3$ is a well-defined probability measure on $\mathscr{P^+}$, since $ \forall x\in\Sigma^\star$:
\cgather{
\Sigma_{\tau\in\Sigma}p_3(x\tau)= \Sigma_{\tau\in\Sigma}\frac{p_1(x\tau)p_2(x\tau)}{\sum_{\alpha\in\Sigma} p_1(x\alpha)p_2(x\alpha)}p_3(x)= p_3(x)
}
\begin{prop}[abelian Group of PFSA]\label{prop:abelian}
 The algebra $(\mathscr{P^+},\oplus)$ forms an Abelian group.
\end{prop}

\begin{IEEEproof}
Closure property and commutativity of $(\mathscr{P^+},\oplus)$ are obvious. The associativity, existence of identity and existence of inverse element are established next.\\
\textit{(1) Associativity} $i.e.$ $(p_1\oplus p_2)\oplus p_3 = p_1\oplus (p_2\oplus p_3)$. Now, $\forall x\in\Sigma^\star,\tau\in\Sigma$, we have:
\mltlne{
    \frac{((p_1\oplus p_2)\oplus p_3)(x\tau)}{((p_1\oplus p_2)\oplus p_3)(x)}
     =
\frac{(p_1\oplus p_2)(x\tau)p_3(x\tau)}{\sum_{\beta\in\Sigma} (p_1\oplus p_2)(x\beta)p_3(x\beta)}
\\
 =
\frac{p_1(x\tau) (p_2\oplus p_3)(x\tau)}{\sum_{\beta\in\Sigma} p_1(x\beta)(p_2\oplus p_3)(x\beta)}     
= 
\frac{(p_1\oplus (p_2\oplus p_3))(x\tau)}{(p_1\oplus (p_2\oplus p_3))(x)}
}
\textit{ (2) Existence of identity}: Let us introduce a probability measure $\mathbf{i}_\circ$ of symbol
strings such that:
\cgather{\label{eqwn}
\forall x\in\Sigma^\star, \ \mathbf{i}_\circ(x)=\left(\frac{1}{|\Sigma|}\right)^{|x|}
}
where $|x|$ denotes the length of the string $x$. Then, $\forall\tau\in\Sigma$ that $\frac{\mathbf{i}_\circ(x\tau)}{\mathbf{i}_\circ(x)}=\frac{1}{|\Sigma|}$. For a measure $p\in\mathscr{P^+}$ and $\forall \tau\in \Sigma$,
\caligns{
\frac{(p\oplus \mathbf{i}_\circ)(x\tau)}{(p\oplus \mathbf{i}_\circ)(x)} & = \frac{p(x\tau)\mathbf{i}_\circ(x\tau)}{\sum_{\alpha\in\Sigma}p(x\alpha)\mathbf{i}_\circ(x\alpha)}
=\frac{p(x\tau)}{p(x)}
}
This implies that $p\oplus \mathbf{i}_\circ = \mathbf{i}_\circ\oplus p = p$ by
Definition \ref{def:addition} and by commutativity.  Therefore,
$\mathbf{i}_\circ$ is the identity of the monoid $(\mathscr{P}^+, \oplus)$. \\
\textit{(3) Existence of inverse}: $\forall p\in\mathscr{P^+}$, $\forall x\in\Sigma^\star$ and $\forall \tau\in\Sigma$, let $-p$ be defined by the following relations:
\cgather{
     (-p)(\epsilon)=1 \\ \frac{(-p)(x\tau)}{(-p)(x)}=\frac{p^{-1}(x\tau)}{\sum_{\alpha\in\Sigma}p^{-1}(x\alpha)}
    \intertext{Then, we have:}
\frac{(p\oplus (-p))(x\tau)}{(p\oplus (-p))(x)}
      =\frac{p(x\tau)(-p)(x\tau)}{\sum_{\alpha\in \Sigma}p(x\alpha)(-p)(x\alpha)}
    =\frac{1}{|\Sigma|}
}
This gives $p\oplus (-p)=\mathbf{i}_\circ$ which completes the proof.
\end{IEEEproof}
%
We denote  the zero-element $\mathbf{i}_\circ$ of the Abelian group $(\mathscr{P}^+, \oplus)$ as \emph{flat white noise (FWN)}. 
%
\vspace{0pt}
\begin{figure}[t]
\centering
 \definecolor{blcol}{RGB}{150,150,250}
\definecolor{murcol}{RGB}{255,150,150}
\definecolor{healcol}{RGB}{150,235,150}
\definecolor{colzeta}{RGB}{250,250,200}
\definecolor{leftpcol}{RGB}{200,255,200}
\definecolor{rightpcol}{RGB}{200,200,255}
\definecolor{boxcol}{RGB}{210,200,200}
\definecolor{linecol}{RGB}{200,180,180}
\definecolor{colplus}{RGB}{235,220,220}
\definecolor{colinv}{RGB}{200,220,255}
\definecolor{cof}{RGB}{219,144,71}
\definecolor{pur}{RGB}{200,200,200}
\definecolor{greeo}{RGB}{91,173,69}
\definecolor{greet}{RGB}{52,111,72}
 \definecolor{nodecol}{RGB}{180,180,220}
 \definecolor{nodeedge}{RGB}{240,248,255}
  \definecolor{nodecolb}{RGB}{220,180,180} 
  \definecolor{nodecolc}{RGB}{180,220,180}
  \definecolor{nodecolcD}{RGB}{100,160,100}
  \definecolor{nodecolW}{RGB}{190,190,190}
  \definecolor{edgecol}{RGB}{160,160,180}
 \definecolor{nodecolD}{RGB}{140,140,180}
  \definecolor{nodecolbD}{RGB}{180,140,140}
\tikzset{oplus/.style={path picture={%
      \draw[black]
       (path picture bounding box.south) -- (path picture bounding box.north) 
       (path picture bounding box.west) -- (path picture bounding box.east);
      }}} 
\tikzset{%
fshadow/.style={      preaction={
         fill=black,opacity=.1,
         path fading=circle with fuzzy edge 20 percent,
         transform canvas={xshift=1mm,yshift=-1mm}
       }}
}

\begin{tikzpicture}[->,>=stealth',shorten >=1pt,auto,node distance=1.2cm,
                    thick,scale=1.5,font=\bf\sffamily\fontsize{7}{7}\selectfont]
  \tikzstyle{every state}=[fill=nodecolW,draw=nodeedge,text=black,minimum size=5, text width=4, fshadow]

\node[state] (A)         at (0,0)           {$\mspace{-4mu}q_1$};
  \path  (A) edge [draw=edgecol,in=120,out=60,loop,above] node {$\sigma_0\vert 0.5$} (A)
        (A) edge [draw=edgecol,in=-60,out=-120,loop,below] node {$\sigma_1\vert 0.5$} (A);
\node [anchor=north] at ([yshift=-.3in]A.south) {\txt{(a) Zero PFSA \\for binary alphabet}};

   \node[state, anchor=west] (B)        at ([xshift=.65in]A.east)         {$\mspace{-4mu}q_1$};
  \path  (B) edge [draw=edgecol,in=10,out=70,loop,above] node {$\sigma_0\vert 1/3$} (B)
         (B) edge [draw=edgecol,in=110,out=170,loop,above] node {$\sigma_1\vert 1/3$} (B)
      (B) edge [draw=edgecol,in=-120,out=-60,loop,below] node {$\sigma_2\vert 1/3$} (B);
\node [anchor=north] at ([yshift=-.3in]B.south) {\txt{(b) Zero PFSA \\for trinary alphabet}};

\end{tikzpicture}
\captionN{Zero PFSAs for different alphabet sizes }\label{figS4}
\end{figure}

\subsection{Explicit Computation of the Abelian Operation $\oplus$} \label{sec:relationship}

The isomorphism between $\mathscr{P}^+$ and $\mathscr{A}^+$ (See Proposition~\ref{pro:isom})
induces the following Abelian operation on
$\mathscr{A}^+$.
\begin{defn}[Addition Operation on PFSA] \label{def:PFSAoperation}
Given any $G_1,G_2\in\mathscr{P}^+$, the addition operation $\pfsum :\mathscr{A}^+\times\mathscr{A}^+\rightarrow\mathscr{A}^+$ is defined as: $$G_1 \ \pfsum \ G_2=\mathds{H}_{-1}^+(\mathds{H}^+(G_1)\oplus\mathds{H}^+(G_2))$$
\end{defn}

If the summand PFSA have identical structure (i.e., their underlying graphs are identical), then the
explicit computation of this sum is stated as follows.
\begin{prop} [PFSA Addition]\label{pro:sameStructure}
 If two PFSA $G_1,G_2\in\mathscr{A}^+$ are of the same structure, i.e., $G_i=(Q,\Sigma,\delta,q_0,\widetilde{\Pi}_i),i=\{1,2\}$, then we have $G_1 \pfsum G_2 = (Q,\Sigma,\delta,q_0, \widetilde{\Pi})$ where
\cgather{ \widetilde{\Pi}(q,\sigma)=\frac{\widetilde{\Pi}_1(q,\sigma)\widetilde{\Pi}_2(q,\sigma)}{\sum_{\alpha\in\Sigma}\widetilde{\Pi}_1(q,\alpha)\widetilde{\Pi}_2(q,\alpha)}
}
\end{prop}
\begin{IEEEproof}
 Let $p_i=\mathds{H}^+(G_i)$, $i=\{1,2\}$ and since $G_1,G_2$ have the same structure, we have from  Eq. \eqref{equ:H}:
\cgather{
\forall \sigma \in \Sigma, \forall x \ \mathrm{s.t.} \  \delta^\star(q_0,x) = q \in Q, \notag \\ \frac{p_i(x\sigma)}{p_i(x)}=\widetilde{\Pi}_i(\delta^\star(q_0,x),\sigma)=\widetilde{\Pi}_i(q,\sigma)
}
Now, by Definition \ref{def:addition} and Definition \ref{def:H},
\caligns{
& \quad \widetilde{\Pi}(q,\sigma)
=\frac{(p_1\oplus p_2)(x\sigma)}{(p_1\oplus p_2)(x)}
=\frac{p_1(x\sigma)p_2(x\sigma)}{\sum_{\alpha\in\Sigma}p_1(x\alpha)p_2(x\alpha)}\\ =&\frac{\frac{p_1(x\sigma)p_2(x\sigma)}{p_1(x)p_2(x)}}{\sum_{\alpha\in\Sigma}\frac{p_1(x\alpha)p_2(x\alpha)}{p_1(x)p_2(x)}}
=\frac{\widetilde{\Pi}_1(q,\sigma)\widetilde{\Pi}_2(q,\sigma)}{\sum_{\alpha\in\Sigma}\widetilde{\Pi}_1(q,\alpha)\widetilde{\Pi}_2(q,\alpha)}
}
\end{IEEEproof}
%
The extension to the general case is achieved by using synchronous composition of probabilistic machines.
\begin{prop}[PFSA Addition (General case)]\label{propPsum}
Given two PFSA $G_1, G_2\in\mathscr{A}^+$, the sum
$G_1 \pfsum  G_2$ is computed via Proposition  \ref{pro:sameStructure} and Definition  \ref{def:product} as follows:
\cgather{
        G_1  \pfsum  G_2 = (G_1 \sync G_2) \pfsum ( G_2 \sync G_1)
}
\end{prop}
\begin{IEEEproof}
Noting that $G_1\sync G_2$ and $G_2\sync G_1$ have the same structure up to state relabeling, it follows from Proposition~\ref{pro:productSameMeasure}:
\caligns{
 \mathds{H}^+(G_1  \pfsum  G_2)
 &=\mathds{H}^+(G_1)\oplus\mathds{H}^+(G_2) \mspace{30mu} \mathrm{(See\ Definition \ \ref{def:PFSAoperation})}\\
& =\mathds{H}^+(G_1 \sync G_2)\oplus\mathds{H}^+(G_2 \sync G_1)\\ &=\mathds{H}^+\bigg ((G_1 \sync G_2)  \pfsum  (G_2 \sync G_1)\bigg )
}
which completes the proof.
\end{IEEEproof}
\begin{example}
Let $G_1$ and $G_2$ be two PFSA with identical structures, such that the probability morph matrices are:
\cgather{
 \widetilde{\Pi}_1 =
\begin{pmatrix}
 0.2 & 0.8 \\
 0.4 & 0.6
\end{pmatrix}
\textrm{  and }
\widetilde{\Pi}_2 =
\begin{pmatrix}
 0.1 & 0.9 \\
 0.6 & 0.4
\end{pmatrix}
}
Then the $\widetilde{\Pi}$-matrix for the sum $G_1 \pfsum G_2$, denoted by $\widetilde{\Pi}_{12}$, is
\cgathers{
\widetilde{\Pi}_{12} =
\begin{pmatrix}
 0.1\times0.2 & 0.9\times 0.8 \\
 0.6 \times 0.4 & 0.4 \times 0.6
\end{pmatrix} \xrightarrow[{ rows}]{{ Normalize}}
\begin{pmatrix}
 0.027 & 0.973 \\
 0.5 & 0.5
\end{pmatrix}
}
\end{example}
\section{Correctness Of Stream Operations}\label{sec-stropr}

In this section (Section~\ref{sec-stropr}), we prove that the stream operations described in Table~\ref{tab1} of main text are indeed correct.
\subsection{Independent Stream Copy}

We show that the ``Independent Stream Copy'' operation produces an independent realization from a pseudo-copy of the PFSA model generating the input stream.
First, we formalize the notion of pseudo-copies.

\begin{defn}\label{defpseudocopy}
Given a PFSA $G=(Q,\Sigma,\delta,\pitilde)$ in the canonical representation, a pseudo-copy is a canonical PFSA $\mathds{P}_\gamma(G)=(Q,\Sigma,\delta,\mathds{P}(\pitilde))$, where we have:
\cgather[3pt]{\label{Eqpseudo1}
\mathds{P}_\gamma(\Pi) = \gamma [ \mathbb{I} - (1-\gamma)\Pi]^{-1} \Pi
}
for some scalar $\gamma \in (0,1)$.
\end{defn}

We note that that while the row-stochastic matrix $\Pi$ may not be invertible, and $ [ \mathbb{I} - \Pi]$ is definitely singular (since $\Pi$ has a eigenvalue at 1); the matrix $\gamma [ \mathbb{I} - (1-\gamma)\Pi]^{-1} $ is always well-defined  for $\gamma \in (0,1)$, and additionally is a non-negative row-stochastic matrix~\cite{CR07}.

We use the following notation:
\begin{notn}
For a given string $s$, the underlying PFSA generator is denoted as $G\leftarrow s$, and for a given PFSA $G$,  $G\rightarrow s$ is a realization generated by $G$.
Note that $G \leftarrow s$ automatically implies that we are referring to the PFSA generator as $\vert s \vert \rightarrow \infty$, since one cannot have a unique generator for bounded strings.
\end{notn}

\begin{prop}[Independent Stream Copy]\label{propstreamcopy} Given a symbol stream $s$ with a hidden PFSA generator $G$, let stream $s'$ be generated via:
\begin{list}{\labelitemi}{\leftmargin=0.75em}
\item[1] Generate stream $\omega_0$ from FWN
 \item[2] Read current symbol $\sigma_1$ from $s$, and $\sigma'$ from
$\omega_0$
\item[3] If $\sigma= \sigma'$, then write to output $s'$
\item[4] Move read positions one step to right, and go to step 1
\end{list}
Then, we have: 
\begin{enumerate}
 \item Well-defined convergence of underlying models:
\cgather{
\lim_{\vert s'\vert\rightarrow \infty} (G' \leftarrow s') = \lim_{\vert s\vert\rightarrow \infty}\mathds{P}_\frac{1}{2}(G \leftarrow s)
}
\item If $s',s''$ are generated from the above algorithm from the same input stream $s$, then, in the limit of infinite length, $s',s''$ are independent realizations of $\mathds{P}_\frac{1}{2}(G\leftarrow s)$.
\item If $s_1',s_2'$ are generated by the algorithm  for input streams $s_1,s_2$ respectively, then we have:
\mltlne{\forall \epsilon > 0,
 \Theta(G_1\leftarrow s_1, G_2 \leftarrow s_2  ) \leqq \epsilon \\ \Rightarrow \Theta(G_1'\leftarrow s_1', G_2' \leftarrow s_2'  ) 
\leqq \vert \Sigma \vert\epsilon
}
\mltlne{\forall \epsilon > 0,
 \Theta(G_1\leftarrow s_1, G_2 \leftarrow s_2  ) \geqq \epsilon \\ \Rightarrow\Theta(G_1'\leftarrow s_1', G_2' \leftarrow s_2'  ) \geqq \frac{\vert \Sigma \vert }{(2\vert \Sigma\vert -1 )^2}\epsilon
}
\end{enumerate}

\end{prop}
\begin{IEEEproof}
(1) Let $\gamma$ be the probability that the first symbol in the input stream $s$ is recorded in the output. Since, the stream $\omega_0$ is FWN, we conclude  that $\gamma= \frac{1}{\Sigma}$, and that $\gamma$ is also the constant probability that any symbol in $s$ is recorded. Thus, assuming that the symbolic derivatives computed are exact ($i.e.$ the input stream is infinite), the transition matrix $M$ of a realization for the PFSA $G'\leftarrow s'$ can be expressed as a function of the transition matrix $\Pi$ for $G\leftarrow s$ as:
\cgather[4pt]{
M = \gamma \Pi +(1-\gamma) \gamma\Pi^2 + (1-\gamma)^2\gamma \Pi^3 + \cdots\\
\Rightarrow M = \gamma [ \mathbb{I} - (1-\gamma)\Pi]^{-1} \Pi = \mathds{P}_\gamma(\Pi)
}
Since, the transformation $[ \mathbb{I} - (1-\gamma)\Pi]^{-1}$ is invertible, the rank of 
$M$ is the same as $\Pi$, which implies that states in $G\leftarrow s$ cannot collapse when we pass to $\mathds{P}_\gamma(G)$, implying in turn that $G'\leftarrow s'$ is has indeed the same minimal structure as $G \leftarrow s$, which establishes claim (1).

(2) Claim (1) implies:
\cgather{
\lim_{\vert s' \vert \rightarrow \infty}G'\leftarrow s' =\lim_{\vert s'' \vert \rightarrow \infty} G''\leftarrow s''
}
The independence claim then follows immediately from noting that random erasure, as executed by the stated algorithm, eliminates any possibility of  synchronization between the states of the same underlying  model $G'\leftarrow s'$ in the limit of infinite  string lengths .


(3) Consider the PFSAs  $G_1 \leftarrow s_1, G_2 \leftarrow s_2$ as $\vert s_1\vert, \vert s_2 \vert \rightarrow \infty$. Let us bring them to the same structure, via the transformations $(G_1 \leftarrow s_1 )\otimes (G_2 \leftarrow s_2)$ and $(G_2 \leftarrow s_2 )\otimes (G_1 \leftarrow s_1)$ respectively (See [REF]). Let us denote the transition matrices of the PFSAs in their transformed representations as $\Pi_1,\Pi_2$ respectively. Then denoting $\Delta = \Pi_1-\Pi_2$, and $\Delta' = \mathds{P}_\gamma(\Pi_1)-\mathds{P}_\gamma(\Pi_2)$, we claim:
\cgather{
\frac{\gamma}{(2-\gamma)^2} \vert\vert \Delta \vert\vert_\infty\leqq  \vert\vert\Delta'\vert\vert_\infty  \leqq  \frac{1}{\gamma}\vert\vert \Delta\vert\vert_\infty \label{eq32}
}
To establish this claim, we first note that for any stochastic matrix $A$, we have:
\cgather[5pt]{
\gamma [ \mathbb{I} - (1-\gamma)A]^{-1} A =  \frac{1}{1-\gamma} \gamma [ \mathbb{I} - (1-\gamma)A]^{-1} - \frac{\gamma}{1-\gamma}\mathbb{I}
}
which implies the upper bound in Eq.~\eqref{eq32} as  follows:
\mltlne{
(1-\gamma)\left (\mathds{P}_\gamma(\Pi_1)-\mathds{P}_\gamma(\Pi_2)\right ) \notag  \\
=  \gamma [ \mathbb{I} - (1-\gamma)\Pi_1]^{-1} -  \gamma [ \mathbb{I} - (1-\gamma)\Pi_2]^{-1} \notag  \\ 
= \gamma(1-\gamma) [ \mathbb{I} - (1-\gamma)\Pi_2]^{-1}(\Pi_2 -  \Pi_1) [ \mathbb{I} - (1-\gamma)\Pi_1]^{-1}
}
This implies:
\mltlne{
\Rightarrow \Delta'  = -\gamma [ \mathbb{I} - (1-\gamma)\Pi_2]^{-1} \Delta  [ \mathbb{I} - (1-\gamma)\Pi_1]^{-1}\\
\shoveleft{\Rightarrow \vert\vert\Delta'\vert\vert_\infty  \leqq \gamma \vert\vert  [ \mathbb{I} - (1-\gamma)\Pi_2]^{-1} \vert \vert_\infty} \notag \\
\times  \vert\vert  [ \mathbb{I} - (1-\gamma)\Pi_1]^{-1} \vert \vert_\infty  \vert\vert \Delta \vert\vert_\infty\\
\shoveleft{\Rightarrow \vert\vert\Delta'\vert\vert_\infty  \leqq \gamma  \times \frac{1}{\gamma} \times \frac{1}{\gamma} \times  \vert\vert \Delta \vert\vert_\infty= \frac{1}{\gamma}  \vert\vert \Delta \vert\vert_\infty}
}
And the lower bound follows from noting:
\mltlne{
\Rightarrow \Delta'  = -\gamma [ \mathbb{I} - (1-\gamma)\Pi_2]^{-1} \Delta  [ \mathbb{I} - (1-\gamma)\Pi_1]^{-1}\\
\shoveleft{\Rightarrow -\gamma  \Delta  = [ \mathbb{I} - (1-\gamma)\Pi_2] \Delta' [ \mathbb{I} - (1-\gamma)\Pi_1]}\\
\shoveleft{\Rightarrow \gamma  \vert \vert \Delta \vert\vert_\infty \leqq \vert \vert[ \mathbb{I} - (1-\gamma)\Pi_2] \vert\vert_\infty \vert \vert[ \mathbb{I} - (1-\gamma)\Pi_1] \vert\vert_\infty \vert \vert \Delta'\vert\vert_\infty}\\
\shoveleft{\Rightarrow \vert\vert\Delta'\vert\vert_\infty \geqq \frac{\gamma}{\vert \vert[ \mathbb{I} - (1-\gamma)\Pi_2] \vert\vert_\infty \vert \vert[ \mathbb{I} - (1-\gamma)\Pi_1] \vert\vert_\infty    }  \vert \vert \Delta \vert \vert_\infty} \\
\geqq \frac{\gamma}{(2-\gamma)^2} \vert\vert \Delta \vert\vert_\infty
}
Next, we compute bounds on the probability morph matrices.  

We denote the probability morph matrices for the relevant  PFSAs as follows
(PFSAs on left, morph matrices on right):
\cgather{
\begin{array}{r|l}
(G_1 \leftarrow s_1 )\otimes (G_2 \leftarrow s_2) & \Pitilde_1\\
(G_2 \leftarrow s_2 )\otimes (G_1 \leftarrow s_1) & \Pitilde_2\\
\mathds{P}_\gamma((G_1 \leftarrow s_1 )\otimes (G_2 \leftarrow s_2)) & \mathds{P}(\Pitilde_1) \\
 \mathds{P}_\gamma( (G_2 \leftarrow s_2 )\otimes (G_1 \leftarrow s_1) ) & \mathds{P}(\Pitilde_2)
\end{array} \notag
\intertext{And, additionally, we use the notation:}
\widetilde{\Delta} = \Pitilde_1 -\Pitilde_2 \\
\widetilde{\Delta}' = \mathds{P}(\Pitilde_1) -\mathds{P}(\Pitilde_2)}
Without loss of generality, we assume that for each PFSA, given a state,  each symbol leads to a distinct state. This can be arranged via state splitting if necessary, and implies:
\cgather{
\vert\vert\Delta\vert \vert_\infty = \vert \vert \Pitilde_1 - \Pitilde_2 \vert\vert_\infty\\ \vert\vert\Delta'\vert \vert_\infty = \vert \vert \mathds{P}(\Pitilde_1) - \mathds{P}(\Pitilde_2) \vert\vert_\infty
} 
which therefore leads to the bounds:
\mltlne{
\frac{\gamma}{(2-\gamma)^2} \vert\vert \Pitilde_1 - \Pitilde_2 \vert\vert_\infty\leqq  \vert\vert\mathds{P}(\Pitilde_1) - \mathds{P}(\Pitilde_2)\vert\vert_\infty  \\ \leqq  \frac{1}{\gamma}\vert\vert \Pitilde_1 - \Pitilde_2\vert\vert_\infty \label{eqpitB}
}

Recall the definition of the PFSA metric  (See Proposition~\ref{propmetric}):
\mltlne{
\Theta(G_1\leftarrow s_1,G_2\leftarrow s_2) \\ =\frac{\vert \Sigma \vert -1}{\vert \Sigma\vert} \lim_{\vert s_1\vert \rightarrow \infty, \vert s_2\vert \rightarrow \infty}
\sum_{x \in \Sigma^\star}\left \{  \frac{\vert \vert \phi^{s_1}_{G_1}(x) - \phi^{s_2}_{G_2}(x)  \vert \vert_\infty}{\vert \Sigma \vert^{2\vert x \vert}}\right \} 
}
and,  note that:
\mltlne{\label{eq46}
\Theta(G_1\leftarrow s_1,G_2\leftarrow s_2) \leqq \epsilon \\ \Rightarrow \forall x \in \Sigma^\star, \vert \vert \phi^{s_1}_{G_1}(x) - \phi^{s_2}_{G_2}(x)  \vert \vert_\infty  \leqq \epsilon
}
\mltlne{\label{eq46}
\Theta(G_1\leftarrow s_1,G_2\leftarrow s_2) \geqq \epsilon \\ \Rightarrow \forall x \in \Sigma^\star, \vert \vert \phi^{s_1}_{G_1}(x) - \phi^{s_2}_{G_2}(x)  \vert \vert_\infty  \geqq \epsilon
}

Since the bounds established in Eq.~\eqref{eqpitB} is applicable to any non-minimal realization of the PFSAs  $(G_1 \leftarrow s_1 )\otimes (G_2 \leftarrow s_2)$ and $(G_2 \leftarrow s_2 )\otimes (G_1 \leftarrow s_1)$, considering the full $\Sigma$-ary tree as the limiting ``unfolded'' realization, we conclude from Eq.~\eqref{eqpitB} that:
\mltlne{
\Theta(G_1\leftarrow s_1,G_2\leftarrow s_2) \leqq \epsilon \\  \Rightarrow \Theta(G_1'\leftarrow s_1',G_2'\leftarrow s_2') \leqq \frac{1}{\gamma}\epsilon
}
and also:
\mltlne{
\Theta(G_1\leftarrow s_1,G_2\leftarrow s_2) \geqq \epsilon \\ \Rightarrow \Theta(G_1'\leftarrow s_1',G_2'\leftarrow s_2') \geqq \frac{\gamma}{(2-\gamma)^2}\epsilon
}
The desired bounds then follow from noting that in the stated algorithm, we have $\gamma =\frac{1}{\vert \Sigma\vert}$. This completes the proof.
 \end{IEEEproof}
%
\begin{rem}
Proposition~\ref{propstreamcopy} establishes that the stream $s'$ 
obtained from an input stream $s$ may not be a realization from the  hidden generator for the latter ($i.e.$ stream $s$), but is a realization from a PFSA which is a pseudo-copy of the generator for $s$. 

Also, note that it is not true in general that a pseudo-copy is close to the original PFSA, in the sense of our metric.

Nevertheless, Proposition~\ref{propstreamcopy} shows that if the distance between two PFSAs is small, then so is the distance between their pseudo-copies; and if the distance between two PFSAs is large, then so is the distance between the pseudo-copies.

Thus, if we determine the distance between pseudo-copies, we have a good estimate of the distance between the original machines.
\end{rem}
\subsection{Stream Summation}
\begin{prop}[Stream Summation]\label{propstreamsum} Given a symbol streams $s_1,s_2$ with  hidden PFSA generators $G_1,G_2$, let stream $s'$ be generated via:\begin{list}{\labelitemi}{\leftmargin=0.75em}
 \item[1] Read current symbols $\sigma_i$ from $s_i$ ($i=1,2$)
\item[2] If $\sigma_1 =  \sigma_2$, then write 
to output $s'$
\item[3] Move read positions one step to right, and go to step 1
\end{list}
Then, denoting the FWN generator as $W$, we have:
\begin{enumerate}
\item If $(G_1\leftarrow s_1) + (G_2 \leftarrow s_2)$ has a single state in its minimal realization, then we have:
\cgather{
(G'\leftarrow s') = (G_1\leftarrow s_1) + (G_2 \leftarrow s_2)
}
$i.e.$, then $s'$ is an exact realization of $(G_1\leftarrow s_1) + (G_2 \leftarrow s_2)$ in the limit $\vert s_1 \vert, \vert s_2 \vert  \rightarrow \infty$.
\item If $(G_1\leftarrow s_1) + (G_2 \leftarrow s_2)$ has $N > 1$ states in its minimal realization, we have the lower bound:
\mltlne{
\Theta((G_1\leftarrow s_1) + (G_2 \leftarrow s_2), W) \geqq \epsilon \\ \Rightarrow \Theta((G' \leftarrow s'), W) \geqq\frac{\epsilon}{\vert \Sigma \vert^N + \vert \Sigma \vert^{N-1}  } 
}
\item We have the upper bound:
\mltlne{
\Theta\left ((G_1\leftarrow s_1) + (G_2 \leftarrow s_2), W \right ) \leqq \epsilon\\ \Rightarrow \Theta\left ((G' \leftarrow s'), W\right ) \leqq \epsilon
}
\end{enumerate}
\end{prop}
\begin{IEEEproof}
(1) If  $(G_1\leftarrow s_1) + (G_2 \leftarrow s_2)$ has a single causal state:
\cgather{\forall x \in \Sigma^\star, 
 \lim_{\vert s' \vert \rightarrow \infty} \phi^{s'}(x) = v,  
\textrm{ where } \sum_i v_i = 1, v_i \geqq 0
}
 We assume without loss of generality that $G_1\leftarrow s_1,G_2\leftarrow s_2$ in their canonical representations (See Definition~\ref{defcanon}). Thus, the streams $s_1,s_2$ can be assumed to start at states of the canonical representation which maps to the corresponding stationary distribution over the states of the corresponding initial-marked PFSAs (See Definition~\ref{defcanon}, and the discussion immediately after). Also, note that we can delete arbitrary prefixes from $s_1$ and $s_2$, and still assume that they start at these states. Thus, we delete prefixes of $s_1,s_2$ up to the point where the next symbols match, and we see our first output symbol. 

 Since we see a symbol in the output if there is a match in both $s_1$ and $s_2$, it follows that the probability of seeing the first symbol in $s'$ as $\sigma_i$ is given by:
\cgather{
\lim_{\vert s_1 \vert, \vert s_2 \vert \rightarrow \infty}  \frac{1}{ \sum_i\phi^{s_1}(\lambda)\vert_i \phi^{s_2}(\lambda)\vert_i   }\phi^{s_1}(\lambda)\vert_i \phi^{s_2}(\lambda)\vert_i = v_i
}

Also, since $G_1\leftarrow s_1$ and $G_2\leftarrow s_2$ can be assumed to have  the same graph without loss of generality (via considering non-minimal realizations if necessary), it follows that the next hidden states $q,q'$ in $s_1,s_2$ after seeing the first output symbol, are still synchronized. Thus, we conclude, that if the first observed symbol in $s'$ is $\sigma$, then the probability that the next symbol is $\sigma_j$, is also given by:
\cgather{
\lim_{\vert s_1  \vert, \vert s_2 \vert \rightarrow \infty}   \frac{1}{ \sum_i\phi^{s_1}(\lambda)\vert_j \phi^{s_2}(\lambda)\vert_j   }\phi^{s_1}(\lambda)\vert_i \phi^{s_2}(\lambda)\vert_i = v_j
}

It follows from straightforward induction, that at any point in $s'$, the distribution of the next symbol is given by $v$, $i.e$, in the limit $s_1 \rightarrow \infty, s_2 \rightarrow \infty$, $s'$ is an exact realization from the sum  $(G_1\leftarrow s_1) + (G_2 \leftarrow s_2)$ (if the latter has a single causal state).

\begin{figure*}[t]
\centering
\vspace{-50pt}
      \tikzset{%
fshadow/.style={      preaction={
         fill=black,opacity=.2,
         path fading=circle with fuzzy edge 20 percent,
         transform canvas={xshift=1mm,yshift=-1mm}
       }}
}
\def\SCALE{.7}
  \begin{tikzpicture}[->,>=latex',shorten >=1pt,auto,node distance=1.2cm  ,scale=.9]
\node []  (N0)  {
  \begin{tikzpicture}[->,>=latex',shorten >=1pt,auto,node distance=1.2cm  ,,scale=\SCALE]
 \definecolor{nodecol}{RGB}{180,180,220}
 \definecolor{nodeedge}{RGB}{240,248,255}
  \definecolor{edgecol}{RGB}{150,150,250}
  \tikzstyle{every state}=[fill=Red1!20,draw=gray,text=black,minimum size=3, text width=10,scale=1,fshadow]
   \tikzstyle{every edge}=[ thick]
  \node[state] (A)         []           {$\mspace{-0mu}\boldsymbol{\wp_{x}}$};
  \node[state]         (B) at ([xshift=1in]A.east) {$\mspace{-4mu}\boldsymbol{\wp_{x'}}$};
  \node[state]         (C) at ([yshift=.5in]$(A)!.5!(B)$) {$\mspace{-4mu}\boldsymbol{\wp_{x''}}$};
  \node[state]         (D) at ([yshift=-.55in]$(A)!.5!(B)$) {$\mspace{-8mu}\boldsymbol{\wp_{x'''}}$};
\coordinate (T1) at ($(B)!.5!(C)$);
  \coordinate (T2) at ($(A)!.5!(D)$);
 \coordinate (T3) at ($(T1)!2.5!(T2)$);
 \coordinate (T4) at ($(T2)!2.5!(T1)$);
 \coordinate (T5) at ($(T3)!.5!(T4)$);
 \coordinate (T6) at ($(A)!.5!(C)$);
 \coordinate (T7) at ($(T5)!3.5!(T6)$);
 \coordinate (T8) at ($(D)!2.5!(C)$);
 \coordinate (T8a) at ($(D)!3.2!(C)$);
 \coordinate (T9) at ($(D)!3.5!(A)$);
 \coordinate (T31) at ($(T1)!3.5!(T2)$);
 \coordinate (T88) at ($(B)!1.5!(A)$);
 \coordinate (T88a) at ($(A)!.5!(T7)$);
 \coordinate (T89) at ($(B)!1.5!(T88a)$);
 \coordinate (T4a) at ($(T2)!3!(T1)$);
  \node[state]         (G) at (T7) {$\mspace{0mu}\boldsymbol{\wp_{\lambda}}$};
  \path (A) edge   [draw=edgecol,bend left]           node [below] {$\sigma_1\vert p_3$} (B)
   (A) edge [draw=edgecol,bend left, above] node [sloped,] {$\sigma_0\vert p_1$} (C)
        (B)  edge   [draw=edgecol,bend left]           node [above] {$\sigma_0\vert p_2$} (A);
             \path   (C) edge [draw=edgecol,bend left] node [sloped,above] {$\sigma_1\vert p_1$} (B);
          \path   (B) edge [draw=edgecol,bend left] node [sloped,below] {$\sigma_1\vert p_1$} (D);
          \path   (D) edge [draw=edgecol,bend left] node [sloped,below] {$\sigma_0\vert p_1$} (A)
 (C) edge [draw=edgecol,in=120,out=60,loop,above] node {$\sigma_0\vert p_1$} (C)
 (G) edge [draw=edgecol,in=120,out=90,loop,above] node {$\sigma_0\vert p_1$} (G)
 (D) edge [draw=edgecol,in=-120,out=-60,loop,below] node {$\sigma_1\vert p_1$} (D);
           \path   (G) edge [draw=edgecol,bend left] node [above] {$\sigma_1\vert p_1$} (C);
           \path   (A) edge [draw=edgecol,] node [below, sloped] {$\sigma_1\vert p_1$} (G);
\end{tikzpicture}
};
\node [anchor=west]  (N1)  at ([xshift=.25in]N0.east) {
  \begin{tikzpicture}[->,>=latex',shorten >=1pt,auto,node distance=1.2cm  ,,scale=\SCALE]
 \definecolor{nodecol}{RGB}{180,180,220}
 \definecolor{nodeedge}{RGB}{240,248,255}
  \definecolor{edgecol}{RGB}{150,150,250}
  \tikzstyle{every state}=[fill=Blue1!10,draw=gray,text=black,minimum size=3, text width=10,scale=1,fshadow]
   \tikzstyle{every edge}=[ thick]
  \node[state] (A)         []           {$\mspace{-0mu}\boldsymbol{\wp_{x}}$};
  \node[state]         (B) at ([xshift=1in]A.east) {$\mspace{-4mu}\boldsymbol{\wp_{x'}}$};
  \node[state]         (C) at ([yshift=.5in]$(A)!.5!(B)$) {$\mspace{-4mu}\boldsymbol{\wp_{x''}}$};
  \node[state]         (D) at ([yshift=-.55in]$(A)!.5!(B)$) {$\mspace{-8mu}\boldsymbol{\wp_{x'''}}$};
\coordinate (T1) at ($(B)!.5!(C)$);
  \coordinate (T2) at ($(A)!.5!(D)$);
 \coordinate (T3) at ($(T1)!2.5!(T2)$);
 \coordinate (T4) at ($(T2)!2.5!(T1)$);
 \coordinate (T5) at ($(T3)!.5!(T4)$);
 \coordinate (T6) at ($(A)!.5!(C)$);
 \coordinate (T7) at ($(T5)!3.5!(T6)$);
 \coordinate (T8) at ($(D)!2.5!(C)$);
 \coordinate (T8a) at ($(D)!3.2!(C)$);
 \coordinate (T9) at ($(D)!3.5!(A)$);
 \coordinate (T31) at ($(T1)!3.5!(T2)$);
 \coordinate (T88) at ($(B)!1.5!(A)$);
 \coordinate (T88a) at ($(A)!.5!(T7)$);
 \coordinate (T89) at ($(B)!1.5!(T88a)$);
 \coordinate (T4a) at ($(T2)!3!(T1)$);
  \node[state]         (G) at (T7) {$\mspace{0mu}\boldsymbol{\wp_{\lambda}}$};
  \path (A) edge   [draw=edgecol,bend left]           node [below] {$\sigma_1\vert q_3$} (B)
   (A) edge [draw=edgecol,bend left, above] node [sloped,] {$\sigma_0\vert q_1$} (C)
        (B)  edge   [draw=edgecol,bend left]           node [above] {$\sigma_0\vert q_2$} (A);
             \path   (C) edge [draw=edgecol,bend left] node [sloped,above] {$\sigma_1\vert q_1$} (B);
          \path   (B) edge [draw=edgecol,bend left] node [sloped,below] {$\sigma_1\vert q_1$} (D);
          \path   (D) edge [draw=edgecol,bend left] node [sloped,below] {$\sigma_0\vert q_1$} (A)
 (C) edge [draw=edgecol,in=120,out=60,loop,above] node {$\sigma_0\vert q_1$} (C)
  (G) edge [draw=edgecol,in=120,out=90,loop,above] node {$\sigma_0\vert q_1$} (G)
(D) edge [draw=edgecol,in=-120,out=-60,loop,below] node {$\sigma_1\vert q_1$} (D);
           \path   (G) edge [draw=edgecol,bend left] node [above] {$\sigma_1\vert q_1$} (C);
           \path   (A) edge [draw=edgecol,] node [below, sloped] {$\sigma_1\vert q_1$} (G);
\end{tikzpicture}
};
\node [anchor=west] (N2) at ([yshift=.35in, xshift=-.5in]N1.east)  {
  \begin{tikzpicture}[->,>=latex',shorten >=1pt,auto,node distance=1.2cm  ,,scale=\SCALE]
 \definecolor{nodecol}{RGB}{180,180,220}
 \definecolor{nodeedge}{RGB}{240,248,255}
  \definecolor{edgecol}{RGB}{150,150,250}
  \tikzstyle{every state}=[fill=Green1!30,draw=Green3,text=black,minimum size=3, text width=10,scale=1,fshadow]
   \tikzstyle{every edge}=[ thick]
  \node[state] (A)         []           {$\mspace{-0mu}\boldsymbol{\wp_{x}}$};
  \node[state]         (B) at ([xshift=1in]A.east) {$\mspace{-4mu}\boldsymbol{\wp_{x'}}$};
  \node[state]         (C) at ([yshift=.5in]$(A)!.5!(B)$) {$\mspace{-4mu}\boldsymbol{\wp_{x''}}$};
  \node[state]         (D) at ([yshift=-.55in]$(A)!.5!(B)$) {$\mspace{-8mu}\boldsymbol{\wp_{x'''}}$};
\coordinate (T1) at ($(B)!.5!(C)$);
  \coordinate (T2) at ($(A)!.5!(D)$);
 \coordinate (T3) at ($(T1)!2.5!(T2)$);
 \coordinate (T4) at ($(T2)!2.5!(T1)$);
 \coordinate (T5) at ($(T3)!.5!(T4)$);
 \coordinate (T6) at ($(A)!.5!(C)$);
 \coordinate (T7) at ($(T5)!3.5!(T6)$);
 \coordinate (T8) at ($(D)!2.5!(C)$);
 \coordinate (T8a) at ($(D)!3.2!(C)$);
 \coordinate (T9) at ($(D)!3.5!(A)$);
 \coordinate (T31) at ($(T1)!3.5!(T2)$);
 \coordinate (T88) at ($(B)!1.5!(A)$);
 \coordinate (T88a) at ($(A)!.5!(T7)$);
 \coordinate (T89) at ($(B)!1.5!(T88a)$);
 \coordinate (T4a) at ($(T2)!3!(T1)$);
  \node[state]         (G) at (T7) {$\mspace{0mu}\boldsymbol{\wp_{\lambda}}$}; 
  \path (A) edge   [draw=edgecol,bend left]           node [below] {$\sigma_1\vert p_3 q_3$} (B)
   (A) edge [draw=edgecol,bend left, above] node [sloped,] {$\sigma_0\vert p_1q_1$} (C)
        (B)  edge   [draw=edgecol,bend left]           node [above] {$\sigma_0\vert p_2 q_2$} (A);
             \path   (C) edge [draw=edgecol,bend left] node [sloped,above] {$\sigma_1\vert p_1q_1$} (B);
          \path   (B) edge [draw=edgecol,bend left] node [sloped,below] {$\sigma_1\vert p_1q_1$} (D);
          \path   (D) edge [draw=edgecol,bend left] node [sloped,below] {$\sigma_0\vert p_1q_1$} (A)
 (C) edge [draw=edgecol,in=120,out=60,loop,above] node {$\sigma_0\vert q_1$} (C)
  (G) edge [draw=edgecol,in=120,out=90,loop,above] node {$\sigma_0\vert p_1q_1$} (G)
(D) edge [draw=edgecol,in=-120,out=-60,loop,below] node {$\sigma_1\vert p_1q_1$} (D);
           \path   (G) edge [draw=edgecol,bend left] node [above] {$\sigma_1\vert p_1q_1$} (C);
           \path   (A) edge [draw=edgecol,] node [below, sloped] {$\sigma_1\vert p_1q_1$} (G);

  \draw  [-latex,edgecol, ultra thick, dashed, gray, opacity=.5,text opacity=1,text=black] (A) .. controls (T88) and  (T89) .. node [midway, left] {$m_1$} (G);
             \draw  [-latex,edgecol, ultra thick, dashed, gray, opacity=.5,text opacity=1,text=black] (B) .. controls (T4a) and  (T8a) .. node [midway, right] {$m_2$} (G);         
            \draw  [-latex,edgecol, ultra thick, dashed, gray, opacity=.5,text opacity=1,text=black] (D) .. controls (T31) and (T9)  .. node [midway, right] {$m_3$} (G); 
                       \draw  [-latex,edgecol, ultra thick, opacity=.5,text opacity=1,dashed, gray, text=black] (C) .. controls (T4) and (T8) .. node [midway, right] {$m_4$} (G);
\end{tikzpicture}
};
\node [anchor=north] (L0) at (N0.south) {(a) $G_1 \leftarrow s_1$};
\node [anchor=north] (L1) at ([xshift=1.75in,yshift=.07in]L0.east) {(b) $G_2 \leftarrow s_2$};
\node [anchor=north] (L2) at ([xshift=2.15in,yshift=.07in]L1.east) {(c) Construct with jumps};
\end{tikzpicture}
\vspace{-10pt}
\captionN{Illustration for Proposition~\ref{propstreamsum}. Note that using the same structure for $G_1 \leftarrow s_1$ and $G_2 \leftarrow s_2$ causes no loss of generality, since arbitrary PFSAs over the same alphabet can be brought to the same structure via possibly non-minimal realizations}\label{figproof1}
\end{figure*}

(2) Let $H=(G_1\leftarrow s_1) + (G_2 \leftarrow s_2)$ and assume $\Theta\left (H, W \right ) \geqq \epsilon$. Also, let the set of  states in the minimal realization of the canonical representation for  $H$ be $Q$, and additionally let  $\Crd(Q)=N > 1$. It follows from the definition of our metric (and the fact that $\epsilon'$-synchronizing strings must occur for all $\epsilon' > 0$), that:
 \cgather{
 \forall q \in Q_H,  \vert \vert \pitilde(q,\cdot) - \mathcal{U}_\Sigma\vert \vert_\infty \geqq \epsilon \label{eq56}
 }
where, as before, $\mathcal{U}_\Sigma =  \begin{pmatrix}
\frac{1}{\vert\Sigma\vert} & \cdots & \frac{1}{\vert\Sigma\vert}
\end{pmatrix}$. We observe that the stream summation algorithm 
can be thought to be producing the output sequence by traversing the arcs in the canonical representation for $H$ augmented with ``jumps'' (See Fig.~\ref{figproof1}), where there are unreported and unlabeled transitions back to the state corresponding to the equivalence class $[ \lambda]$ from each state: whenever we have a mismatch we jump back to $[\lambda]$. The probabilities of these back transitions can be easily computed, but not important here. However, this implies that if $q'$ is a state in the canonical representation for  $G' \leftarrow s'$, then we have (assuming $\pitilde'$ is the morph function for  $G' \leftarrow s'$):
\mltlne{
\pitilde'(q',\cdot)= p(q') \pitilde([\lambda],\cdot) + (1-p(q')) \pitilde(q,\cdot), \\ \textrm{for some} q \in Q, \textrm{where } p(q') \in [0,1]
}
Since, $\pitilde'(q',\cdot)$ is a weighted average of $\pitilde([\lambda],\cdot)$, and $\pitilde(q,\cdot)$ (both of which satisfy Eq.~\eqref{eq56}), it is possible that:
\cgather[4pt]{
 \vert \vert \pitilde'(q',\cdot)- \mathcal{U}_\Sigma\vert \vert_\infty \leqq \epsilon 
}
Assume, if possible, that ($Q'$ is the state set in the minimal realization of the canonical representation for $G' \leftarrow s'$):
\cgather{
\forall q' \in Q', \vert \vert \pitilde'(q',\cdot)- \mathcal{U}_\Sigma\vert \vert_\infty \leqq \epsilon \label{eqass}
}
We note that the same argument in claim (1) implies that:
\cgather{
 \pitilde'([\lambda],\cdot)=  \pitilde([\lambda],\cdot) \geqq \epsilon
}
where the inequality follows from Eq.~\eqref{eq56}.
Since, $ \pitilde'([\lambda],\cdot)$ is some weighted average of all $\pitilde'(q',\cdot)$ vectors, and by assumption Eq.~\eqref{eqass}, the norm of the difference  of each of these vectors from $\mathcal{U}_\Sigma\vert$ is bounded above by $\epsilon$, it follows that:
\cgather{
 \vert \vert \pitilde'([\lambda],\cdot)- \mathcal{U}_\Sigma\vert \vert_\infty \leqq \epsilon 
}
which is a contradiction. Thus, there exists at least one state $q_\star \in Q'$, such that
\cgather{
 \vert \vert \pitilde'(q_\star,\cdot)- \mathcal{U}_\Sigma\vert \vert_\infty \geqq \epsilon 
}
We also note that $G'\leftarrow s'$ has at most $N$ states, since, if 
none of the $\pitilde'$ rows are equal, then we can represent $G'\leftarrow s'$ using the same graph as $H$, with the rows of $\pitilde$ replaced with those from $\pitilde'$.

Now, we compute $\Theta(G'\leftarrow s', W)$.
We note that since $G'\leftarrow s'$ has at most $N$ states, $\lim_{\vert s' \vert \rightarrow \infty}\phi^{s'}(x)$ equals  $\pitilde'(q_\star,\cdot)$ at least once every $N$ levels, which implies:
\cgather{
\Theta(G'\leftarrow s', W) \geqq \frac{\vert \Sigma\vert -1}{\vert \Sigma\vert} \sum_{i=0}^\infty \frac{1}{\vert \Sigma\vert^{2i +N}}\epsilon =  \frac{\epsilon}{\vert \Sigma \vert^N + \vert \Sigma \vert^{N-1}  } 
}

(3) Assume $\Theta\left (H, W \right ) \leqq \epsilon$. It follows from the definition of our metric (and the fact that $\epsilon'$-synchronizing strings must occur for all $\epsilon' > 0$), that:
 \cgather{
 \forall q \in Q_H,  \vert \vert \pitilde(q,\cdot) - \mathcal{U}_\Sigma\vert \vert_\infty \leqq \epsilon \label{eq56}
 }
where, as before, $\mathcal{U}_\Sigma =  \begin{pmatrix}\frac{1}{\vert\Sigma\vert} & \cdots & \frac{1}{\vert\Sigma\vert}
\end{pmatrix}$.
Since, (using the notation used for claim (2) above) $\pitilde'(q',\cdot)$ is a weighted average of $\pitilde([\lambda],\cdot)$,
it follows immediately that:
\cgather{
\forall q' \in Q', \vert \vert \pitilde'(q',\cdot)- \mathcal{U}_\Sigma\vert \vert_\infty \leqq \epsilon
}
and also, that the norm of the difference between  any weighted average (with the weights being positive and summing to unity), of the rows of the $\pitilde$ matrix from $\mathcal{U}_\Sigma$, is bounded above by $\epsilon$, This implies that each term in $\Theta(G'\leftarrow s',W)$ is bounded above by $\epsilon$, which establishes the desired bound:
\cgather[4pt]{
\Theta(G'\leftarrow s',W) \leqq \epsilon
}
This completes the proof.
\end{IEEEproof}

\begin{rem}
Note that the lower bound established in claim (2) is obviously not tight; since if $N=1$, then we have exact summation, whereas the bound is off by a factor of $\vert \Sigma \vert$.
\end{rem}
\begin{rem}
 Proposition~\ref{propstreamsum}  establishes that the stream summation algorithm  works perfectly if the summands sum to a one state machine, which includes FWN. For arbitrary inputs, the deviation of the realized sum from FWN is small if the deviation of the sum of the original models is small; and conversely the deviation of the realized sum from FWN is large if the deviation of the sum of the original models is large.
\end{rem}

\begin{cor}[Contrapositives to Proposition~\ref{propstreamsum}]\label{corinvstream}
Using the notation of  Proposition~\ref{propstreamsum}, we have:
\mltlne{
\textrm{Lower Bound:} \mspace{20mu} \Theta(G'\leftarrow s', W  ) < \epsilon \\ \Rightarrow\Theta( (G_1\leftarrow s_1) + (G_2 \leftarrow s_2),W ) < \left ( \vert\Sigma\vert^N + \vert \Sigma \vert^{N-1} \right )\epsilon
}
\mltlne{
\textrm{Equality:} \mspace{20mu} \Theta(G'\leftarrow s', W  ) =0\\ \Rightarrow\Theta( (G_1\leftarrow s_1) + (G_2 \leftarrow s_2),W ) =0  \label{eqeq}
}
\mltlne{
\textrm{Upper  Bound:} \mspace{20mu} \Theta(G'\leftarrow s', W  ) > \epsilon\\ \Rightarrow\Theta( (G_1\leftarrow s_1) + (G_2 \leftarrow s_2),W ) > \epsilon 
}
\end{cor}
\begin{IEEEproof}
(Equality:) Note that $ G'\leftarrow s' = W $, and the fact that every state in $G'$ is a convex combination of $[\lambda]$ and some state $q \in Q_H$,  implies:
\cgather{
\forall x \in \Sigma^\star, \lim_{\vert s' \vert \rightarrow \infty}\phi^{s'}(x)= p \pitilde([\lambda],\cdot) + (1-p) \pitilde(q,\cdot) = \mathcal{U}_\Sigma
}
Also, since $\pitilde([\lambda],\dot) =  \lim_{\vert s' \vert \rightarrow \infty}\phi^{s'}(\lambda)$ for arbitrary input streams, it follows that:
\cgather{
\forall q \in Q_H, p \mathcal{U}_\Sigma + (1-p) \pitilde(q,\cdot) = \mathcal{U}_\Sigma 
\Rightarrow \forall q \in Q_H, \pitilde(q,\cdot) = \mathcal{U}_\Sigma
}
which establishes Eq.~\eqref{eqeq}.

The other bounds follow by taking the contrapositive of the inequalities established in Proposition~\ref{propstreamsum}.
%
%
\end{IEEEproof}
\begin{rem}
Corollary~\ref{corinvstream} implies that if the output sequence from the stream summation algorithm is FWN, then the summands are exact inverses of each other.
\end{rem}

\subsection{Stream Inversion}
\begin{lem}[Stream Summation to FWN]\label{lemstreaminv}
Let streams $s_i, i=1,\cdots, \vert \Sigma\vert$ be $\vert \Sigma \vert$ independent realizations from a PFSA $G$ defined over the alphabet $\Sigma$. And let $s'$ be generated as follows:
\begin{list}{\labelitemi}{\leftmargin=0.75em}
\item[1] Read current symbols $\sigma_i$ from $s_i$ ($i=1,\cdots,\vert\Sigma\vert$)
\item[2] If $\sigma_i\neq  \sigma_j$ for all distinct  $i,j$, then write 
$\sigma_1$ to output $s'$
\item[3] Move read positions one step to right, and go to step 1
\end{list}
Then, we have:
\cgather{
\Theta(G'\leftarrow s',W) = 0
}
where $W$ is the FWN generator for the alphabet size $\vert\Sigma\vert$.
\end{lem}
\begin{IEEEproof}
Let the set of states for the minimal realization of the canonical representation for $G$ be $Q$, and the corresponding morph function be $\pitilde$. Similarly, let the state set for  $G'\leftarrow s'$ be $Q'$, and the associated probability morph function be $\pitilde'$. 

Let $s_i'$ be the sequence obtained by copying the current symbol from the input stream $s_i$ in Step (2) of the above scheme. (Thus, $s_1' = s'$.)
It is obvious from the symmetry of the scheme that:
\cgather{
\forall i,j \in \{1,\cdots,\vert\Sigma\vert \}, (G'_i \leftarrow s_i') = (G'_j \leftarrow s_j')
}
It follows that, if  $f_i^j$ is the frequency of the $j^{th}$ symbol from the alphabet in the stream $s_i'$, then we have:
\cgather{
\forall i, \lim_{\vert s_i' \vert \rightarrow \infty} \frac{f_i^j}{\vert s_i'\vert} = \pitilde'([\lambda],\cdot)\big \vert_j\\
\Rightarrow \lim_{\vert s_i' \vert \rightarrow \infty} \frac{\sum_i f_i^j}{\vert s_i'\vert} = \lim_{\vert s_i' \vert \rightarrow \infty} \frac{\vert\Sigma\vert f_i^j}{\vert s_i'\vert} = \vert\Sigma\vert \pitilde'([\lambda],\cdot)\big \vert_j
}
where we also used the fact that $\vert s_i'\vert = \vert s_i' \vert$.
Next, noting that we have an output symbol in each $s_i'$ only when each new symbol is distinct, we conclude:
\cgather{
\forall j,k, \sum_i f_i^j = \sum_i f_i^k
}
which in turn implies for the $j^{th}$ and $k^{th}$ alphabet symbols:
\mltlne{
 \lim_{\vert s_i' \vert \rightarrow \infty} \frac{ \sum_i f_i^j}{\vert s_i' \vert} =  \lim_{\vert s_i' \vert \rightarrow \infty} \frac{\vert \Sigma\vert f_i^j}{\vert s_i'\vert}  = \vert \Sigma\vert \pitilde'([\lambda],\cdot)\big \vert_j \\ = \lim_{\vert s_i' \vert \rightarrow \infty} \frac{ \sum_i f_i^k}{ \vert s_i' \vert}   =  \lim_{\vert s_i' \vert \rightarrow \infty} \frac{\vert \Sigma\vert f_i^k}{\vert s_i'\vert} = \vert \Sigma\vert \pitilde'([\lambda],\cdot)\big \vert_k 
}
and hence we conclude:
\cgather{
\pitilde'([\lambda],\cdot) = \lim_{\vert s'\vert \rightarrow \infty} \phi^{s'}(\lambda) = \mathcal{U}_\Sigma \label{eq72}
}

Next, denote the $r^{th}$ symbol in the stream $s'$ as $s'(r)$. Then, if we assume that the streams $s_i$ were all synchronized to the same state of $G$ just prior to the generation of $s'(r)$, we have:


\mltlne{
\forall \sigma_i,\sigma_k \in \Sigma, Prob(s'(r)=\sigma_i)  \\ = c\prod_{\sigma_j \in \Sigma}\pitilde(\wp_\lambda,\sigma_j) = Prob(s'(r)=\sigma_k) 
}
which implies:
\cgather{
\forall \sigma_k \in \Sigma,  Prob(s'(r)=\sigma_k) = \frac{1}{\vert \Sigma\vert}\label{eq80}
}

Next, we consider the following construction: Consider the PFSA $G$, with the streams $s_i$ traversing the transitions via the symbol-labeled arcs, with each $s_i$ initialized to the state $[\lambda]$. Note that we have a new symbol in the output $s'$ if all current symbols in the $\vert \Sigma\vert$ input  streams are distinct; which can occur in two possible ways:
\begin{enumerate}
\item all $s_i$ streams are synchronized to some state $q \in Q$, and a distinct symbol is generated for each $s_i$
\item no such synchronization; but the symbols generated are distinct
\end{enumerate}
In the second case, we assume that a re-initialization occurs; $i.e.$, all the streams jump back to state $[\lambda]$ before the distinct symbols are generated causing the new output symbol.

Note that this construction causes no loss of generality, as we are simply defining a path, with jumps, for the given input streams through the PFSA $G$. 

Denote the probability distribution of the output symbol, when the streams are synchronized at some state $q \in Q$ be $v^q$, $i.e.$ $v^q_i$ is the probability of seeing the $i^{th}$ symbol, given that we indeed have a new output symbol.
 
Next we observe that the streams $s_i$ can be assumed to be synchronized at $[\lambda]$ when the first symbol appears in the output $s'$ (since deletion of  arbitrary leading prefixes has no effect in the limit of infinite data). Thus, we have (from Eq.~\eqref{eq72}):
\cgather{
\lim_{\vert s' \vert \rightarrow \infty} \phi^{s'}(\lambda) = \mathcal{U}_\Sigma
}
 Note that the next symbol may be produced after a ``silent'' jump to some state $q \in Q$. Additionally, the probability that the jump occurs to a specific state $q$ is an explicit function of the parameters (morph probabilities, and transition structure) of the PFSA $G$. However we do not need to compute these probabilities; we simply conclude that:
\cgather{
\forall x \in \Sigma^\star, \lim_{\vert s' \vert \rightarrow \infty} \phi^{s'}(x) = \sum_{q \in Q} p(q,x) v^q\\
\textrm{where } p(q,x) \in [0,1], \sum_q p(q,x) = 1
}
Noting that Eq.~\eqref{eq80} establishes that $\forall q \in Q, v^q = \mathcal{U}_\Sigma$, we conclude:
\cgather{
\forall x \in \Sigma^\star, \lim_{\vert s' \vert \rightarrow \infty} \phi^{s'}(x) = \mathcal{U}_\Sigma
}
which establishes that $G'\leftarrow s'$ is the FWN generator. This completes the proof.
\end{IEEEproof}

\begin{prop}[Stream Inversion]\label{propstreaminversion}
Given a stream $s$ which is generated by some hidden PFSA $G$, let stream $s'$ be generated via:
\begin{list}{\labelitemi}{\leftmargin=0.75em}
\item[1] Generate $\vert \Sigma \vert-1$ independent copies of $s_1$: $s_1,\cdots, s_{\vert\Sigma\vert -1}$
 \item[2] Read current symbols $\sigma_i$ from $s_i$ ($i=1,\cdots,\vert\Sigma\vert-1$)
\item[3] If $\sigma_i\neq  \sigma_j$ for all distinct  $i,j$, then write 
$\Sigma \setminus \bigcup_{i=1}^{\vert \Sigma \vert-1}\sigma_i$
to output $s'$
\item[4] Move read positions one step to right, and go to step 1
\end{list}
Then, we have:
\cgather{
\Theta(-G,G'\leftarrow s')=0
}
\end{prop}
\begin{IEEEproof}
Follows immediately from Lemma~\ref{lemstreaminv}.
\end{IEEEproof}
\begin{prop}[Asymptotic Complexity]\label{propcompleity}
 The asymptotic time complexity of carrying out the stream operations is $O(\vert s \vert \vert \Sigma \vert$
\end{prop}
\begin{IEEEproof}
 The algorithmic steps in each of the operations of stream copy and stream summation  proceed in a symbol-by-symbol fashion, with no memory of previous symbols. Also, each step involves a constant number of integer comparisons. Assuming that each new symbol from the FWN processes involved can be generated with constant complexity, we conclude that the asymptotic time complexity of both stream summation and stream copy is $O(\vert s \vert$. The stream inversion operation needs to generate $\vert \Sigma \vert -1$ stream copies, implying that its asymptotic time complexity is $O(\vert s \vert \vert \Sigma \vert)$.
\end{IEEEproof}
\begin{figure}[t]
\hspace{10pt}
 \begin{tikzpicture}
\begin{scope}[scale=1, shift={(9,0)},font=\bf \fontsize{8}{8}\selectfont]
 \pgfplotsset{every axis legend/.append style={
at={(0.265,0.625)},
anchor=south}}
\begin{semilogyaxis}[ legend cell align=left,legend style={draw=gray,font= \fontsize{8}{8}\selectfont},
axis line style={black, opacity=0.5,  thick, rounded corners=0pt},
axis on top=true, 
scale=1,grid style={dashed, gray!30},
enlargelimits=false, 
width=2.5in, 
height=2in,     
,ymax=1, 
 ymin=0.00001, 
semithick,grid, 
axis background/.style={fill=black!0,},
xlabel={Alphabet Size ($\vert \Sigma \vert$)},
yticklabel style={xshift=.025in},
ylabel={Upper bound on \\ annihilation efficiency},
ylabel style={, align=center,yshift=.1in},
    scaled x ticks = false,   
      x tick label style={/pgf/number format/fixed,
      /pgf/number format/1000 sep = \thinspace 
      } 
  ];
\input{Figures/beta.tex}
\node at (axis cs:2.75,.25) {0.25};
\node at (axis cs:3.75,.075) {0.074};
\node at (axis cs:4.75,.024) {0.023};
\end{semilogyaxis}
\end{scope}
\end{tikzpicture}
\vspace{-5pt}

\captionN{Upper bound on annihilation efficiency $ \frac{(\vert \Sigma \vert -1)! }{ \vert \Sigma \vert^{\vert \Sigma \vert}}$ vs  alphabet size $\vert \Sigma \vert$. This illustrates why having a fine quantization would necessitate large amounts of data to pass the self-annihilation test.}\label{figS5}
\end{figure}
\subsection{Annihilation Efficiency}
To pass the self-annihilation test, a data stream must be sufficiently long; and  the required length $\vert s \vert $ of the  input $s$ with a specified threshold $\epsilon^\star$ is dictated by the characteristics of the generating process. Thus  the rate of convergence of the self-annihilation error as a function of  $\vert s\vert$ quantifies the sample  complexity of information annihilation. Let $s'$ be obtained from $s$ via stream inversion, and $s''$ be obtained via stream summation of $s$ and $s'$. Then, it follows that $s''$ is always a realization of the FWN process, which has an uniform probability of generating any symbol at any point. Thus, for any $x\in \Sigma^\star$, the vectors $\phi^{s''}(x)$ (See Table~\ref{tab1}, row 4) are empirical distributions which converge to the flat distribution as $\vert s''\vert \rightarrow \infty$. Additionally,  the Central Limit Theorem (CLT)~\cite{fell45} dictates  the convergence rate to scale as $1/\sqrt{\vert s''\vert}$ irrespective of the generating process for the input $s$. However, 
 selective erasure in annihilation (See Table~\ref{tab1}) implies that $\vert s''\vert < \vert s\vert$, and the expected shortening ratio $\beta = \mathbf{E}(\vert s''\vert / \vert s\vert) $ does indeed depend on the generating process. We refer to  $\beta$ as the \textit{annihilation efficiency}, since the convergence rate of the self-annihilation error scales as $1/\sqrt{\beta\vert s\vert}$. 
 
 Next, we compute $\beta$ in terms of the symbol frequencies:

\begin{prop}\label{propbeta}
 Given an input stream $s$, let stream $s'$ be produced via stream inversion from $s$, and let $s''$ be produced via stream summation of $s$ and $s'$. Let $p_i$ be the probability of observing symbol $\sigma_i \in \Sigma$. Then, we have
\cgather{
\beta = \mathbf{E}(\vert s''\vert / \vert s \vert ) = (\vert \Sigma \vert -1)!\prod_i p_i 
}
\end{prop}
\begin{IEEEproof}
 To generate $s'$ from $s$, we first need to generate $\vert \Sigma \vert -1$ independent stream copies of $s$. It is clear from the stated algorithm for independent stream copy (See Table~\ref{tab1} in main text and Proposition~\ref{propstreamcopy}) that the expected length of each of these copies is $\frac{1}{\vert \Sigma\vert} \vert s\vert$. The probability of obtaining a symbol in the output by comparing these $\vert \Sigma \vert -1$ streams (to get $s'$) is simply the probability of seeing a different symbol in each of the copied streams (as stated in the algorithm for stream inversion in Table~\ref{tab1} of main text). Denoting this probability as $\alpha$, we have:
\mltlne{
\alpha = (\vert \Sigma \vert -1)!\sum_{j=1}^{\vert \Sigma \vert} \prod_{i\neq j} p_i\\
= (\vert \Sigma \vert -1)! \left \{ \prod_{i} p_i \right \}\left \{ \sum_{i} \frac{1}{p_i} \right \}
= \frac{1}{H} \vert \Sigma \vert! \prod_{i} p_i \label{eq32}
}
where $H$ is the harmonic mean of the probability vector $p$.
Thus, the expected length of $s'$ is $\alpha \vert s \vert$. The final step is stream summation of $s$ and $s'$ to obtain $s''$. We note that the probability of seeing symbol $\sigma_i$ in the inverted stream $s'$ is $k/p_i$, where $k=\frac{1}{\sum_i \frac{1}{p_i}} = \frac{H}{\vert \Sigma \vert}$. It follows that stream summation of $s$ and $s'$, would result in an expected length of $\frac{H}{\vert \Sigma \vert}\vert s'\vert$, which when combined with Eq.~\eqref{eq32} completes the proof. 
\end{IEEEproof}
\begin{cor}
 Annihilation efficiency satisfies:
\cgather{
(\vert \Sigma \vert -1)! \eta^{\vert \Sigma \vert} \{ \frac{1}{\eta} +1 -\vert \Sigma \vert \} \leqq \beta \leqq \frac{(\vert \Sigma \vert -1)! }{ \vert \Sigma \vert^{\vert \Sigma \vert}}
}
where $\eta$ is the probability of occurrence of the rarest symbol in the input stream $s$.
\end{cor}
\begin{IEEEproof}
 The upper bound follows from noting that the product $\prod_i p_i$ is maximized when $\forall i, p_i = 1/\vert \Sigma\vert$. The lower bound is obtained by assuming that $\eta = \min_i p_i$, upon which the minimum value of the product $\prod_i p_i$ is given by:
\cgather{
\eta^{\vert \Sigma \vert -1}(1-\eta (\vert \Sigma \vert -1)) = \eta^{\vert \Sigma \vert }\{ \frac{1}{\eta} +1 -\vert \Sigma \vert \}
}
\end{IEEEproof}
\begin{rem}
 The upper bound for the annihilation efficiency is realized if the input  $s$ is FWN. 
\end{rem}
\subsection{Distance Between Hidden Generators}
\begin{defn}[FWN Deviation Estimators]\label{defdevest}
 For a string $s$, the complete white noise deviation estimator $\zeta(s)$ is defined as:
\cgather[5pt]{
\zeta(s) =\frac{\vert \Sigma \vert -1}{\vert \Sigma\vert}\sum_{x \in \Sigma^\star} \left \{\frac{1}{\vert \Sigma \vert^{2\vert x \vert}} \vert \vert \phi^s(x) - \mathcal{U}_\Sigma  \vert \vert_\infty\right \}
}
And the partial white noise deviation estimator $\hat{\zeta}(s,\ell)$ is defined:
\cgather[5pt]{
\hat{\zeta}(s,\ell) =\frac{\vert \Sigma \vert -1}{\vert \Sigma\vert}\sum_{x \in \Sigma^\ell} \left \{\frac{1}{\vert \Sigma \vert^{2\vert x \vert}} \vert \vert \phi^s(x) - \mathcal{U}_\Sigma  \vert \vert_\infty\right \}
}
which only carries out the summation for all strings up to length $\ell$. 
\end{defn}
\begin{prop}[Causality Claim 1]\label{prophiddenmodel}
 Given a stream $s$, and denoting the hidden generator for $s$ as $G_s$, and the zero model as $W$, we have:
\cgather[5pt]{
\lim_{\vert s \vert \rightarrow \infty} \vert \Theta(G_s,W) -\zeta(s) \vert =0 
 }
\end{prop}
\newcommand{\supdiffphi}{\vert \vert \phi^s(x) -\phi^\star(x) \vert \vert_\infty)}
\begin{IEEEproof}
For a string $s'$ generated by the model $G(s)$, we denote $\displaystyle\lim_{\vert s' \vert \rightarrow \infty} \phi^{s'}(x)$ as $\phi^\star(x)$. Then:
\mltlne{
\Theta(G(s),W)= \frac{\vert \Sigma \vert -1}{\vert \Sigma\vert}\sum_{x \in \Sigma^\star} \frac{1}{\vert \Sigma \vert^{2\vert x \vert}} \vert \vert \phi^\star(x) - \mathcal{U}_\Sigma    \vert \vert_\infty \notag \\
\shoveleft{\leqq  \frac{\vert \Sigma \vert -1}{\vert \Sigma\vert}\sum_{x \in \Sigma^\star} \left \{ \frac{1}{\vert \Sigma \vert^{2\vert x \vert}} \vert \vert \phi^\star(x) - \phi^s(x)\vert \vert_\infty \right \}} \\ \shoveright{+ \frac{\vert \Sigma \vert -1}{\vert \Sigma\vert}\sum_{x \in \Sigma^\star} \left \{\frac{1}{\vert \Sigma \vert^{2\vert x \vert}} \vert \vert \phi^s(x) - \mathcal{U}_\Sigma  \vert \vert_\infty\right \}}\notag\\
\leqq  \underbrace{\frac{\vert \Sigma \vert -1}{\vert \Sigma\vert}\sum_{x \in \Sigma^\star} \left \{\frac{1}{\vert \Sigma \vert^{2\vert x \vert}} \vert \vert \phi^\star(x) - \phi^s(x)  \vert \vert_\infty \right \}}_{B}+ \zeta(s) 
}
Starting with $\zeta(s)$ on the RHS, we end up with:
\mltlne{
\zeta(s) \leqq \Theta(G(s),W) \\ + \frac{\vert \Sigma \vert -1}{\vert \Sigma\vert}\sum_{x \in \Sigma^\star} \left \{\frac{1}{\vert \Sigma \vert^{2\vert x \vert}} \vert \vert \phi^\star(x) - \phi^s(x)  \vert \vert_\infty \right \}
}
which then implies:
\cgathers{
 \vert \Theta(G_s,W) -\zeta(s) \vert \leqq \frac{\vert \Sigma \vert -1}{\vert \Sigma\vert}\sum_{x \in \Sigma^\star} \left \{\frac{1}{\vert \Sigma \vert^{2\vert x \vert}} \vert \vert \phi^\star(x) - \phi^s(x)  \vert \vert_\infty \right \}
}
We note that $\phi^s(x)$ is an empirical estimate of $\phi^\star(x)$, which then implies via the Glivenko-Cantelli theorem~\cite{Fl70} that
\cgather[5pt]{
\forall x \in \Sigma^\star, \supdiffphi \xrightarrow{a.s.} 0
}
which completes the proof.
\end{IEEEproof}
\vspace{5pt}

Finally we establish our causality claim:
while the deviation from FWN  is estimated by function $\hat{\zeta}(s,\ell)$ from a finite observed  string $s$ and consideration of finite  histories of length bounded by $\ell$,  it converges to the deviation of the underlying process from FWN  in the limit of infinite data (See next proposition). It thus follows that the distance $\epsilon$ calculated by annihilating a stream $s$ against a second stream $s'$ converges to the absolute deviation of $G_s -G_{s'}$ from the FWN generator $W$.
%
\begin{figure}[t]
\tikzexternaldisable
\centering 
\definecolor{blcol}{RGB}{150,150,250}
\definecolor{murcol}{RGB}{255,150,150}
\definecolor{healcol}{RGB}{150,235,150}
\definecolor{colzeta}{RGB}{250,250,200}
\definecolor{leftpcol}{RGB}{200,255,200}
\definecolor{rightpcol}{RGB}{200,200,255}
\definecolor{boxcol}{RGB}{210,200,200}
\definecolor{linecol}{RGB}{200,180,180}
\definecolor{colplus}{RGB}{235,220,220}
\definecolor{colinv}{RGB}{200,220,255}
\definecolor{cof}{RGB}{219,144,71}
\definecolor{pur}{RGB}{200,200,200}
\definecolor{greeo}{RGB}{91,173,69}
\definecolor{greet}{RGB}{52,111,72}
 \definecolor{nodecol}{RGB}{180,180,220}
 \definecolor{nodeedge}{RGB}{240,248,255}
  \definecolor{nodecolb}{RGB}{220,180,180}
  \definecolor{nodecolc}{RGB}{180,220,180}
  \definecolor{nodecolcD}{RGB}{100,160,100}
  \definecolor{nodecolW}{RGB}{190,190,190}
  \definecolor{edgecol}{RGB}{160,160,180}
 \definecolor{nodecolD}{RGB}{140,140,180}
  \definecolor{nodecolbD}{RGB}{180,140,140}
\tikzset{%
fshadow/.style={      preaction={
         fill=black,opacity=.1,
         path fading=circle with fuzzy edge 20 percent,
         transform canvas={xshift=1mm,yshift=-1mm}
       }}
}
\begin{tikzpicture}
\newsavebox\Gtwostate
\sbox{\Gtwostate}{%
\begin{scope}[->,>=stealth',shorten >=1pt,auto,node distance=1.2cm,
                    semithick,scale=.65,font=\bf \fontsize{6}{6}\selectfont]
  \tikzstyle{every state}=[fill=DodgerBlue3,draw=nodeedge,text=black,minimum size=3, text width=2,scale=1,fshadow, opacity=.5, text opacity=1]
  \node[state] (A)         []           {$\mspace{-6mu}q_1$};
  \node[state]         (B) [right of=A] {$\mspace{-6mu}q_2$};
  \path (A) edge   [draw=edgecol,bend left]           node {$\sigma_1\vert 0.7$} (B)
        (A) edge [draw=edgecol,in=120,out=60,loop,above] node {$\sigma_0\vert 0.3$} (A)
        (B) edge [draw=edgecol,in=60,out=120,loop,above] node {$\sigma_1\vert 0.9$} (B)
            edge   [draw=edgecol,bend left]           node {$\sigma_0\vert 0.1$} (A);
\end{scope}
}
\newsavebox\Gfourstate
\sbox{\Gfourstate}{%
\begin{scope}[->,>=stealth',shorten >=1pt,auto,node distance=1.2cm,
                    semithick,scale=.65,font=\bf \fontsize{6}{6}\selectfont]
  \tikzstyle{every state}=[fill=Red1,draw=nodeedge,text=black,minimum size=3, text width=2,scale=1,fshadow, opacity=.5, text opacity=1]
  \node[state] (A)         []           {$\mspace{-6mu}q_1$};
  \node[state]         (B) [right of=A] {$\mspace{-6mu}q_2$};
  \node[state]         (C) [below of=B] {$\mspace{-6mu}q_3$};
  \node[state]         (D) [below of=A] {$\mspace{-6mu}q_4$};
  \path 
(A) edge   [draw=edgecol,bend left]           node {$\sigma_1\vert 0.7$} (B)
(B) edge   [draw=edgecol,bend left, below]           node [rotate=90] {$\sigma_1\vert 0.7$} (C)
(C) edge   [draw=edgecol,bend left]           node {$\sigma_0\vert 0.7$} (D)
(D) edge   [draw=edgecol,bend left,above]           node [rotate=90]{$\sigma_0\vert 0.7$} (A)
        (A) edge [draw=edgecol,in=170,out=120,loop,above] node [xshift=.075in]{$\sigma_0\vert 0.3$} (A)
        (C) edge [draw=edgecol,in=-10,out=-60,loop,right] node [yshift=0.05in]{$\sigma_1\vert 0.9$} (C)
        (D)    edge   [draw=edgecol,bend left, below]           node [rotate=45, xshift=.0in, yshift=0.025in]{$\sigma_1\vert 0.1$} (B)        (B)    edge   [draw=edgecol,bend left,above]           node[rotate=45, xshift=0in, yshift=-.05in] {$\sigma_0\vert 0.1$} (D);
\end{scope}
}

\begin{scope}[scale=1, shift={(0,0)},font=\bf \fontsize{8}{8}\selectfont]
 \pgfplotsset{every axis legend/.append style={
at={(0.265,0.625)},
anchor=south}}
\begin{loglogaxis}[ legend cell align=left,legend style={draw=gray,font= \fontsize{8}{8}\selectfont},
axis line style={black, opacity=0.5,  thick, rounded corners=0pt},
axis on top=true, 
scale=1,grid style={dashed, gray!30},
enlargelimits=false, 
width=3.25in, 
height=2.75in,     
semithick,grid,
 axis background/.style={
shade,top color=gray!30,bottom color=white},
xlabel={Symbol Length [No. of symbols]},xlabel style={yshift=.051in},
yticklabel style={xshift=.025in},
ylabel={Self-annihilation Error},
ylabel style={yshift=-.1in},
    scaled x ticks = false,
      x tick label style={/pgf/number format/fixed,
      /pgf/number format/1000 sep = \thinspace 
      }
  ]; 
\input{Figures/conv1.txt}
\input{Figures/conv2.txt}
\draw [dashed,black, very thick] (axis cs:100,0.085) -- (axis cs:9000,0.0085);
\end{loglogaxis}
\node[inner sep=65pt] at (4.85,2.5) (NgoneT1) {\usebox{\Gtwostate}};
\node[inner sep=0pt] at (1,2) (NgoneT2) {\usebox{\Gfourstate}};
\node[inner sep=2pt, draw=black, fill=pur] at (2.5,4.5) (NgoneT3) {$\sim O(n^{-\frac{1}{2}}) $};
\node[inner sep=2pt, draw=DodgerBlue2, fill=pur, font= \fontsize{8}{8}\selectfont] at (5.1,4.7) (NgoneT4) {\color{DodgerBlue3}\txt{Smaller \# of states\\But slower \\convergence}};
\draw[->,>=stealth',shorten >=1pt,  thick, drop shadow] (NgoneT3) -- (3,2.35);
\draw[->,>=stealth',shorten >=1pt,  thick, drop shadow] (NgoneT4) -- (4.75+.75,3.15+0);
\end{scope}
%
%
\end{tikzpicture}
\vspace{-35pt} 

\captionN{\textbf{ Convergence rate of the self-annihilation error} Shown to scale as $O(1/\sqrt{n})$ as dictated by the Central Limit Theorem. The convergence rates do not depend directly on the descriptional complexity of the generating processes; note that the data from the two state process has a slower convergence rate compared to that from the four state process. As discussed in the section on computational complexity, the convergence   rate scales as  $O(\sqrt{\beta n})$ where $\beta$ is the expected shortening of the input 
 stream due to the selective erasure of the symbols in the different steps of the annihilation process.
 We establish in Proposition~\ref{propbeta} that  if $p_i$ is the occurrence probability of the symbol $\sigma_i$ in the stream $s''$, then we have: $\beta = (\vert \Sigma \vert -1) !\prod_i p_i $.
}\label{fig4C}
\end{figure}
\begin{prop}[Causality Claim 2]\label{prophiddenmodeltwo}
 Given a stream $s$, and denoting the hidden generator for $s$ as $G_s$, and the zero model as $W$, we have:
\cgather[5pt]{
\lim_{\vert s \vert \rightarrow \infty} \vert \Theta(G_s,W) -\hat{\zeta}(s,\ell) \vert \leqq \epsilon 
 }
where the partial estimator is evaluated upto length $\displaystyle\ell = \frac{\ln\frac{1}{\epsilon}}{\ln \vert \Sigma \vert} $
\end{prop}

\begin{IEEEproof}Noting that we have:
 \mltlne{
\vert \Theta(G_s,W) -\hat{\zeta}(s,\ell) \vert \leqq \vert \Theta(G_s,W) -\zeta(s) \vert \\ +
\left \vert   \frac{\vert \Sigma \vert -1}{\vert \Sigma\vert}\sum_{x \in \Sigma^\star \setminus \Sigma^\ell} \left \{\frac{1}{\vert \Sigma \vert^{2\vert x \vert}} \vert \vert \phi^s(x) - \mathcal{U}_\Sigma  \vert \vert_\infty\right \}\right  \vert \notag\\
 \leqq \vert \Theta(G_s,W) -\zeta(s) \vert + \frac{1}{\vert \Sigma \vert^\ell}
}
the result follows using Proposition~\ref{prophiddenmodel}.
\end{IEEEproof}

\section{Quantization Techniques}\label{sec-quant}

Information annihilation operates on symbolic sequences. Thus, we need to specify a quantization scheme to map possibly continuous-valued data streams to symbolic sequences. This is accomplished by the choice of a symbol alphabet, where each letter in the alphabet denotes a slice of the data range. 
Given a particular quantization scheme, we map each continuous-valued observation to the symbol representing the slice of the data range to which the observation belongs. Thus any chosen quantization scheme incurs error, which can be made  small by using a fine quantization, $i.e.$, a large alphabet.

However, the length of the observed data limits the size of the alphabet that we can use. This is a direct consequence of the fact that the \textit{annihilation efficiency} falls rapidly with the alphabet size (See Proposition~\ref{propbeta}, and Fig.~\ref{figS5}). Thus, if $s$ is the input stream, $s'$ is obtained via stream inversion from $s$,    and $s''$ is the output from stream summation of $s$ and $s'$, then the expected ratio of the lengths $\vert s'' \vert / \vert s \vert $ falls rapidly as the alphabet size $\vert \Sigma \vert$ is increased, making the estimation of the deviation of $s''$ from FWN more and more difficult. Since the convergence rate of the self-annihilation error scales as $1/\sqrt{\beta \vert s \vert}$, it follows that the self-annihilation error increases rapidly with finer quantization (See Fig.~\ref{figS1} for illustration on the EEG dataset). 

\subsection{Desired Properties of Quantization Schemes}
It follows that a good quantization scheme is defined by the following properties:
\begin{enumerate}
\item The frequency of the rarest symbol in the quantized data streams are too small. This is to  ensure that  symbols are represented faithfully according to its generation probability from each state in the hidden model; too few occurrences of a particular symbol may represent statistical fluctuations rather than the generation probabilities.
\item The average self-annihilation error for the observed data streams is small, $i.e.$, if $\epsilon_{ii}$ is the self-annihilation error for the observed data stream $s_i$, then we require that $\displaystyle \frac{1}{T}\sum_{i=1}^T \epsilon_{ii}$ is small (where $T$ is the total number of observed data streams).
\item The average discrimination between data-streams is high, $i.e.$, if for two streams $s_i,s_j$, the similarity computed by information annihilation is $\epsilon_{ij}$, then we require that $\displaystyle \frac{1}{T(T-1)}\sum_{i=1}^T\sum_{\begin{subarray}{c}
j=1\\ i\neq j                                     
                                    \end{subarray}
}^T\epsilon_{ij}$
 is large (where $T$ is the total number of observed data streams).
\end{enumerate}

One approach to choosing a quantization that satisfies the stated properties is the following: We restrict ourselves to maximum-entropy quantizations, $i.e.$,  schemes in which each symbol occurs with the same frequency in the data set. In Fig.~\ref{figS1} plates (a)-(c), we show three such maximum-entropy schemes for the EEG-dataset. The alphabet size is increased from $2$ to $4$, and we choose the slices of the data-range such that each slice contains an approximately equal number of data points. For example, in plate (c) of Fig.~\ref{figS1}, each of the four slices contains approximately $25\%$ of the total number of observations in the data set.
Such maximum-entropy schemes guarantee that property $1$ (See above) is satisfied. For the remaining properties, we plot the mean self-annihilation error and the mean discrimination, for each alphabet size. As expected, we see that finer alphabets lead to high average  discrimination, while at the same time incur high average self-annihilation errors (See Fig.~\ref{figS1}(d)). The ratio of the two quantities is more useful, and in Fig.~\ref{figS1}(e), we note that the trinary maximum-entropy quantization minimizes this ratio; implying high discrimination and low self-annihilation error.
\begin{figure}[t]
%
\centering
\input{Figures/figSIfour.tex}
\vspace{-20pt}

\captionN{Plates(a)-(c): Maximum-entropy quantization schemes for the EEG dataset with alphabet sizes $2$, $3$, and $4$ respectively ($1.5s$ of a single time series shown for clarity). As explained in the text, ``maximum-entropy'' in this context refers to the fact that each data slice contains approximately an equal number of observations. Plate (d) shows how the average self-annihilation error, as well as the average discrimination between different data streams,  increases exponentially with alphabet size. Plate (e) shows that the ratio of the average self-annihilation error to average discrimination has a minimum at an alphabet size $3$. }\label{figS1}
\end{figure}

We note that if our chosen quantization is too coarse, then distinct processes may evaluate to be similar. However, too coarse an alphabet produces errors in only one direction; identical processes will still evaluate to be identical (or nearly so), provided the streams pass the self-annihilation test.

\section{Comparison Against Simple Statistical Approches to Similarity}\label{sec-addl}
\subsection{Smashing \& Mutual Information}\label{sec-addla}
Smashing two finite quantized  data streams manipulates the statistical information 
contained in them. Notions of information-theoretic interdependence of sequential data have been investigated
in the literature extensively; one such concept is \textit{mutual information} between streams. For discrete random variables, mutual information quantifies 
the  amount  of  information  one  random  variable  contains  about  another.

\begin{figure}[t]
\centering
\def\SCALE{.75}
\tikzset{%
fshadow/.style={      preaction={
         fill=black,opacity=.1,
         path fading=circle with fuzzy edge 20 percent,
         transform canvas={xshift=1mm,yshift=-1mm}
       }}
}
\mnp{1.7in}{\centering
\begin{tikzpicture}[scale=\SCALE,->,>=stealth',shorten >=1pt,auto,node distance=1.2cm,
                    semithick,font=\bf \fontsize{8}{8}\selectfont]
  \tikzstyle{every state}=[fill=IndianRed1,draw=black,text=black,minimum size=3, text width=2,scale=\SCALE,,fshadow]
  \node[state,  thick] (A)         [align=center]           {$\mspace{-7mu}q_1$};
  \path  (A) edge [draw=Red4,in=120,out=60,loop,above,   thick] node {$\sigma_0\vert 0.7$} (A)
        (A) edge [draw=Red4,in=-60,out=-120,loop,below,  thick] node {$\sigma_1\vert 0.3$} (A);
\node [draw=black, path fading=east,fill=IndianRed1!50,font=\sffamily \fontsize{6}{6}\selectfont,label=below:{\bf \fontsize{8}{8}\selectfont Stream A1}] (A1) at ([xshift=1in, yshift=.475in]A.east) [align=left] {\sffamily
110001100101\\
001101000100\\
100100001100\\
011000011100\\
010010100010\\
100101000110\\
10001100$\cdots$};        
\node [draw=black, path fading=east,fill=IndianRed1!50,font=\sffamily \fontsize{6}{6}\selectfont,,label=below:{\bf \fontsize{8}{8}\selectfont Stream A2}] (A2) at ([xshift=1in, yshift=-.475in]A.east) [align=left] {\sffamily
010001111100\\
000100100000\\
100010100111\\
010001011000\\
000001000000\\
100100100001\\
10001110$\cdots$}; 
        \node [anchor=north,font=\bf \fontsize{8}{8}\selectfont] at ([yshift=-1in, xshift=1in]A.south) {(a) Model A \& sample streams};
        \node [anchor=north,font=\bf \fontsize{8}{8}\selectfont] at ([yshift=-.5in]A.south) { Model A};
\draw [->,>=latex, dashed, thick ] (A.east) -- (A1.west);        
\draw [->, >=latex, dashed, thick ] (A.east) -- (A2.west);        
\end{tikzpicture}
}
\mnp{1.7in}{\centering
\begin{tikzpicture}[scale=\SCALE,->,>=stealth',shorten >=1pt,auto,node distance=1.2cm,
                    semithick,font=\bf \fontsize{8}{8}\selectfont]
  \tikzstyle{every state}=[fill=DodgerBlue1!50,draw=black,text=black,minimum size=3, text width=2,scale=\SCALE,,fshadow]
  \node[state,  thick] (A)         [align=center]           {$\mspace{-7mu}q_1$};
  \path  (A) edge [draw=DodgerBlue2,in=120,out=60,loop,above,   thick] node {$\sigma_0\vert 0.9$} (A)
        (A) edge [draw=DodgerBlue2,in=-60,out=-120,loop,below,  thick] node {$\sigma_1\vert 0.1$} (A);
\node [draw=black, path fading=east, ,fill=lightgray,text=DodgerBlue2,font= \fontsize{6}{6}\selectfont,label=below:{\bf \fontsize{8}{8}\selectfont Stream B1}] (B1) at ([xshift=1in, yshift=.475in]A.east) [align=left] {\sffamily
100000100000\\
000000010100\\
000000001110\\
000000000011\\
000000010000\\
000000000010\\ 
00000000$\cdots$};          
\node [draw=black, path fading=east, ,fill=lightgray,text=DodgerBlue2,font= \fontsize{6}{6}\selectfont,label=below:{\bf \fontsize{8}{8}\selectfont Stream B2}] (B2) at ([xshift=1in, yshift=-.475in]A.east) [align=left] {\sffamily
000000000000\\
101100000000\\
000100001010\\
000000101100\\
001000000000\\
000100000010\\
00000100$\cdots$};        
        \node [anchor=north,font=\bf \fontsize{8}{8}\selectfont] at ([yshift=-1in, xshift=1in]A.south) {(a) Model B \& sample streams};
        \node [anchor=north,font=\bf \fontsize{8}{8}\selectfont] at ([yshift=-.5in]A.south) { Model B};
\draw [->,>=latex, dashed, thick ] (A.east) -- (B1.west);        
\draw [->, >=latex, dashed, thick ] (A.east) -- (B2.west);        
\end{tikzpicture}
}
\vspace{0pt}

\captionN{Two distinct  PFSA models (alphabet $\Sigma=\{\sigma_0,\sigma_1\}$) and initial sections of generated sample streams ($\sigma_0$ shown as $0$, and $\sigma_1$ shown as $1$)): While streams generated in independent runs of the same model have near zero mutual information, they are correctly evaluated as having similar generators via data smashing (See Tables~\ref{tabmi1} and \ref{tabmi2}). Also, runs from the different models also have near-zero mutual information, while smashing them correctly reveals a significant difference in the  generators. }\label{figmi}
\end{figure}
Formally, let  $X,Y$  be   discrete  random  variables
with  alphabets  $\Sigma_X, \Sigma_Y$ and  probability  mass  function $ p(x)  =  Prob\{X  =  x :   x  \in \Sigma_X\},  p(y)  =  Prob\{Y  =  y :   y  \in \Sigma_Y\} $.  Also, considering $(X,Y)$ as a single vector vector-valued random variable, we have the mass function $p(x,y)$.  Then, mutual information between the discrete random variables $X,Y$ is defined as:
\cgather{
I(X,Y) = \sum_{y \in \Sigma_Y}\sum_{x \in \Sigma_X} p(x,y) \log \left (  \frac{p(x,y)}{p(x) p(y)}  \right )
} 
Mutual information is related to the notion of entropy: the  entropy  of  a  random  variable  is  a measure  of  the  amount  of  information  required 
on  the  average  to  describe  the  random  variable; while mutual information is the amount of information one variable contains about the other; or, more precisely, the degree to which the uncertainty in one can be reduced by knowing about the other.

Needless to say, if two data streams $X,Y$  are generated independently from the same underlying generator, then we have:
\cgather{
I(X,Y) = \sum_{y \in \Sigma_Y}\sum_{x \in \Sigma_X} p(x,y) \log \left (  \frac{p(x,y)}{p(x) p(y)}  \right ) \notag \\ = \sum_{y \in \Sigma_Y}\sum_{x \in \Sigma_X} p(x)p(y) \log \left (  \frac{p(x)p(y)}{p(x) p(y)}  \right ) =0
}

Thus, sharing a common generative process does not imply a high mutual information; and conversely, high mutual information is indicative of some sort of statistical synchronization between the generative processes; which may be very different themselves. 

Thus, the concept of mutual information and data smashing is "orthogonal" in the sense that while we measure statistical dependence for computing  the former, the streams need to be  statistically independent (or very nearly so) for the latter to work. Note that in the computation of the anti-stream, we generated streams that approximate  independent copies of the input stream, which are then manipulated to yield the inverse. The algorithm requires this independence; in absence of which Proposition~\ref{propstreaminversion} falls apart.

We can illustrate these points by a simple example (See Fig.~\ref{figmi}). We consider two simple one-state PFSAs (A and B), with different event probabilities, and generated $10000$ bit  streams $A1,A2$ and $B1,B2$, Note that simply ``running" a given  PFSA twice, $i.e.$ choosing a start state randomly, and generating symbols in accordance with the event probabilities, implies that the generations are  independent. We smash the streams $A1,A2,B1,B2$ against each other, and compute the pairwise distance matrix shown in Table~\ref{tabmi1}. Note that streams $(A1,A2)$ annihilate nearly perfectly, as do the stream $(B1,B2)$ ; while streams $(A1,B1)$, $(A1,B2)$, $(A2,B1)$ and $(A2,B2)$ fail to do so. This results in  clearly clustered values in Table~\ref{tabmi1}, which correctly indicate that streams $(A1,A2)$ and  $(B1,B2)$ have identical generators, which differ significantly  from each other.

\begin{table}[t]
\centering 
\newcommand{\VSP}{\rule{0pt}{2.2ex}}
\newcommand{\VsP}{\rule{0pt}{1.2ex}}
\def\LOGOWIDTH{0.55in}
\def\PCOL{gray}
\def\QWIDT{1.1in} 
\def\SCALE{.85}
\def\HGT{.75in}
\def\WDT{1.5in}
\setlength{\arrayrulewidth}{1pt}
\setlength{\extrarowheight}{1pt}
  \fontsize{7}{8}\selectfont 
\captionN{Distance matrix obtained by smashing streams from models A and B.  {\bf   \itshape (Note clear clusters corresponding to runs from the same model)}}\label{tabmi1}
\begin{tabular}{||c||c|c|c|c||}\hline
\bf   DATA SMASHING & \VSP A1 & A2  & B1 & B2\\\hline\hline
\VSP A1 & \bf   \color{Red3} 0.005  & \bf    \color{Red3} 0.019 &  0.264 &0.269 \\\hline
\VSP A2 & \bf   \color{Red3} 0.021 & \bf    \color{Red3} 0.006 &  0.246  & 0.253
 \\\hline
\VSP B1 & 0.262 & 0.251 & \bf    \color{DodgerBlue2} 0.005 & \bf   \color{DodgerBlue2} 0.009 
 \\\hline
\VSP B2 & 0.264 & 0.254 & \bf   \color{DodgerBlue2} 0.011&  \bf   \color{DodgerBlue2} 0.006
 \\\hline
\end{tabular}
\end{table}

\begin{table}[t]
\centering 
\newcommand{\VSP}{\rule{0pt}{2.2ex}}
\newcommand{\VsP}{\rule{0pt}{1.2ex}}
\def\LOGOWIDTH{0.55in}
\def\PCOL{gray}
\def\QWIDT{1.1in} 
\def\SCALE{.85}
\def\HGT{.75in}
\def\WDT{1.5in}
\setlength{\arrayrulewidth}{1pt}
\setlength{\extrarowheight}{1pt}
  \fontsize{7}{8}\selectfont 
\captionN{Pairwise mutual information of streams from models A and B. {\bf   \itshape (No indication of generative difference)}}\label{tabmi2}
\begin{tabular}{||c||c|c|c|c||}\hline
\bf   Mutual Inf. & \VSP A1 & A2  & B1 & B2\\\hline\hline
\VSP A1 &  0.89 &   0.00000476 &  0.00017155 &  0.00002713
 \\\hline
\VSP A2 &   0.0000047 &  0.87  & 0.00001186  & 0.00003927

 \\\hline
\VSP B1 &    0.00017155  & 0.00001186  &0.48&   0.00000996

 \\\hline
\VSP B2 &    0.00002713 & 0.00003927 &   0.00000996  & 0.47

 \\\hline
\end{tabular}
\end{table}
\begin{figure}[t]
\centering
\def\datafileP{Figures/Dr/X1.dat}
\def\datafileQ{Figures/Dr/X2.dat}
\def\datafileR{Figures/Dr/X3.dat}
%
\def\WDT{1.75in}
\def\HGT{.35in}
\def\PCOL{Red4}
\begin{tikzpicture}[font=\bf\sffamily\fontsize{8}{10}\selectfont]
\begin{groupplot}[ grid,
 axis background/.style={
shade,top color=gray!40,bottom color=white}, 
ticklabel style={font=\bf \sffamily},
enlargelimits=false,
 axis on top, 
 width=\WDT, 
 height=\HGT,  
 scale only axis, 
 scaled x ticks=false,
 scaled y ticks=true,
ticklabel style={font=\bf \sffamily \footnotesize \color{black}},
 ticklabel style={/pgf/number format/fixed},
group style={
group name=my plots,
group size=1 by 3,horizontal sep=.1in , vertical sep=.35in,
xlabels at=edge bottom,
ylabels at=edge left,
yticklabels at=edge left,
}, title style={yshift=-.1in},
ylabel={}, ylabel style={yshift=-0.2in},xlabel={Simulation time (s)}]
  \pgfkeys{/pgf/number format/.cd,sci,precision=3}
  
\nextgroupplot[ultra thick, gray, 
title={(a) $r=0.209$ },
xlabel={},
]
\addplot[thin, ,no markers   ]
  table[x expr=(\coordindex+1),y expr=(\thisrowno{3})] {\datafileP};
\nextgroupplot[ultra thick, gray, 
title={(b) $r=0.257$ }, 
xlabel={},ylabel={$\Delta X(t)$},ylabel style={yshift=0.15in},
]
\addplot[thin, ,no markers   ]
  table[x expr=(\coordindex+1),y expr=(\thisrowno{3})] {\datafileQ};
\nextgroupplot[ultra thick, gray, 
title={(c) $r=0.299$ },,
]
\addplot[thin, ,no markers   ]
  table[x expr=(\coordindex+1),y expr=(\thisrowno{3})] {\datafileR};

\end{groupplot}
\end{tikzpicture}
\captionN{\textbf{Model system.} The Lotka-Volterra system of reactions is a stochastic model that captures a simple two-species predator-prey dynamical system. We parameterize the system with $r$ (the propensity of the predation reaction), and generate time series data using Gillespie's stochastic simulation algorithm. Plates (a0-(c) show prey numbers varying with time in three different runs with different $r$ values.}\label{figsi-x0} 
\end{figure}
\begin{figure*}[t]
\begin{center}
\input{Figures/sc_figs10.tex}
\end{center}
\vspace{-10pt}

\captionN{We compute distances between the simulated time series for population numbers of the species $X$. First column is obtained via pairwise smashing. The second column is obtained as the absolute difference of means, and the third column is obtained as the absolute difference of variances. The second row illustrates the effect of normalizing the data to make each time series have zero mean. The third row illustrates the effect of normalization to make each series have zero mean and unit variance. Data smashing yields clear clusters, which are preserved through the normalization; whereas the simple statistical measures do not.  }\label{figsi-x1}
\end{figure*}

Pairwise computation of mutual information between the streams $A1,A2,B1,B2$ is not expected to reveal this generative structure. Since the streams are generated independently, the mutual information between any two distinct streams would be zero (or nearly so for finitely generated streams). This is illustrated in Table~\ref{tabmi2}. Note that while the diagonal terms (which represent the self-information or entropy) are high; all off-diagonal terms are very nearly zero, and no clusters are discernible.

Thus, we can summarize:
\begin{itemize}
\item Mutual information measures the degree of statistical dependence between data streams; data smashing computes the distance between the generative processes, provided the data streams are independent or nearly so.
\item We proved that maximizing entropy of a single stream maximizes the annihilation efficiency (See Proposition~\ref{propbeta}, and its corollary)
\item Thus, data smashing is conceptually orthogonal to the notion of mutual information
\end{itemize}
\subsection{Smashing Vs Simple Statistical Measures}\label{sec-addlb}
The pairwise distances computed via  data smashing is clearly a function of the 
statistical information buried in the streams. However, it might not be easy to find the right statistical tool to mine this information for a particular problem. In this section we provide an example of a dynamical system, in which data smashing is able to recover meaningful nontrivial structure, which is missed by simple statistical measures.

We  consider  the Lotka-Volterra system of stochastic ``reactions'', modeling a simple closed eco-system of two species, one of which preys on the other. While deterministic differential equation models for this system do exist (and is widely studied), a more realistic model is this set of three simple reactions (See Fig.~\ref{figsi-x0}, plate A), primarily due to its ability to model the stochastic component. The generally accepted method to solve such systems, to produce the time traces of population numbers (See Fig.~\ref{figsi-x0}, plate B), is via the Gillespie's stochastic simulation algorithm. (Note: While the preceding theoretical development assumes ergodicity and stationarity, the theoretical considerations fall apart gracefully as we deviate from these idealizations).

In our simple model, as shown below:
\calign{
X &\xrightarrow{1.0} 2X  \tag{R1}\\
X +Y &\xrightarrow{0.005} 2Y \tag{R2}\\
Y &\xrightarrow{\phantom{0}\textrm{\large r}\phantom{0}} \varnothing \tag{R3}
}
 we consider  the propensity of one of the reactions to be  parameterized by $r$, which ranges  between $0.2$ to $0.3$ in steps of $0.001$. For each set of reaction parameters, we simulate the system $1000$ times for a maximum of $200s$ using Gillespie's algorithm. In each simulation run, we initialize the system with $X=128, Y=256$. We assume that we can make observations every $0.1s$ from the simulated dynamics.
A few sample paths for the change in the number of prey with time are shown in Fig.~\ref{figsi-x1}. 

Note that the probability of each reaction at any point in time is proportional to the 
number of combinatorial ways that particular reaction can transpire, as well as the propensity of the reaction itself. This combinatorial number is a function of the current population count of each species; and hence the reaction probabilities are strongly dependent on the current state vector. Since the simulation terminates when any one species becomes extinct, we cannot assume stationary behavior. Also, the initial state at least partially dictates the length of the time the ecosystem survives, implying non-ergodicity. Note that removing the restriction of a strictly positive and integer-valued population count, might result in a more well-behaved system.

Given our parameterization, we actually have $101$ distinct systems with different sets of parameters, and 
for the system with index $i$, we have set the propensity of the third reaction as:
\cgather{
r= 0.2 + 0.001 i
}
%
\begin{figure*}[t]
\begin{center}
\def\WDT{4in}
\def\HGT{2in
}
\def\datafileAb{Figures/Dr/smash-3}
\def\dataclsfile{Figures/Dr/CLS-}

\begin{tikzpicture}[font=\bf \sffamily \fontsize{8}{9}\selectfont]

\node[] (A) {
\begin{tikzpicture}
  \begin{axis}[title={ \bf \sffamily {\large A } Distinct dynamical features in parameterized Lotka-Volterra system }, 
    ,   
    axis background/.style={
      shade,top color=gray!35,bottom color=white}, 
    axis on top,
    grid, 
    width=\WDT, 
    height=\HGT, 
    scale only axis, 
ylabel style={align=center}, 
    ylabel={Minimum \# of predators present\\ at the end of $100$s of simulation},scaled ticks=false,
yticklabel style={/pgf/number format/.cd,fixed,precision=5},
legend style={inner sep=3pt,row sep=-1pt,draw=none, fill=CadetBlue4!50, opacity=0.5, text opacity=1, 
at={(1,0.95)},},cycle list name=exotic,
xlabel={system index $i$, where parameter $r=0.2 + 0.001 i$ }
]
\def\fileext{.dat}
\def\colors{{Red4,Green4,Blue4,teal,Tomato,cyan,Yellow2,Blue1,Red1,Green2}}
\foreach \i in {1,...,10}
{
    \def\datafile{\dataclsfile\i\fileext}
    \addplot+[,semithick,
    opacity=1 ] 
    table[x index={0}, y expr=((\thisrowno{1}))]
    {\datafile};
 
}

  \end{axis}
   \end{tikzpicture} 
};

\def\WDT{1.2in}
\def\HGT{1.2in}
\node[anchor=east] (B) at ([xshift=-.2in,yshift=.2in]A.east) {
\begin{tikzpicture}\begin{axis}[ title style={yshift=-.1in}, width=\WDT, height=\HGT,scale only axis,legend pos=south east,xlabel={ Delivery ratio },axis on top, axis background/.style={thick,
shade,top color=gray!0,bottom color=white}, ylabel={}, ylabel style={yshift=-0.2in},xlabel={Simulation time (s)},
    point meta min=-3.7, point meta max=-0.7,
    colormap={whiteblack}{rgb255=(0, 0, 0) rgb255=(27, 27, 27) rgb255=(54, 54, 54) rgb255=(82, 82, 82) rgb255=(109, 109, 109) rgb255=(136, 136, 136) rgb255=(163, 163, 163) rgb255=(190, 190, 190) rgb255=(206, 206, 206) rgb255=(221, 221, 221) rgb255=(236, 236, 236) rgb255=(239, 239, 239) rgb255=(242, 242, 242) rgb255=(245, 245, 245) rgb255=(248, 248, 248) rgb255=(250, 250, 250) rgb255=(251, 251, 251) rgb255=(251, 251, 251) rgb255=(251, 251, 251) rgb255=(251, 251, 251) rgb255=(251, 251, 251) rgb255=(251, 251, 251) rgb255=(251, 251, 251) rgb255=(251, 251, 251) rgb255=(251, 251, 251) rgb255=(251, 251, 251) rgb255=(252, 252, 252) rgb255=(252, 252, 252) rgb255=(252, 252, 252) rgb255=(252, 252, 252) rgb255=(252, 252, 252) rgb255=(252, 252, 252) rgb255=(252, 252, 252) rgb255=(252, 252, 252) rgb255=(252, 252, 252) rgb255=(252, 252, 252) rgb255=(252, 252, 252) rgb255=(253, 253, 253) rgb255=(253, 253, 253) rgb255=(253, 253, 253) rgb255=(253, 253, 253) rgb255=(253, 253, 253) rgb255=(253, 253, 253) rgb255=(253, 253, 253) rgb255=(253, 253, 253) rgb255=(253, 253, 253) rgb255=(253, 253, 253) rgb255=(253, 253, 253) rgb255=(254, 254, 254) rgb255=(254, 254, 254) rgb255=(254, 254, 254) rgb255=(254, 254, 254) rgb255=(254, 254, 254) rgb255=(254, 254, 254) rgb255=(254, 254, 254) rgb255=(254, 254, 254) rgb255=(254, 254, 254) rgb255=(254, 254, 254) rgb255=(255, 255, 255) rgb255=(255, 255, 255) rgb255=(255, 255, 255) rgb255=(255, 255, 255) rgb255=(255, 255, 255) rgb255=(255, 255, 255) } ,
    colorbar style={
   y tick label style={rotate=30, yshift=-.1in}, tick style={gray, very thick},
   scale only axis, 
   yticklabel=\pgfmathparse{10^\tick}\pgfmathprintnumber\pgfmathresult,
yshift=0in,
xshift=-.in,
width=.1in,
font=\sffamily\fontsize{7}{7}\selectfont,
    },ultra thick, gray, title={{ \bf \sffamily {\large \color{black} B }} Smashing Distances },ymin=1, ymax=101,xmin=1,xmax=101,xlabel={},      y dir=reverse, yticklabel pos=right
]
 \addplot graphics [ymin=1, ymax=101,xmin=1,xmax=101 ] {\datafileAb};
\draw[Red1, semithick] (axis cs:1,1) rectangle (axis cs:5,5);
\draw[Red1, semithick] (axis cs:6,6) rectangle (axis cs:8,8);
\draw[Red1, semithick] (axis cs:9,9) rectangle (axis cs:13,13);
\draw[Red1, semithick] (axis cs:14,14) rectangle (axis cs:19,19);
\draw[Red1, semithick] (axis cs:20,20) rectangle (axis cs:31,31);
\draw[Red1, semithick] (axis cs:32,32) rectangle (axis cs:45,45);
\draw[Red1, semithick] (axis cs:46,46) rectangle (axis cs:62,62);
\draw[Red1, semithick] (axis cs:63,63) rectangle (axis cs:70,70);
\draw[Red1, semithick] (axis cs:71,71) rectangle (axis cs:91,91);
\draw[Red1, semithick] (axis cs:92,92) rectangle (axis cs:101,101);

\end{axis}\end{tikzpicture} 
};

\draw [gray, opacity=.85] ([xshift=-.1in,yshift=-.1in]B.center) -- ([yshift=-.65in, xshift=.35in]A.center);
\draw [gray, opacity=.85] ([xshift=-.3in,yshift=0.1in]B.center) -- ([yshift=-.45in, xshift=-.15in]A.center);
\draw [gray, opacity=.85] ([xshift=-.4in,yshift=0.25in]B.center) -- ([yshift=-.15in, xshift=-.65in]A.center);

\end{tikzpicture} 
\end{center}
\vspace{-10pt}

\captionN{\textbf{Mapping clusters recovered via data smashing to a meaningful dynamical feature.} Plotting the minimum number of predators remaining in the system after a fixed simulation time of $100s$ (minimization carried out over the $1000$ simulation runs of each system), illustrates that the clusters almost perfect correspond to monotonic domains of this function.  }\label{figsi-x3}
\end{figure*}
Since the third reaction models predator death, we expect that increasing its propensity will make the predator degradation more probable. Thus, we can clearly expect
smaller number of predators and larger number of prey on average, as $r$ is increased.
However, a truly interesting structure would be uncovered if the behaviors exhibit some sort of clustering; as opposed to simply a monotonic dependence on $r$. We aanalyzed our set of $r$-parameterized dynamical systems as follows:
\begin{enumerate}
\item We concatenated all $1000$ time series for species $X$ generated by simulating each system. Thus, the $i^{th}$ system generates the concatenated series $s_i$ for species $X$.
\item Next we generated $s'_i$ from $s_i$ by taking one step differences from $s_i$, $i.e.$ $s'_i$ is the time series of relative updates for the population of species $X$.
\item We mapped each  sequential data series $s'_i$, to a symbol stream using a binary partition function, which maps negative entries in the series to symbol $0$, and positive entries to symbol $1$. 
\item We collided the symbolic streams pairwise, and compute the smashing distance matrix $H$. Thus, the $ij^{th}$ entry in $H$ is the deviation of the sum of $s'_i$ and an inverted copy of $s'_j$ from flat white noise. The result is shown in plate A(i) of Fig.~\ref{figsi-x1}.
\item We also generated the pairwise absolute differences of means, and variances. In each of these cases, the $ij^{th}$ entry of the corresponding matrix is the absolute difference between the corresponding statistical measure between the data series $s'_i$ and $s'_j$ (See plates A(ii-iii) in Fig.~\ref{figsi-x1}).
\end{enumerate}

{ \itshape Notably, the smashing matrix in  plate A(i) of Fig.~\ref{figsi-x1} shows clear clusters, whereas the matrices corresponding to mean and variance show trivial monotonic dependence on the parameter $r$. }
To ascertain if the clustering obtained via data smashing is dependent on the mean or variance of the input data streams, we redid the analysis, after: 
\begin{enumerate}
\item Normalization to zero mean signals prior to symbolization
\item Normalization to  zero mean and unit variance signals prior to symbolization
\end{enumerate}
In the first case, zeroing the mean makes the clusters appear more prominently (See plate B(i) of Fig.~\ref{figsi-x1}), while additionally normalizing the variance has little effect (See  plate C(i) of Fig.~\ref{figsi-x1}).
None of these changes allow the simple statistical measures to recover the clear clusters obtained via data smashing.
%
%
The Lotka-Volterra system has a rich set of dynamical regimes, and it would not be surprising if such measures fail to capture this complexity. To that effect, we plotted the minimum number of predators after $100s$ of simulation (minimum calculated over the $1000$ simulation runs carried out for each parameter set, as discussed before). The result is shown in Fig.~\ref{figsi-x3}. 

{\bf \itshape The clusters identified via data smashing is now seen to correspond almost perfectly for each monotonic domain of this function.} 
This  illustrates that data smashing finds meaningful categorization, which simple statistical tools may miss. 
The differences discovered via  smashing is obviously a function of the statistical structure of the observed data. However, the preceding example illustrates that  it may not be easy to find the right statistical tool. Data smashing approach alleviates this challenge to a considerable degree.
\section{Conclusion}\label{sec-conc}
We introduced data smashing to measure causal similarity between series of sequential observations. We demonstrated that our insight allows  feature-less model-free classification in diverse applications, without the need for training, or expert tuned heuristics. Non-equal length of time-series, missing data, and  possible phase mismatches are of no consequence.


While better classification algorithms may exist for  specific problem domains,  such algorithms are difficult to tune. The strength of data smashing lies in its ability to circumvent both the need for expert-defined heuristic features and expensive training;  eliminating key bottlenecks in contemporary big data challenges.
}
\bibliographystyle{IEEEtran}
\bibliography{metr,ref1}
\end{document}